\documentclass{article}

\usepackage{amsmath,amsfonts,bm}

\newcommand*{\rom}[1]{%
\textup{\uppercase\expandafter{\romannumeral#1}}%
}

\def\eqref#1{equation~\ref{#1}}

\def\1{\bm{1}}

\def\vs{{\bm{s}}}

\def\vu{{\bm{u}}}

\def\vw{{\bm{w}}}
\def\vx{{\bm{x}}}

\def\vz{{\bm{z}}}

\def\muon{\texttt{Muon}\xspace}
\def\signgd{\texttt{SignGD}\xspace}
\def\adam{\texttt{Adam}\xspace}
\def\adamw{\texttt{AdamW}\xspace}
\def\scaledgd{\texttt{ScaledGD}\xspace}
\def\gd{\texttt{GD}\xspace}
\def\shampoo{\texttt{Shampoo}\xspace}

\def\mA{{\bm{A}}}
\def\mB{{\bm{B}}}

\def\mE{{\bm{E}}}

\def\mG{{\bm{G}}}
\def\mH{{\bm{H}}}
\def\mI{{\bm{I}}}
\def\mJ{{\bm{J}}}

\def\mM{{\bm{M}}}

\def\mO{{\bm{O}}}

\def\mQ{{\bm{Q}}}
\def\mR{{\bm{R}}}
\def\mS{{\bm{S}}}

\def\mU{{\bm{U}}}
\def\mV{{\bm{V}}}
\def\mW{{\bm{W}}}
\def\mX{{\bm{X}}}
\def\mY{{\bm{Y}}}
\def\mZ{{\bm{Z}}}

\def\mLambda{{\bm{\Lambda}}}

\def\mTheta{{\bm{\Theta}}}
\def\mSigma{{\bm{\Sigma}}}
\def\mDelta{{\bm{\Delta}}}

\DeclareMathAlphabet{\mathsfit}{\encodingdefault}{\sfdefault}{m}{sl}
\SetMathAlphabet{\mathsfit}{bold}{\encodingdefault}{\sfdefault}{bx}{n}

\DeclareMathOperator{\sign}{\mathsf{sign}}
\DeclareMathOperator{\dsign}{\mathsf{diag}\text{-}\mathsf{sign}}
\DeclareMathOperator{\msign}{\mathsf{msign}}

\newcommand{\cO}{\mathcal{O}}

\newcommand{\cS}{\mathcal{S}}

\newcommand{\bE}{\mathbb{E}}

\newcommand{\bP}{\mathbb{P}}

\newcommand{\norm}[1]{\left\lVert#1\right\rVert}
\newcommand{\fro}{\mathrm{F}}

\newcommand{\diag}{\mathsf{diag}}

\def\bR{{\mathbb{R}}}

\usepackage[dvipsnames]{xcolor}
\usepackage{pgfplots}
\usepackage{booktabs}
\usepackage{natbib}
\usepackage{graphicx,wrapfig,lipsum}
\usepackage{float}    %
\usepackage{verbatim} %
\usepackage{amsmath, amssymb, amsthm, mathtools}  %
\usepackage[colorlinks=true,citecolor=blue,linkcolor=blue,urlcolor=blue]{hyperref}
\usepackage{cleveref,comment}
\usepackage{bookmark}
\usepackage{fullpage}
\usepackage{enumerate}
\usepackage{paralist}
\usepackage{xspace}
\usepackage{multirow}
\usepackage{caption}
\usepackage{subcaption}   %
\usepackage{bbm}
\usepackage{bm}
\usepackage{url}
\usepackage{lipsum}
\usepackage{algorithm}
\usepackage{algorithmic}
\usepackage{color}
\usepackage{microtype}
\usepackage[section]{placeins}
\usepackage{lscape}
\usepackage{multirow}
\usepackage{todonotes}
\usepackage{nicefrac}

\allowdisplaybreaks

\newcommand{\stoptocwriting}{%
	\addtocontents{toc}{\protect\setcounter{tocdepth}{-5}}}
\newcommand{\resumetocwriting}{%
	\addtocontents{toc}{\protect\setcounter{tocdepth}{\arabic{tocdepth}}}}

\newcommand{\vertiii}[1]{{\left\vert\kern-0.25ex\left\vert\kern-0.25ex\left\vert #1
		\right\vert\kern-0.25ex\right\vert\kern-0.25ex\right\vert}}

\makeatletter

\definecolor{yxc}{RGB}{255,0,0}

\definecolor{yjc}{RGB}{0,155,100}

\allowdisplaybreaks

\newcommand{\tr}{\operatorname{tr}}

\newcommand{\vecc}{\mathsf{vec}}

\newtheorem{theorem}{Theorem}

\newtheorem{lemma}{Lemma}

\newtheorem{remark}{Remark}

\usepackage{tikz}
\definecolor{sketchblue}{HTML}{007BFF}

\begin{document}

	\title{%
    Preconditioning Benefits of Spectral Orthogonalization in Muon}

\author{%
Jianhao Ma\thanks{Department of Statistics and Data Science,  Wharton School, University of Pennsylvania.} \\
Penn
\and
Yu Huang\footnotemark[1] \\
Penn 
\and Yuejie Chi\thanks{Department of Statistics and Data Science, Yale University.} \\
Yale 
 \and Yuxin Chen\footnotemark[1] \thanks{Department of Electrical and Systems Engineering,  University of Pennsylvania.} \\
 Penn \\
}
    
	\maketitle
	\begin{abstract}
		 The \texttt{Muon} optimizer, a matrix-structured algorithm that leverages spectral orthogonalization of gradients, is a milestone in the pretraining of large language models. However, the underlying mechanisms of \texttt{Muon}---particularly the role of gradient orthogonalization---remain poorly understood, with very few works providing end-to-end analyses that rigorously explain its advantages in concrete applications.  
        We take a step by studying the effectiveness of a simplified variant of \texttt{Muon} through two case studies: matrix factorization, and in-context learning of linear transformers. For both problems, we prove that simplified \texttt{Muon} converges linearly with iteration complexities independent of the relevant condition number, provably outperforming gradient descent and \texttt{Adam}. Our analysis reveals that the \muon dynamics decouple into a collection of independent scalar sequences in the spectral domain, each exhibiting similar convergence behavior. Our theory formalizes the preconditioning effect induced by spectral orthogonalization, offering insight into \texttt{Muon}’s effectiveness in these matrix optimization problems and potentially beyond.
	\end{abstract}

\stoptocwriting
\section{Introduction}

The emergence of \muon---a matrix-structured, spectrum-aware  optimizer recently proposed by \citet{jordan2024muon}---has marked a milestone in the pretraining of large language models (LLMs) and beyond. Standing for {\em MomentUm Orthogonalized by Newton-Schulz} and leveraging spectral orthogonalization of gradients, \muon  was initially shown to set new training speed records on benchmarks like CIFAR-10 and NanoGPT, outperforming conventional optimizers \citep{jordan2024muon}. 
Subsequent work has scaled \muon to multi-billion-parameter LLMs, demonstrating approximately a twofold improvement in training efficiency over the \adamw optimizer  \citep{liu2025muon}. Such empirical advances have positioned \muon as a compelling alternative to established optimizers such as \adam and \texttt{AdamW}, and have motivated theoretical investigation into the mechanisms underlying \muon's practical efficiency.

\subsection{The \muon algorithm and prior theory}

Setting the stage, consider an unconstrained optimization problem:
\begin{align}
    \text{minimize}_{\mX} \quad f(\mX),  
\end{align}
where $\mX\in \bR^{m\times n}$ is a matrix variable. At each iteration $t\geq 0$,  \muon  executes the following  update:
\begin{subequations}
\begin{align}\label{eq:muon-updates-full}
    \mB_t = \nabla f(\mX_t) + \mu \mB_{t-1},  \\
    \mX_{t+1} = \mX_t - \eta_t \msign(\mB_t),
\end{align}
\end{subequations}
where $\mB_t$ represents an auxiliary momentum-like iterate that aggregates the current gradient with exponentially discounted past gradients, $0\leq \mu <1$ controls the degree of momentum (exponential averaging),  $\eta_t>0$ stands for the learning rate at iteration $t$, and 
$\msign(\cdot)$ denotes the matrix sign function defined as
\begin{align}
    \msign(\mZ) \coloneqq \arg\min_{\mO} \big\{ \| \mZ - \mO \|_{\mathrm{F}}: \text{either }\mO\mO^{\top} = \mI \text{ or }\mO^{\top}\mO = \mI \big\}.
\end{align}
Equivalently, if a matrix $\mZ$ has compact singular value decomposition (SVD) $\mZ=\mU_Z\mSigma_Z\mV_Z^{\top}$---where $\mU_Z$ (resp.~$\mV_Z$) denotes the left (resp.~right) singular matrix---then its matrix sign is given by 
$\msign(\mZ)=\mU_Z\mV_Z^{\top}$, although in practice $\msign(\cdot)$ is computed efficiently using Newton-Schulz iterations \citep{jordan2024muon,higham2008functions}.  
A notable special case of (\ref{eq:muon-updates-full}) arises when momentum is disabled by setting $\mu=0$, yielding the simplified update rule
\begin{align}
\label{eq:muon-simplified-updates}
    \mX_{t+1} = \mX_t - \eta_t \msign\big(\nabla f(\mX_t)\big),
    \qquad t=0,1,\cdots
\end{align}
This important variant is commonly referred to as  {\em simplified \muon} or the {\em spectral gradient method}. Turning off momentum substantially simplifies theoretical analysis \citep{an2025asgo,shen2025convergence,davis2025spectral,su2025isotropic}, while often retaining comparable empirical performance to its momentum-based counterpart for nonstochastic settings \citep{shen2025convergence}. 
In contrast to standard optimizers like \adam \citep{kingma2014adam} and \texttt{AdamW} \citep{loshchilov2019decoupled}  that apply independent per-coordinate preconditioning, a distinguishing feature of \muon or spectral gradient methods lies in the use of gradient orthogonalization: update directions are obtained by spectrally orthogonalizing the gradient estimates.

Motivated by \muon's remarkable empirical success, the past year has witnessed a surge of theoretical efforts aimed at elucidating the mechanisms behind its effectiveness from diverse perspectives. 
From an optimization standpoint,  
\citet{li2025note,shen2025convergence} established convergence guarantees of \muon on smooth objectives. In particular, \cite{shen2025convergence} showed that \muon's convergence is governed by the gradient Lipschitz parameter defined w.r.t.~the spectral norm, which can sometimes be substantially smaller than its Euclidean counterpart and hence offers a potential explanation for \muon's accelerated convergence. \citet{kovalev2025understanding} interpreted \muon as a trust region method with non-Euclidean
trust regions and derived tighter convergence rates for certain function classes. 
Complementing this line of work, \cite{chen2025muon} showed that: \muon (with decoupled weight decay) approximately enforces a spectral norm constraint on weight updates,   which implicitly reduces the worst-case smoothness of the optimization landscape and enables the use of larger learning rates.  Stepping beyond worst-case convergence guarantees, \cite{davis2025spectral} compared the one-step progress of spectral updates (as in (\ref{eq:muon-simplified-updates})) relative to Euclidean gradient updates, and showed that \muon yields a larger {\em one-step} reduction in the objective than gradient descent (\gd) when the gradient rank exceeds the activation rank. Another recent work \cite{su2025isotropic} introduced an ``isotropic curvature model''---proposed through heuristic arguments and validated empirically in transformer training---and derived gradient orthogonalization iterations as the optimal updates under certain assumptions.

Despite these theoretical pursuits, however, the theoretical foundation of \muon remains far from complete. In particular, very few existing results were able to offer end-to-end, rigorous analyses that provably demonstrate \muon's advantages over classical optimizers in concrete applications.

\subsection{This paper: preconditioning with \muon}

In this work, we take a step towards theoretically justifying the effectiveness of \muon by investigating its preconditioning effect---a core feature built into its design via spectral orthogonalization---that is hypothesized to make the optimizer better align with the geometry of neural networks
\citep{jordan2024muon,bernstein2024old,vasudeva2025muon,lau2025polargrad}. Rather than tackling the most general settings, we focus on two concrete, yet fundamental, matrix optimization problems: (a) matrix factorization, and (b) in-context learning of linear transformers. 
By focusing on these stylized applications, we develop end-to-end convergence theory unveiling provable advantages of \muon over classical optimizers like \gd and \adam. Our main contributions are summarized below. 
\begin{itemize}
    \item {\em Matrix factorization.} We show in \Cref{thm::muon-main} that simplified \muon converges linearly, encompassing both exactly-parameterized and over-parameterized settings. Notably, \muon's iteration complexity is provably independent of the condition number $\kappa$ of the matrix to be factorized---a stark contrast to both \gd and \signgd (a simplified variant of \adam with momentum disabled), whose iteration complexities scale at least linearly with $\kappa$ (cf.~\Cref{thm::adam-lower-bound}).  

    \item {\em In-context learning of linear transformers.} Akin to the matrix factorization case, we establish linear convergence of simplified \muon in \Cref{thm::muon-transformer}, with an iteration complexity independent of the condition number of the target covariance matrix. This contrasts sharply with both \gd and \signgd, for which we develop iteration complexity lower bounds (cf.~\Cref{prop::lower-bound-transformer}) that scale polynomially with the condition number of interest. 

\end{itemize}
See \Cref{tab:convergence_rates_mf} for more detailed comparisons. 
For both problems, our results reveal that by normalizing the gradient spectrum at each iteration, \muon exhibits preconditioning benefits that yield provably faster, condition-number-free, convergence rates. These theoretical findings are complemented by a series of numerical experiments that corroborate the preconditioning benefits of \muon. At a more technical level, our analyses uncover that the dynamics of \muon decouple into a collection of independent scalar sequences in the spectral domain, each associated with one eigenvalue of the target matrix and exhibiting similar convergence behavior.  
While our theory is restricted to two simple matrix optimization problems and by no means exhaustive, we expect the preconditioning effect of \muon to manifest in broader applications.

\newcommand{\topsepremove}{\aboverulesep = 0mm \belowrulesep = 0mm} \topsepremove

\begin{table}[t]
    \centering
    \renewcommand{\arraystretch}{1.4} %

    {\small
\begin{tabular}{c|c|c|c|c}
\toprule
 problem & \multicolumn{2}{c|}{algorithm} & iterations & paper\tabularnewline
\toprule 
 & \multirow{2}{*}{simplified \muon} & exactly-parameterized & $\log\frac{1}{\varepsilon}$ & this work (\Cref{thm::muon-main})\tabularnewline
 &  & over-parameterized & $\log\frac{1}{\varepsilon}$ & this work (\Cref{thm::muon-main})\tabularnewline
\cline{2-5}
matrix & \multirow{3}{*}{\gd} & exactly-parameterized & $\kappa\log\frac{1}{\varepsilon}$ & \cite{chi2019nonconvex}\tabularnewline
factorization &  & over-parameterized & $\kappa^{3}\log\frac{1}{\varepsilon}$ & \cite{stoger2021small}\tabularnewline
 &  & lower bound & $\kappa\log\frac{1}{\varepsilon}$ & folklore\tabularnewline
\cline{2-5}
 & \signgd & lower bound & $\kappa$ & this work (\Cref{thm::adam-lower-bound})\tabularnewline
\toprule 
 & simplified \muon & exactly-parameterized & $\log\frac{1}{\varepsilon}$ & this work (\Cref{thm::muon-transformer})
 \tabularnewline
\cline{2-5}
in-context learning & \gd & lower bound & $\sqrt{\kappa}\log\frac{1}{\varepsilon}$ & \citet{d2021acceleration}\tabularnewline
\cline{2-5}
 & \signgd & lower bound & $\kappa$ & this work (\Cref{prop::lower-bound-transformer})\tabularnewline
 \bottomrule
\end{tabular}
}
        \caption{Summary of convergence theory for simplified \muon, \gd and \signgd for both matrix factorization and in-context learning tasks. We report the numbers of iterations required to achieve $\varepsilon$-accuracy; only the orders are shown, with all preconstants omitted. 
    \label{tab:convergence_rates_mf}}
\end{table}

\subsection{Additional related work}

We now provide additional discussion of related prior work. 
The convergence analyses in \cite{li2025note,shen2025convergence,chen2025muon} were motivated in part by \cite{bernstein2024old}, which interpreted (some simplified variants of) \adam, \shampoo, and \texttt{Prodigy} as steepest descent under certain norm constraints. 
It is noteworthy that the idea of spectral initialization of gradients has appeared in earlier designs of optimizers (e.g., \citet{carlson2015stochasticB,carlson2015stochastic,carlson2015preconditioned,tuddenham2022orthogonalising}). 
Another line of research studied the implicit bias of \muon. 
For example, \cite{fan2025implicit} showed that in multi-class linear classification, \muon (or its idealized variant with exact orthogonal updates) converges to solutions that maximize the margin w.r.t.~the spectral norm of the weight matrix, which contrasts with the biases of \texttt{SGD} or \adam that favor max-margin solutions w.r.t.~Euclidean or coordinate-wise norms. 
 Moreover, spectrum-aware optimizers like \muon were shown to improve generalization on tasks with imbalanced or long-tailed data distributions \citep{vasudeva2025muon,wang2025muon}, as \muon  (with the aid of spectral orthogonalization) tends to learn all principal components of the data at a more uniform rate instead of over-emphasizing the dominant features. 
Further insights were provided by \cite{zhang2025concurrence}, who demonstrated statistical benefits of layer-wise preconditioning in simplified settings, and by \cite{wang2025muon,vasudeva2025muon}, who showed that \muon yields a more isotropic singular value spectrum than \adam. Moreover, \cite{tveit2025muon} reported that \muon accelerates grokking, offering further evidence of its practical advantages in long-horizon training dynamics. 
There have also been discussions drawing connections between \muon and other second-order methods---for example, \cite{jordan2024muon,shah2025practical} noted that \muon's update can be interpreted as an approximate form of \shampoo \citep{gupta2018shampoo}. 
Lastly, several prior work derived \muon and closely related methods from alternative theoretical perspectives, with some of these studies even predating the formal introduction of \muon \citep{pethick2025training,carlson2015preconditioned,lau2025polargrad,bernstein2024modular,bernstein2024old,an2025asgo}.

Moving beyond \muon, it is worth noting that 
preconditioning has emerged as a powerful tool for accelerating nonconvex matrix factorization. \cite{tong2021accelerating} introduced \scaledgd, with a nonsmooth version presented in \cite{tong2021low}. They proved that in the exactly-parameterized regime with spectral initialization, \scaledgd achieves linear convergence at a rate independent of the condition number. Subsequent work by \cite{zhang2021preconditioned,zhang2023preconditioned} extended these results to the over-parameterized setting, demonstrating condition-number-free convergence  when suitably initialized.  \cite{xu2023power} showed that \scaledgd remains effective under small random initialization, further broadening the scope of condition-number-free guarantees.

\subsection{Notation}

We also introduce a set of useful notation. 
For any matrix $\mM$, we denote by $\sigma_i(\mM)$ the $i$-th largest singular value of $\mM$, let $\sigma_{\min}(\mM)$ be its smallest singular value,  and we let $\|\mM\|$ (resp.~$\|\mM\|_{\mathrm{F}}$) represent its spectral norm (resp.~Frobenius norm). For any $k\leq d$, we let $\mathcal{O}_{d\times k}$ denote the set of orthonormal matrices in $\mathbb{R}^{d\times k}$. For any set of scalars $(a_1,\dots,a_d)$, we denote by $\diag\{a_1,\dots,a_d\}$ the diagonal matrix whose diagonal entries are $a_1,\dots,a_d$. Finally, for any scalar $x\in\bR$, we define the sign function as $\sign(x)=1$ if $x>0$, $\sign(x)=0$ if $x=0$, and $\sign(x)=-1$ if $x<0$.

\section{Main results: two case studies}
\label{sec:main-results}

In this section, we carry out both theoretical and empirical studies on two simple yet fundamental matrix optimization problems. Here and throughout, we shall focus on analyzing  simplified \muon described in \Cref{eq:muon-simplified-updates}, which discards the momentum term and thereby facilitates analysis.

\subsection{Matrix factorization}
\label{sec:main-results-mf}

The first problem considered herein is symmetric matrix factorization,  which can be formulated as
\begin{equation}
    \mathop{\text{minimize}}\limits_{\mU\in \bR^{d\times k}}\quad f(\mU)=\frac{1}{4}\big\|\mU\mU^{\top}-\mM^\star\big\|_\fro^2.
    \label{eq:optimization-matrix-factorization}
\end{equation}
Here, $\mM^\star \in \bR^{d \times d}$ is a rank-$r$ positive semidefinite matrix, and $\mU \in \bR^{d \times k}$ is a (possibly over-parameterized) factor containing $k$ ($k\geq r$) columns. In a nutshell, we seek to factorize the target matrix $\mM^{\star}$ as $\mU\mU^{\top}$ by solving the optimization problem (\ref{eq:optimization-matrix-factorization}). 
Throughout, we let $\mM^\star = \mV^\star \mLambda^\star \mV^{\star\top}$ be the eigen-decomposition of $\mM^\star$, where $\mLambda^\star = \diag\{\lambda_1^\star, \dots, \lambda_r^\star\}$ contains the nonzero eigenvalues $\lambda_1^\star \geq \cdots \geq \lambda_r^\star > 0$, and $\mV^{\star}\in \mathbb{R}^{d\times r}$ is an orthonormal matrix whose columns correspond to the associated eigenvectors.  The condition number of $\mM^{\star}$ is defined and denoted by $$\kappa \coloneqq \lambda_1^\star / \lambda_r^\star.$$

\subsubsection{Convergence guarantees for \muon} 
When applied to the matrix factorization problem (\ref{eq:optimization-matrix-factorization}), the simplified \muon algorithm (\ref{eq:muon-simplified-updates}) yields a straightforward closed-form update rule
\begin{align}
    \label{eq:muon-mf}
    \mU_{t+1} = \mU_t - \eta_t \msign \big( (\mU_t\mU_t^{\top}-\mM^{\star}) \mU_t \big),
    \qquad t=0,1,\cdots 
\end{align}
For both the exactly-parameterized (i.e., $k=r$) and over-parameterized  (i.e., $k>r$, or even $k>d$) settings, we establish rapid convergence of simplified \muon to the ground truth, as formalized in the theorem below.

\begin{theorem}
\label{thm::muon-main}
    Suppose that $\lambda_{\max}^{\star}\geq \lambda_1^{\star}\geq \dots \geq \lambda_r^{\star}> 0$, and consider any $0<\varepsilon<\lambda_{\max}^{\star}$.  
    \begin{itemize}
        \item[(a)] Consider the case with $k\geq d$. Set the learning rates as $\eta_t = C_{\eta}\sqrt{\lambda_{\max}^\star}\rho^t$ for $1/2 \leq \rho < 1$, with $C_{\eta}$ uniformly sampled from the interval $[1, 2]$. Set the initialization as  $\mU_0 = \alpha\mO$, where $0<\alpha\leq C_{\eta}\sqrt{\lambda_{\max}^\star}$ and $\mO\mO^{\top}=\mI_d$. Then with probability $1$, it holds that
        $\big\|\mU_T\mU_T^\top - \mM^\star \big\| \leq \varepsilon$
        as long as
        \begin{align}
        T\geq \frac{1}{1-\rho} \log \bigg( \frac{8\lambda_{\max}^{\star}}{\varepsilon} \bigg). 
        \end{align}

        \item[(b)] Consider the case with $r\leq k<d$. Set the learning rates as $\eta_t = C_{\eta, t}\sqrt{\lambda_{\max}^\star}\rho^t$ for $2/3 \leq \rho < 1$, with $C_{\eta, t}$ independently and uniformly sampled from the interval $[1, 2]$. Set the initialization $\mU_0 = \alpha\mO$ for some  $\alpha>0$, where $\mO \in \mathcal{O}_{d \times k}$ is an orthonormal matrix sampled uniformly at random from $\mathcal{O}_{d \times k}$. Then with probability at least $0.99$, we have
      $ \big\|\mU_T\mU_T^\top - \mM^\star \big\| \leq \varepsilon $
        as soon as
        \begin{align}
        T=\left\lceil \frac{1}{1-\rho} \log \bigg( \frac{16\lambda_{\max}^{\star}}{(1-\rho)^2\varepsilon} \bigg)\right\rceil, 
        \end{align}
        provided that $\alpha$ is sufficiently small. 
    \end{itemize}
\end{theorem}
\begin{remark}
    Careful readers may note that our theory for the regime $r\leq k<d$ requires more restrictive conditions than those in the regime with $k\geq d$.  We believe that these restrictions are not fundamental and can potentially be relaxed via more refined analyses, which we leave for future work. 
\end{remark}

Remarkably, \Cref{thm::muon-main} uncovers that when applied to matrix factorization, the iteration complexity of \muon is entirely independent from the condition number $\kappa$ of the target matrix $\mX^{\star}$, and scales only logarithmically with the inverse accuracy level $1/\varepsilon$ (thereby establishing linear convergence for \muon).  This finding suggests that the gradient orthogonalization step in \muon serves as an effective preconditioner, accelerating convergence by mitigating ill-conditioning in the gradient search directions. Even in the presence of overparameterization, \muon is guaranteed to achieve condition-number-free linear convergence. 

We also briefly explain the rationale for using exponentially decaying learning rates. In contrast to \gd---where the distance moved in each iteration depends on both the gradient norm and the learning rate---each \muon iteration moves a fixed distance determined solely by the learning rate $\eta_t$. Consequently, to achieve linear convergence, the length of each movement---namely, $\eta_t$---must decrease geometrically over iterations.

\subsubsection{Comparisons with other optimizers} 
To better demonstrate the preconditioning benefits offered by \muon, we compare its convergence theory established in \Cref{thm::muon-main} against two prominent baselines: gradient descent, and a simplified variant of \adam with momentum turned off. Notably, the latter two optimizers are  unable to achieve the desirable condition-number-free convergence rates.

Let us begin by examining \gd, whose convergence properties for matrix factorization have been extensively studied. More precisely, consider the following \gd update rule:
\begin{align}
\textbf{(\gd)}\qquad
    \mU_{t+1} = \mU_t - \eta_t \big( (\mU_t\mU_t^{\top}-\mM^{\star})\mU_t\big),
    \qquad t= 0,1,\cdots
\end{align}
The state-of-the-art convergence theory for this algorithm can be summarized as follows: by taking the learning rates $\eta_t=\Theta(1/\lambda_1^{\star})$, \gd yields $\|\mU_t\mU_t^{\top}-\mM^{\star}\|\leq \varepsilon$ in
\begin{align}
    \begin{cases}
        O\big(\kappa \log (1/\varepsilon)\big) \text{ iterations } & \text{if }k=r~(\text{exactly-parameterized}), \\
        O(\min\{\kappa^3\log(1/\varepsilon),\,\lambda_1^\star/\varepsilon)\}\text{ iterations } & \text{if }k>r~(\text{over-parameterized});
    \end{cases}
\end{align}
see, e.g.,  \cite{chi2019nonconvex,stoger2021small,zhuo2024computational,xiong2023over,xu2024provable} for more details. This implies that \gd cannot attain condition-number-free convergence guarantees without compromising linear convergence.

Next, let us turn attention to a simplified variant of \adam given by
\begin{align}
    \label{eq:signgd-mf}
    \textbf{(\signgd)}\qquad
    \mU_{t+1} = \mU_t - \eta_t \sign \big( (\mU_t\mU_t^{\top}-\mM^{\star}) \mU_t \big),
    \qquad t=0,1,\cdots 
\end{align}
where the sign function $\sign(\cdot)$ is applied entrywise. This algorithm (\ref{eq:signgd-mf}), which disables momentum in \adam, is also referred to as \signgd. Note that \signgd is more amenable to theoretical analysis than its momentum-based counterpart, while still capturing several core features of \adam like entrywise preconditioning \citep{bernstein2024old}. To demonstrate the provable advantage of \muon over  \signgd, we establish the following lower bound on the iteration complexity of (\ref{eq:signgd-mf}). 
\begin{theorem}
\label{thm::adam-lower-bound}
Let $r_0 \in (0, 1/16]$ be a universal constant. Consider the \signgd algorithm~(\ref{eq:signgd-mf}) with any non-increasing, positive learning rate sequence $\{\eta_t\}_{t \geq 0}$ satisfying $\eta_0\leq r_0$. Then, one can find a ground-truth matrix $\mM^\star$ with condition number $\kappa$, along with an initialization $\mU_0$ obeying $
\norm{\mU_0\mU_0^{\top} - \mM^\star}_{\fro} \leq r_0
$,  such that: for any given $\varepsilon \leq \frac{9r_0^2}{4096\kappa^2}$, $f(\mU_T)\leq \varepsilon$ cannot happen unless $$T\geq \frac{\kappa - 1}{4}.$$
\end{theorem}

In words, this lower bound demonstrates that  a momentum-free variant of \adam may incur at least a linear dependency on $\kappa$ in the iteration complexity. The proof of this lower bound is deferred to \Cref{sec:lower-bound-signgd-mf}.

\subsubsection{Intuition}

Thus far, we have established the  advantage of simplified \muon over a simplified variant of \adam (i.e., \signgd). In this subsection, we seek to provide some intuitive explanations about their differences in convergence rates. To streamline the presentation, we restrict our discussion  to the exactly-parameterized regime where $k = r$.

\paragraph{Decoupling of \muon dynamics into independent scalar sequences.} 
To build intuition for the working mechanism of simplified \muon, we adopt for the moment a simplifying assumption: 
\begin{align}
\label{eq:assumption-Ut-mf-intuition}
    \mU_t = \mV^{\star} \mSigma_t \mR^{\top}, 
    \qquad \text{for all }t\geq 1
\end{align}
for some diagonal matrix $\mSigma_t = \diag \{\sigma_{1,t},\dots,\sigma_{r,t}\} \in \mathbb{R}^{r\times r}$ and some orthonormal matrix $\mR \in \mathcal{O}_{r\times r}$. In words,  (\ref{eq:assumption-Ut-mf-intuition}) asserts that each \muon iterate $\mU_t$ has its singular subspace perfectly aligned with the true subspace $\mV^{\star}$. 
Although this assumption may appear overly restrictive at first glance, it will be approximately justified in our analysis in \Cref{sec:analyis-muon-mf}.  

Under this simplifying assumption (\ref{eq:assumption-Ut-mf-intuition}), the update rule (\ref{eq:muon-mf}) satisfies
\begin{align}
    \mU_{t+1} &= \mU_t - \eta_t \msign\big( (\mU_t\mU_t^{\top}-\mM^{\star})\mU_t\big) \notag\\
    &= \mV^{\star} \mSigma_t \mR^{\top}
     - \eta_t\msign\big( (\mV^{\star} \mSigma_t \mR^{\top}\mR \mSigma_t\mV^{\star\top} - \mV^{\star} \mLambda^{\star}\mV^{\star\top})\mV^{\star} \mSigma_t \mR^{\top} \big) \notag\\
     &= \mV^{\star} \mSigma_t \mR^{\top}
     - \eta_t\msign\big( \mV^{\star}( \mSigma_t^3   - \mLambda^{\star}\mSigma_t)  \mR^{\top} \big) \notag\\
     &=\mV^{\star} \mSigma_t \mR^{\top} - \eta_t\mV^{\star} 
       \dsign( \mSigma_t^3   - \mLambda^{\star}\mSigma_t)   \mR^{\top},
      \label{eq:Ut-iteration-intuition}
\end{align}
where for any diagonal matrix $\mSigma = \{\sigma_{1},\dots,\sigma_r\}$, we define $\diag(\mSigma) = \{\sign(\sigma_{1}),\dots,\sign(\sigma_r)\}$. If we write $\mU_{t+1}=\mV^{\star} \mSigma_{t+1} \mR^{\top}$ according to (\ref{eq:assumption-Ut-mf-intuition}), then it readily follows from \Cref{eq:Ut-iteration-intuition} that
\begin{align}
    \mSigma_{t+1} =
    \mSigma_t - \eta_t \dsign( \mSigma_t^3   - \mLambda^{\star}\mSigma_t). 
    \label{eq:Sigmat-iteration-intuition}
\end{align}
Crucially, all terms in \Cref{eq:Sigmat-iteration-intuition} are diagonal, thereby allowing it to be decomposed into $r$ independent scalar recursions:
\begin{align}
    \sigma_{i,t+1} = \sigma_{i,t} - \eta_t\sign\big( \sigma_{i,t}^3 - \lambda_i^{\star}\sigma_{i,t} \big),\qquad t=0,1,\cdots
    \label{eq:scalar-sequence-intuition}
\end{align}
for each $1\leq i\leq r$, 
each associated with one eigenvalue of $\mM^{\star}$. Noteworthily, the $r$ scalar sequences in (\ref{eq:scalar-sequence-intuition}) evolve completely independently, with no interaction across sequences. 

Owing to its simplicity, the scalar recursion in (\ref{eq:scalar-sequence-intuition}) admits a straightforward analysis. As we shall formally establish in \Cref{sec:muon-dynamics-scalar-mf}, elementary calculations give
\begin{align}
    \big| \sigma_{i,t+1}^2 - \lambda_i^{\star} \big|=O\big(\sqrt{\lambda_{\max}^\star}\eta_t\big) = O\big( \lambda_{\max}^{\star}\rho^t \big), 
\end{align}
provided that the learning rates decay exponentially as $\eta_t =C_{\eta} \sqrt{\lambda_{\max}^{\star}} \rho^t$. This linear convergence feature---with the convergence rate $\rho$ a numerical constant within $[1/2,1)$---mitigates the imbalance between large and small eigenvalues, thereby paving the way for condition-number-free convergence. 

This intuition further hints at a connection between \muon and the scaled gradient descent (\scaledgd) method \citep{tong2021accelerating}. We formalize this connection and discuss its implications in Appendix~\ref{sec::connection-muon-scaledgd}.

\paragraph{Why do \signgd and \adam fail?} As illustrated in \Cref{thm::adam-lower-bound}, the convergence rate of \signgd (a simplified variant of \adam) is sensitive to the condition number of $\mM^\star$. This arises because \signgd employs a {\em per-coordinate preconditioner}, which disregards the richer curvature structure of the problem and hence fails to adapt as effectively as \muon.

To see this more formally, denote by $\vu_t = \vecc(\mU_t)$ the flattened iterate, where $\vecc(\mZ)$ stacks the rows of a matrix $\mZ$ into a single column vector. Invoking the identities $\vecc(\mA \mX \mB) = (\mA \otimes \mB^\top)\vecc(\mX)$ and $\msign(\mZ)=\mZ(\mZ^{\top}\mZ)^{-1/2}$ for $\mZ\in \mathbb{R}^{d\times r}$, we can express the \muon update (\ref{eq:muon-mf}) as
\begin{equation}
\vu_{t+1} = \vu_t - \eta_t \big(\mI\otimes  ( \nabla f(\mU_t)^\top \nabla f(\mU_t) )^{-1/2}\big) \vecc\big(\nabla f(\mU_t)\big),
\end{equation}
where $\mI\otimes  ( \nabla f(\mU_t)^\top \nabla f(\mU_t) )^{-1/2}$ can be interpreted as a blockwise preconditioner.  Crucially, this preconditioning matrix is {\em not} diagonal, even in the limit when $\mU_t$ converges to the truth.

In contrast, \signgd and \adam employ diagonal preconditioners. For instance, the \signgd update (\ref{eq:signgd-mf}) can be expressed as
\begin{equation}
\vu_{t+1} = \vu_t - \eta_t\, \diag\left\{ |\vecc(\nabla f(\mU_t))|^{-1} \right\} \vecc\big(\nabla f(\mU_t)\big),
\end{equation}
where $|\bm{z}|^{-1}$ denotes the entrywise inverse of the entrywise magnitude of a vector $\bm{z}$. This diagonal preconditioner completely neglects cross-coordinate curvature. Consequently, \adam fails to adapt to the geometry of the matrix factorization problem, leading to slow convergence when $\mM^\star$ is ill-conditioned.

\subsubsection{Numerical experiments}
\label{sec:numerics-mf}
We now carry out a series of numerical experiments to validate the theoretical separation in convergence rates between \muon, \gd, and \signgd, with results displayed in \Cref{fig::matrix-factorization}.
In the top row (a–c) of \Cref{fig::matrix-factorization}, we investigate the impact of the condition number $\kappa \in \{1, 5, 25, 125, 625\}$, while fixing the matrix dimension to $d = 100$, target rank $r = 2$, and search rank $k = 2$.
In the bottom row (d–f) of \Cref{fig::matrix-factorization}, we evaluate the effect of search rank $k \in \{2, 3, 100\}$, fixing the condition number $\kappa = 1$, matrix dimension $d = 100$, and target rank $r = 2$.
All experiments adopt a more robust exponentially decaying learning rate schedule: the learning rate is reduced by a factor of $0.3$ if the loss does not decrease for $50$ consecutive iterations. 

Across all settings, \muon exhibits fast and stable convergence, reaching machine-level precision within a few hundred to a few thousand iterations—even under large condition numbers or severe rank over-specification. In contrast, both \gd and \signgd experience significant slowdowns as the condition number increases or the search rank grows. These results underscore the robustness of \muon vis-\`a-vis ill-conditioning and over-parameterization.
Moreover, while our theoretical guarantees for \muon require small initialization, we observe that in practice \muon converges robustly even with moderately sized initialization. In all experiments, we use an initialization scale of $\alpha = 0.1$. Rigorously elucidating why \muon remains stable and convergent under a broader range of initialization is an important direction for future work.

\begin{figure}[t]
     \centering
     \begin{subfigure}[b]{0.3\textwidth}
         \centering
         \includegraphics[width=\textwidth]{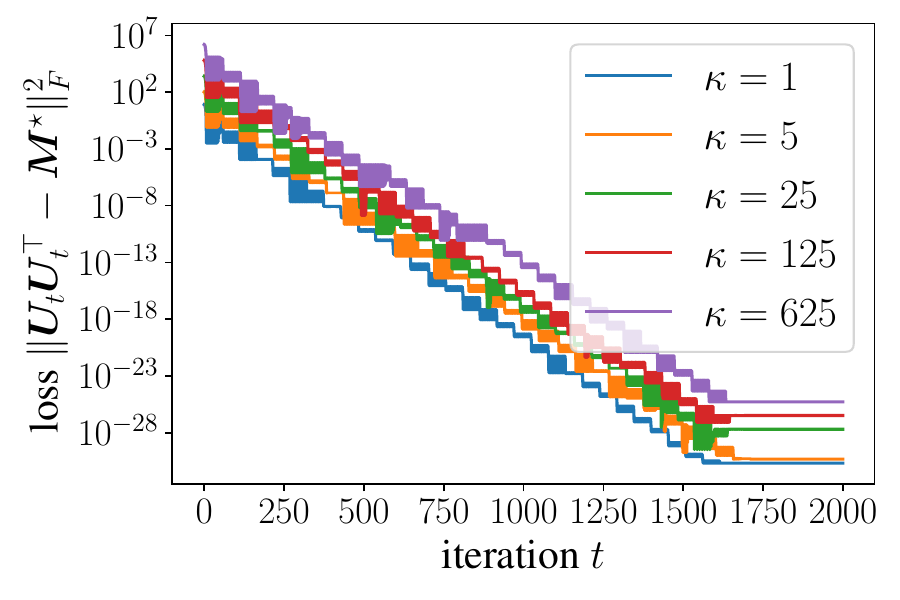}
         \caption{\muon}
         \label{fig::muon-condition_number}
     \end{subfigure}
     \begin{subfigure}[b]{0.3\textwidth}
         \centering
         \includegraphics[width=\textwidth]{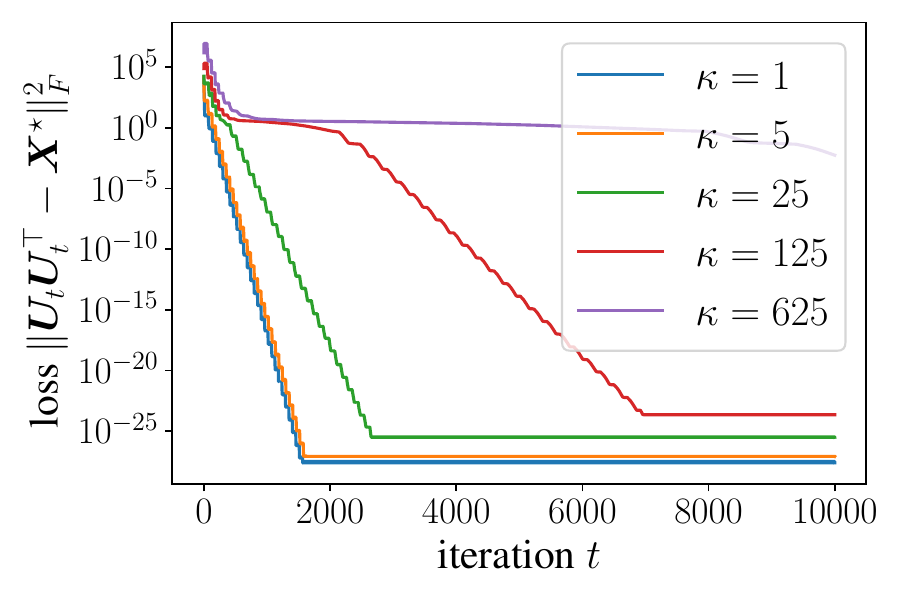}
         \caption{\signgd}
         \label{fig::adam-condition_number}
     \end{subfigure}
     \begin{subfigure}[b]{0.3\textwidth}
         \centering
         \includegraphics[width=\textwidth]{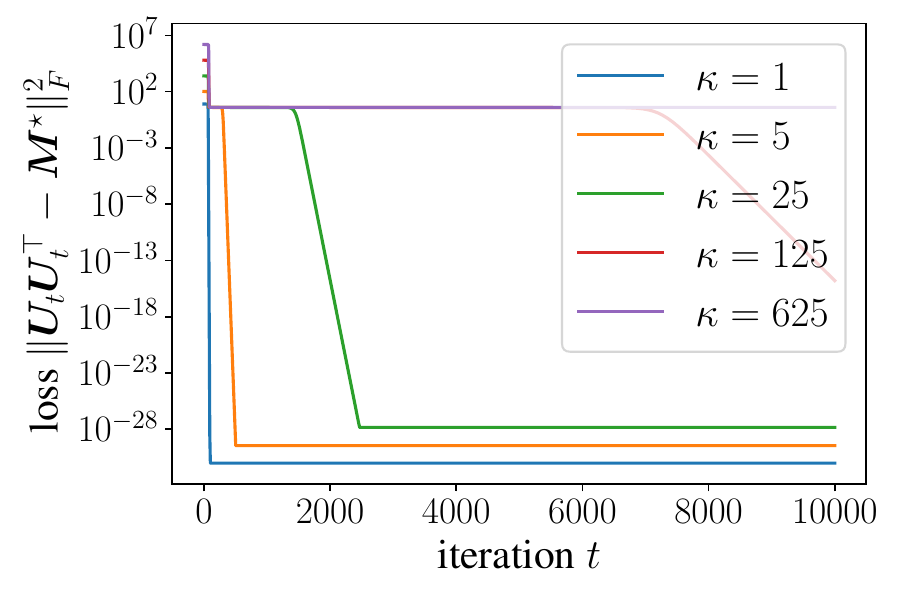}
         \caption{\gd}
         \label{fig::gd-condition_number}
     \end{subfigure}\\
     \begin{subfigure}[b]{0.3\textwidth}
         \centering
         \includegraphics[width=\textwidth]{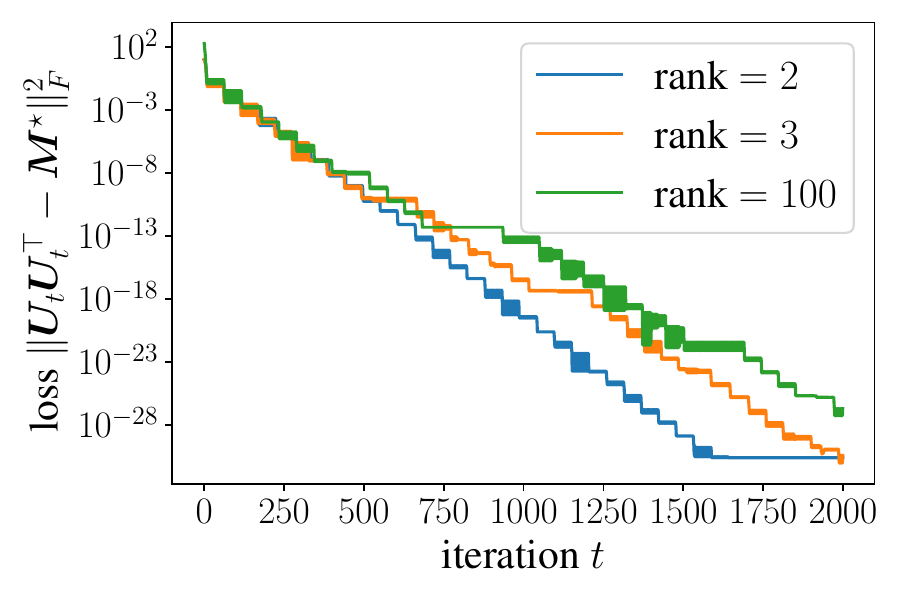}
         \caption{\muon}
         \label{fig::muon-rank}
     \end{subfigure}
     \begin{subfigure}[b]{0.3\textwidth}
         \centering
         \includegraphics[width=\textwidth]{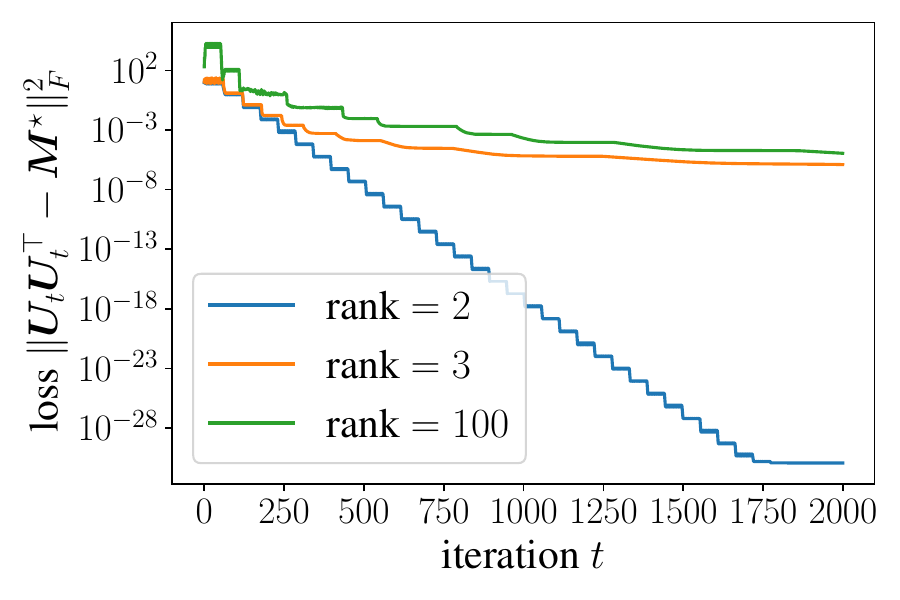}
         \caption{\signgd}
         \label{fig::adam-rank}
     \end{subfigure}
     \begin{subfigure}[b]{0.3\textwidth}
         \centering
         \includegraphics[width=\textwidth]{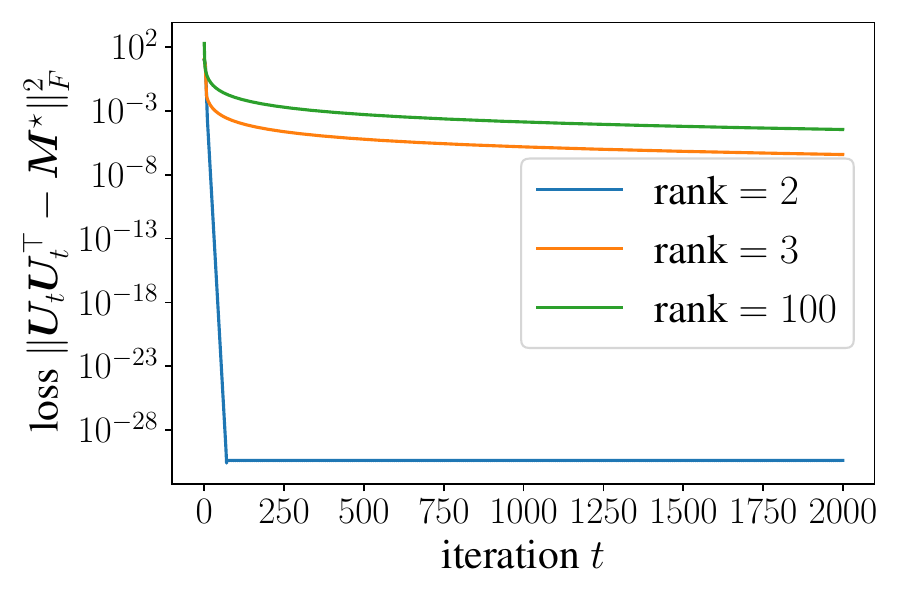}
         \caption{\gd}
         \label{fig::gd-rank}
     \end{subfigure}
        \caption{
    Numerical convergence behavior of \muon,  \signgd, and \gd on matrix factorization tasks under varying condition numbers and search ranks. 
    }
        \label{fig::matrix-factorization}
\end{figure}

\subsection{In-context learning with linear transformers}
\label{sec:main-results-transformer}

Next, we turn to the second case study, motivated by in-context learning with linear transformers. 
Let us first state the optimization problem before describing the motivation. Let $\{\vx_i\}_{i=1}^N \subseteq \bR^d$ be a fixed set of $N$ vectors. Define the empirical covariance matrix as
\begin{align}
    \mS \coloneqq \frac{1}{N}\sum_{i=1}^N \vx_i \vx_i^\top \label{eq-emp-cov},
\end{align}
which is assumed to be invertible throughout. 
We aim to solve the following optimization problem:
\begin{equation}
\label{eq:transformer-obj}
\mathop{\text{minimize}}\limits_{\mQ \in \bR^{d\times d}}\quad 
f(\mQ)
\coloneqq 
\frac{1}{2}\tr\!\big( (\mS\mQ-\mI)\mS(\mS\mQ-\mI)^\top \big).
\end{equation}
This is a simple quadratic optimization problem with $\mQ^\star = \mS^{-1}$ the minimizer. Letting $\kappa(\mS)$ denote the condition number of the matrix $\mS$, we see that the quadratic form induced by~(\ref{eq:transformer-obj})
has an effective condition number that scales as 
\begin{align}
\kappa \coloneqq  \kappa(\mS)^3.
\end{align}

\paragraph{Motivation: in-context learning of a single-layer linear transformer.}
In-context learning (ICL) refers to the phenomenon whereby a pretrained model can make predictions from a
\emph{prompt} on the fly \citep{brown2020language}. More specifically, the prompt contains a sequence of $N$ labeled examples (i.e., the context), followed by a query token,
and the model must infer the query label from the context at inference time without updating its parameters. 
Transformers~\citep{vaswani2017attention}
arise as a natural model class that supports ICL.
Here, we focus on a special case: in-context fixed-design linear regression, where the set of possible input vectors $\{\vx_i\}_{i=1}^N\subset\mathbb R^d$
is fixed with empirical covariance $\mS$,
and each task is indexed by a vector $\vw\in \bR^d$ with corresponding labels $y_{\vw,i}=\vw^\top \vx_i$.
At a high level, the context can be summarized by the vector
$\frac1N\sum_{i=1}^N y_{\vw,i}\vx_i=\mS\vw$.
Given a query $\vx_{\mathrm q}\in \mathbb R^d$, a simple in-context predictor uses a shared meta-parameter
$\mQ\in\mathbb R^{d\times d}$ to map the query to an effective readout $\mQ\vx_{\mathrm q}$,
and predicts via the bilinear form
\[
\widehat y_{\mathrm q} = (\mS\vw)^\top \mQ\vx_{\mathrm q}
= \vw^\top \mS \mQ \vx_{\mathrm q}.
\]
Averaging the squared prediction risk over tasks with $\bE[\vw]=\mathbf{0}$ and
$\bE[\vw\vw^\top]=\mI$, and over uniformly sampled queries  
$\vx_{\mathrm q}\sim\mathsf{Unif}\{\vx_1,\dots,\vx_N\}$, yields the expected loss that coincides with the objective function in (\ref{eq:transformer-obj}).
Moreover, this predictor can be realized by a single-layer \emph{linear} transformer (attention without softmax)
under a standard reparameterization~\citep{zhang2024trained,huang2023context}. 
See \Cref{app:icl-linear-transformer} for more details.

\paragraph{Convergence guarantees for \muon.} 
When applied to the optimization problem (\ref{eq:transformer-obj}), the update rule of the simplified \muon algorithm admits a closed-form expression as follows: 
\begin{equation}
\label{eq:muon-update-transformer}
\mQ_{t+1}=\mQ_t-\eta_t\, \msign\!\big(\mS^2\mQ_t\mS-\mS^2\big),
\qquad t=0,1,\cdots
\end{equation}
Encouragingly, this algorithm is guaranteed to converge linearly at a rate
independent of $\kappa$, as asserted by our theory below. 
\begin{theorem}
\label{thm::muon-transformer}
Let the initialization be $\mQ_0 = \mathbf{0}$ and set the learning rate schedule as $\eta_t = \frac{C_\eta}{\sigma_{\min}(\mS)} \rho^t$ for some  quantities $C_\eta \geq 1 $ and $\rho \in [1/2, 1)$. Then, for any $\varepsilon > 0$, 
 simplified \muon (\ref{eq:muon-update-transformer}) achieves
$ \| \mQ_T - \mQ^\star \| 
    = \| \mQ_T - \mS^{-1} \|\leq \varepsilon$
as long as 
\begin{align}
T \geq   \frac{1}{1 - \rho} \log \bigg(\frac{C_{\eta}}{\sigma_{\min}(\mS)\varepsilon}\bigg) .
\end{align}
\end{theorem}

This theorem establishes that the number of iterations needed for simplified \muon to yield $\varepsilon$-accuracy is independent of the condition number $\kappa$ underlying this quadratic optimization problem. Akin to the matrix factorization counterpart, the \muon dynamics admit a decomposition into a set of independent scalar sequences in the spectral domain, each evolving at a comparable rate of convergence irrespective of the magnitude of the associated eigenvalue, a feature that we shall rigorize in the proof presented in \Cref{sec:analysis-muon-transformer}.

\paragraph{Comparisons with other optimizers.}

To demonstrate the provable benefits of \muon compared against other optimizers, we discuss in this subsection the convergence rate of \gd and \signgd.

When applied to this problem (\ref{eq:transformer-obj}), \gd follows the update rule
\begin{align}
\textbf{(\gd)}\qquad
\mQ_{t+1}&=\mQ_t-\eta_t(\mS^2\mQ_t\mS-\mS^2),
\qquad t=0,1,\cdots
\label{eq:gd-update-transformer}
\end{align}
Given that this problem is a strongly convex quadratic optimization problem, classical optimization theory already reveals that the number of iterations needed for \gd to achieve $\varepsilon$-accuracy is lower bounded by (see, e.g., \cite{d2021acceleration})
$$\Omega\left(\sqrt{\kappa}\log(1/\varepsilon)\right).$$ 
This lower bound for \gd scales proportionally with $\sqrt{\kappa}$, unveiling the unavoidable dependency of its iteration complexity on the condition number.

We then switch attention to \signgd (recall that this is a variant of \adam with momentum turned off), which adopts the update rule 
\begin{align}
\textbf{(\signgd)}\qquad
\mQ_{t+1}&=\mQ_t-\eta_t\,\sign\!\big(\mS^2\mQ_t\mS-\mS^2\big),
\qquad t=0,1,\cdots \label{eq:signgd-update-transformer}
\end{align}
where the $\sign(\cdot)$ operator is applied entrywise. 

\begin{theorem}
\label{prop::lower-bound-transformer}
    Consider the \signgd algorithm (\ref{eq:signgd-update-transformer}) with any non-increasing, positive learning rate schedule $\{\eta_t\}_{t\geq 0}$. Consider any $0<\varepsilon\le \sqrt{2}\eta_0/\kappa$. 
    Then, 
  there exists an empirical covariance matrix $\mS$, along with an initialization $\mQ_0$, such that $\|\mQ_T-\mQ^{\star}\|_{\mathrm{F}}\leq \varepsilon$ cannot happen unless
  $$
    T \geq \frac{\kappa - 1}{4}. 
  $$
\end{theorem}
In words, \Cref{prop::lower-bound-transformer} rigorously establishes that the \signgd algorithm cannot achieve condition-number-free convergence for solving this problem, and is therefore substantially outperformed by \muon. 
The proof of \Cref{prop::lower-bound-transformer} is provided in \Cref{sec:lower-bound-signgd-transformer}.

\paragraph{Numerical experiments.}
We now evaluate and compare the numerical convergence performance of \muon, \signgd, and \gd on in-context learning tasks with one-layer linear transformers. We vary the condition number $\kappa \in \{1, 5, 25, 125, 625\}$ while fixing the matrix dimension to $d = 100$.
All experiments use an exponential decay learning rate schedule: the learning rate is reduced by a factor of $0.3$ whenever the loss fails to decrease for $50$ consecutive iterations.
\muon achieves rapid convergence across all condition numbers and reaches machine precision within a few hundred steps. In contrast, \signgd and \gd suffer from significantly slower rates, particularly under ill-conditioned settings, thereby validating the robustness and efficiency of \muon for ill-conditioned problems.

\begin{figure}
     \centering
     \begin{subfigure}[b]{0.3\textwidth}
         \centering
         \includegraphics[width=\textwidth]{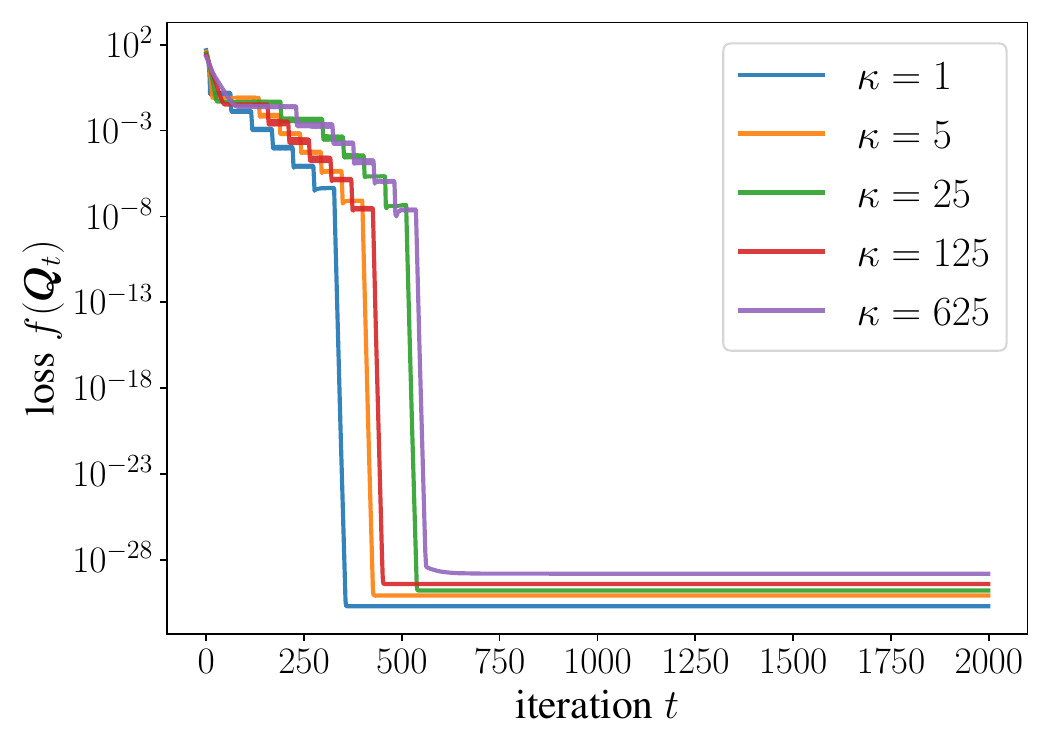}
         \caption{\muon}
         \label{fig::muon-condition_number-transformer}
     \end{subfigure}
     \begin{subfigure}[b]{0.3\textwidth}
         \centering
         \includegraphics[width=\textwidth]{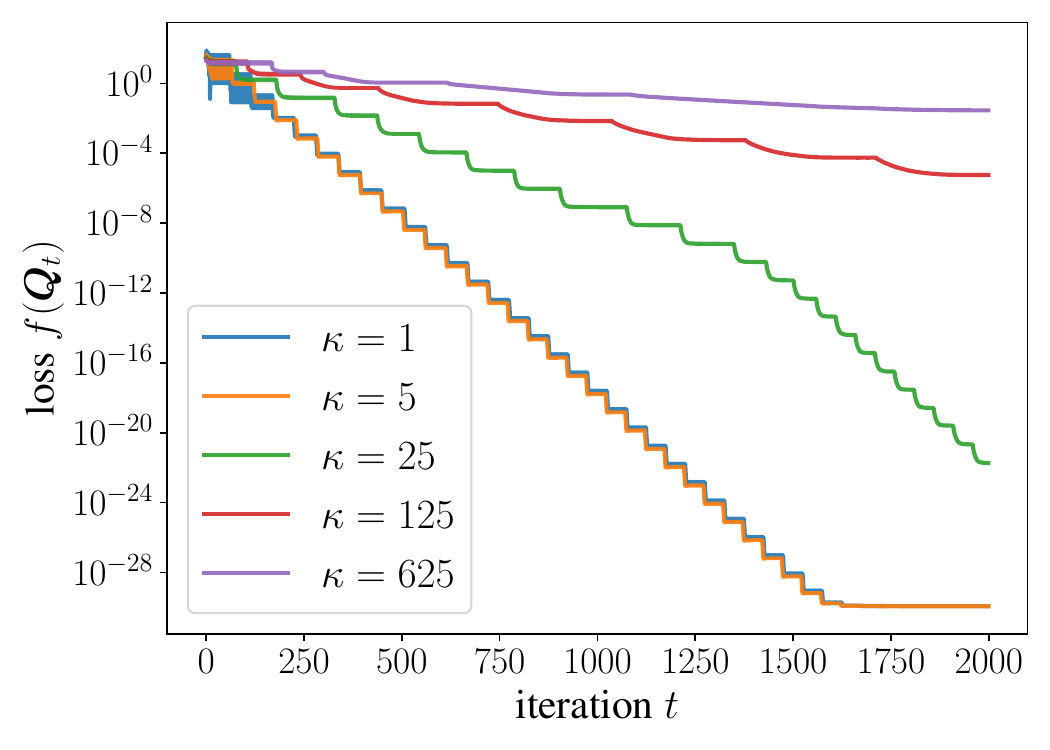}
         \caption{\signgd}
         \label{fig::adam-condition_number-transformer}
     \end{subfigure}
     \begin{subfigure}[b]{0.3\textwidth}
         \centering
         \includegraphics[width=\textwidth]{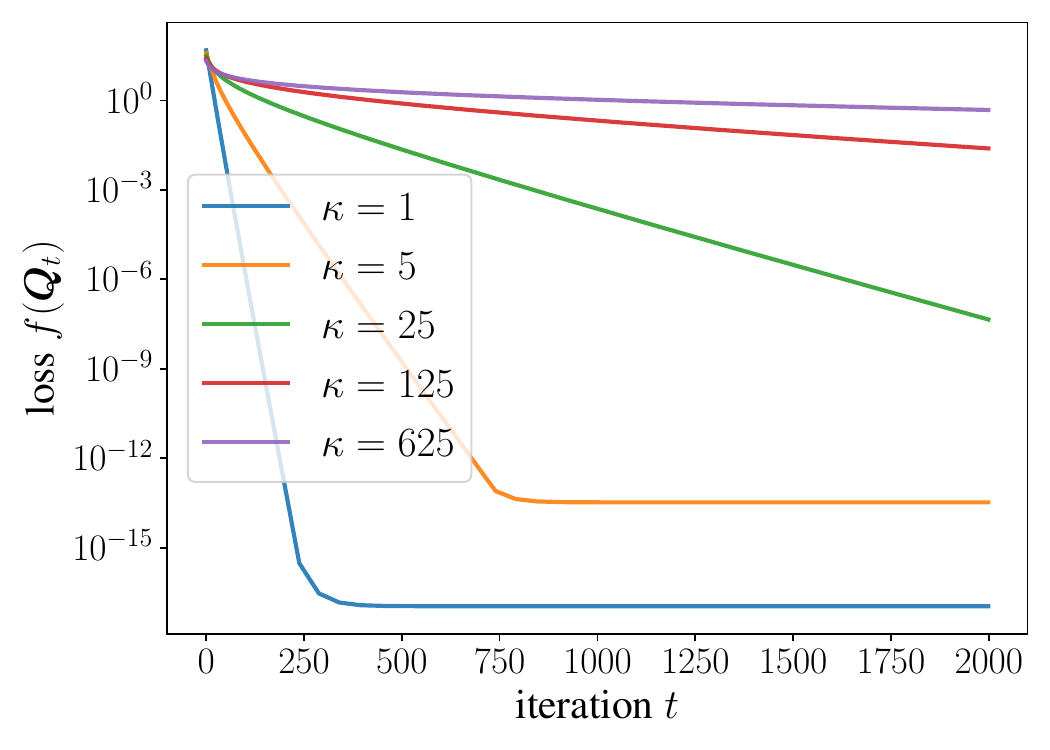}
         \caption{\gd}
         \label{fig::gd-condition_number-transformer}
     \end{subfigure}
        \caption{
    Numerical convergence behavior of \muon, \signgd, and \gd on in-context learning problems with one-layer linear transformers under varying condition numbers.
    }
        \label{fig::condition_numer-transformer}
\end{figure}

\section{Analysis for matrix factorization (proof of \Cref{thm::muon-main})}
\label{sec:analyis-muon-mf}

In this section, we establish our convergence guarantees for \muon applied to matrix factorization (i.e., \Cref{thm::muon-main}). 
Our analysis is structured into several parts. Firstly, we analyze the dynamics of \muon for a special scalar case. Secondly, 
building on this scalar recurrence analysis, we establish the desirable convergence assuming that $\bm{U}_t$ has its singular subspace perfectly aligned with $\mV^{\star}$. With these preparations in place, Steps 3 and 4 then prove the full convergence theory for the cases with $k\geq d$ and $r\leq k < d$, respectively.

\subsection{Step 1: dynamics of \muon in the scalar case}
\label{sec:muon-dynamics-scalar-mf}

Before delving into the general case, let us first consider a special case that aims at solving the following scalar optimization problem:
\begin{align}
\mathop{\text{minimize}}\limits_{u\in \bR}\quad (u^2-\lambda^\star)^2,
\label{eq:optimization-matrix-factorization-1d}
\end{align}
where $\lambda^{\star}\geq 0$. Evidently, this problem can be viewed as a 1-dimensional special case of  (\ref{eq:optimization-matrix-factorization}). The \muon algorithm (\ref{eq:muon-mf}) applied to (\ref{eq:optimization-matrix-factorization-1d}) follows the scalar dynamic below:
\begin{equation}
\label{eq:1d-muon-rec}
u_{t+1} = u_t - \eta_t \sign\left( (u_t^2 - \lambda^\star)u_t \right),
\qquad t=0,1,\cdots
\end{equation}
where $u_0 \in \mathbb{R}$ indicates the initialization.

In order to analyze the dynamics of (\ref{eq:1d-muon-rec}), we first demonstrate in the following lemma that with probability 1, the iterates $u_t$ never reach $0$, as long as $C_{\eta}$ is randomly generated. 
\begin{lemma}
\label{lem::never-reach-zero}
    Consider any update sequence taking the form of $u_{t+1} = u_t + \eta_ts_t $ for $t \ge 0$, 
    where $u_0 \neq 0$ is the initialization, and $s_t\in \{1, -1\}$ for all $t\geq 0$.  The learning rates are taken as $\eta_t = C_{\eta} \sqrt{\lambda^\star_{\max}} \rho^t$ for some $\lambda^\star_{\max} > 0$ and $\rho \in [1/2, 1)$, where the prefactor $C_{\eta}$ is uniformly sampled from the interval $[1, 2]$ and is independent of $u_0$.
    Then, with probability $1$, one has $u_t \neq 0$ for all $t \ge 0$.
\end{lemma}

The fact that $\{u_t\}$ never hits 0 eliminates the need to analyze this undesirable stationary point. 
We are now positioned to develop theoretical convergence guarantees for the scalar dynamics (\ref{eq:1d-muon-rec}).  

\begin{lemma}[Convergence of scalar \texttt{Muon}]
\label{lem::1d-muon}
Consider the scalar updates in (\ref{eq:1d-muon-rec}), where $0 \le \lambda^\star \le \lambda_{\max}^\star$. Set the learning rate schedule to be $\eta_t = C_{\eta}\sqrt{\lambda_{\max}^\star}\rho^t$
for some quantities $1/2 \le \rho < 1$ and $C_{\eta} \ge 1$. Assume that $0<|u_0| \le C_{\eta}\sqrt{\lambda_{\max}^\star} = \eta_0$.
Then, with probability $1$, for all $t \ge 0$, it holds that
\begin{subequations}
\begin{align}
\big| |u_{t+1}| - \sqrt{\lambda^\star} \big| &\le \eta_t \leq 2\sqrt{\lambda_{\max}^\star}\rho^t, \label{eq:1d-bound} \\
|u_{t+1}^2 - \lambda^\star|
&\leq 8\lambda_{\max}^\star\rho^{t}.
\label{eq:1d-linear}
\end{align}
\end{subequations}
\end{lemma}
In words, \Cref{lem::1d-muon} reveals that \muon converges linearly at a rate $\rho$ for this scalar case. 
Remarkably, analyzing this scalar case not only addresses this special setting, but also sheds light on the spectral dynamics underlying \muon for the more general case, as detailed in subsequent subsections.

\begin{proof}[Proof of \Cref{lem::never-reach-zero}]
    Regarding $t=0$, we have $u_0 \neq 0$ by assumption. For any $t \ge 1$, we can express $u_t$ by expanding the recurrence relation:
    \begin{equation}
        u_t = u_0 + \sum_{k=0}^{t-1} s_k \eta_k =u_0 + C_{\eta} \sqrt{\sigma^\star_{\max}} \left( \sum_{k=0}^{t-1} s_k \rho^k \right) 
        \eqqcolon u_0 + C_{\eta} S_t.
    \end{equation}
    If $S_t=0$, we have $u_t=u_0\neq 0$. Otherwise, the condition $u_t = 0$ is equivalent to
    $
        C_{\eta} = -{u_0}/{S_t}.
   $
   In other words, 
    for any given $t$ and any fixed sequence $\{s_k\}_{k=0}^{t-1}$, there exists exactly one value 
    of $C_{\eta}$ that can make $u_t$ equal $0$.

    Let $\mathcal{C}$ be the set containing all such critical values for all possible $t$ and $\{s_t\}$:
    \begin{equation}
        \mathcal{C} = \bigcup_{t=1}^{\infty} \bigcup_{s \in \{-1, 1\}^t} \left\{ -\frac{u_0}{\sqrt{\sigma^\star_{\max}} \sum_{k=0}^{t-1} s_k \rho^k} \,\bigg|\, \sum_{k=0}^{t-1} s_k \rho^k \neq 0 \right\},
    \end{equation}
    which is clearly a countable set given that the set of time steps and the set of possible sign sequences are both countable.  
Therefore, when $C_{\eta}$ is uniformly sampled from the interval $[1,2]$, the probability of this continuous random variable taking values in a countable set is 0, i.e., 
    \begin{equation}
        \mathbb{P}(\exists t\geq 0 : u_t = 0) \leq \mathbb{P}(C_{\eta} \in \mathcal{C}) = 0.
    \end{equation}
    Thus, it follows that, with probability $1$, $u_t \neq 0$ holds for all $t\geq 0$.
\end{proof}

\begin{proof}[Proof of \Cref{lem::1d-muon}]
 First,  \Cref{lem::never-reach-zero} tells us that with probability $1$, $u_t\neq 0$ for all $t\geq 0$. Moreover, if $u_t^2=\lambda^\star$, then the iterate has reached the optimal solution, and will stay unchanged thereafter. 
 Consequently, it suffices in the sequel to analyze the case where $(u_t^2 - \lambda^\star)u_t\neq 0$.
 
 To proceed, observe that
\begin{equation}
\label{eq:factor}
(u_t^2 - \lambda^\star)u_t = (u_t - \sqrt{\lambda^\star})(u_t + \sqrt{\lambda^\star})u_t.
\end{equation}
\begin{itemize}
\item  
If $u_t > 0$, then $u_t(u_t + \sqrt{\lambda^\star}) > 0$, and hence
\begin{subequations}
\label{eq:sign-pos-neg}
\begin{equation}
\label{eq:sign-pos}
\sign\left( (u_t^2 - \lambda^\star)u_t \right) = \sign(u_t - \sqrt{\lambda^\star}).
\end{equation}
\item 
If $u_t < 0$, then $u_t(u_t - \sqrt{\lambda^\star}) > 0$, and as a result, 
\begin{equation}
\label{eq:sign-neg}
\sign\left( (u_t^2 - \lambda^\star)u_t \right) = \sign(u_t + \sqrt{\lambda^\star}).
\end{equation}
\end{subequations}
\end{itemize}
\noindent 
This implies that in both of the above cases, the search direction is the sign of the difference between $u_t$ and its nearest root of $\lambda^{\star}$.
In light of this, we find it helpful to 
define 
\begin{align}
\Delta_t \coloneqq \big| |u_t| - \sqrt{\lambda^\star} \big|.
\label{eq:defn-Delta-t-scalar}
\end{align}
Making use of \Cref{eq:1d-muon-rec,eq:sign-pos-neg} allows one to easily verify that 
\begin{equation}
\label{eq:delta-rec}
\Delta_{t+1} = \big| \Delta_t - \eta_t \big| \le \max\{\Delta_t - \eta_t, \eta_t\}.
\end{equation}
Armed with this inequality, we are ready to prove the claim (\ref{eq:1d-bound}), which we accomplish by induction.

\begin{itemize}
\item \textit{Base case ($t=0$).}
Given that $\sqrt{\lambda^\star} \le \sqrt{\lambda_{\max}^\star}$ and $|u_0| \le \eta_0$, we have
\begin{equation}
\label{eq:delta0}
\Delta_0 = \big| |u_0| - \sqrt{\lambda^\star} \big|
\le |u_0| + \sqrt{\lambda^\star}
\le \eta_0 + \sqrt{\lambda_{\max}^\star}
\le 2\eta_0,
\end{equation}
where we have used $\eta_0=C_{\eta}\sigma_{\max}^{\star}$ for $C_{\eta}\geq 1$. 
Combining this with (\ref{eq:delta-rec}) at $t=0$ gives
\begin{equation}
\label{eq:delta1}
\Delta_1 \le \max\{\Delta_0 - \eta_0, \eta_0\} \le \max\{\eta_0, \eta_0\} = \eta_0,
\end{equation}
which establishes the claim (\ref{eq:1d-bound}) for $t=0$.

\item \textit{Inductive step.}
Assume $\Delta_{t+1} \le \eta_t$ for some $t \ge 0$. Then in view of \Cref{eq:delta-rec},
\begin{equation}
\label{eq:ind-step1}
\Delta_{t+2}
\le \max\{\Delta_{t+1} - \eta_{t+1}, \eta_{t+1}\}
\le \max\{\eta_t - \eta_{t+1}, \eta_{t+1}\}.
\end{equation}
Equipped with our assumptions $\eta_{t+1} = \rho \eta_t$ and $1/2\le \rho <1 $, we obtain
\begin{equation}
\label{eq:eta-gap}
\eta_t - \eta_{t+1} = (1-\rho)\eta_t \le \rho \eta_t = \eta_{t+1},
\end{equation}
which taken together with \Cref{eq:ind-step1} yields $$\Delta_{t+2} \le \eta_{t+1}.$$ 
This establishes  the claim (\ref{eq:1d-bound}) for iteration $t+2$, which in turn finishes the proof of the claim (\ref{eq:1d-bound}) for all $t\geq 0$ by induction.

\end{itemize}

Lastly, with inequality (\ref{eq:1d-bound}) in place, we can readily demonstrate that, 
for any $t \ge 0$, 
\begin{equation}
\label{eq:square-diff}
|u_{t+1}^2 - \lambda^\star|
= \big| |u_{t+1}| - \sqrt{\lambda^\star} \big| \big( |u_{t+1}| + \sqrt{\lambda^\star} \big)
\le \Delta_{t+1} \big( \Delta_{t+1} + 2\sqrt{\lambda^\star} \big)\leq 8\lambda_{\max}^\star\rho^{t}
\end{equation}
as claimed, 
where the last inequality holds since  $\lambda^\star\leq \lambda^\star_{\max}$ and $\Delta_{t+1}\leq \eta_{t}= C_{\eta}\sqrt{\lambda_{\max}^\star}\rho^{t} \leq 2\sqrt{\lambda_{\max}^\star}$. 
\end{proof}

\subsection{Step 2: dynamics of \muon with perfectly initialized column space}

Next, we extend our analysis beyond the scalar case to another special case involving a particular---albeit often impractical---choice of initialization. As will become clear momentarily, the general case is intimately connected to this special setting.

More precisely, suppose that the initialization can be decomposed as
\begin{align}
\label{eq:special-U0-mf}
    \mU_0 = \mV^{\star} \bm{\Sigma}_0 \bm{O}_{\mathsf{init}}^{\top}, 
\end{align}
where $\mSigma_0 = \diag\{\sigma_{1,0},\dots,\sigma_{r,0}\}$ is a diagonal matrix in $\mathbb{R}^{r\times r}$, 
and $\bm{O}_{\mathsf{init}}\in \mathbb{R}^{k\times r}$ is some arbitrary orthonormal matrix  with $k\geq r$  obeying $\bm{O}_{\mathsf{init}}^{\top}\bm{O}_{\mathsf{init}}=\mI_r$. 
Armed with this initialization, we can establish convergence guarantees of \muon by extending the scalar analysis in \Cref{lem::1d-muon}, as formalized in the lemma below.
\begin{lemma}
\label{lem:muon-mf-special-init}
Suppose that $\bm{U}_0$ satisfies (\ref{eq:special-U0-mf}). 
Then for all $t\geq 0$,  $\bm{U}_t$ can be decomposed as
\begin{subequations}
\label{eq:lemma-Ut-decomp-mf}
\begin{equation}
\label{eq:case1-form}
\mU_t = \mV^\star \mSigma_t \bm{O}_{\mathsf{init}}^{\top}
\qquad
\text{for some }\mSigma_t = \diag\{\sigma_{1,t},\dots,\sigma_{r,t}\} \in \bR^{r\times r}. 
\end{equation}
In particular, for every $t\geq 0$ and $1\leq i\leq r$, one has 
\begin{align}
\sigma_{i,t+1} = \sigma_{i,t} - \eta_{t} \sign\big( (\sigma_{i,t}^2 - \lambda_i^\star)\sigma_{i,t} \big).
\label{eq:case1-Sigma-iterative-updates}
\end{align}
\end{subequations}
\end{lemma}
Importantly, \Cref{lem:muon-mf-special-init} reveals that: if the initialization has its left singular subspace perfectly aligned with the desired $\mV^{\star}$, then along the entire trajectory, the ``spectrum'' of each \muon iterate decouples into $r$ scalar sequences, each resembling the dynamics analyzed in \Cref{lem::1d-muon}. 
Therefore, invoking \Cref{lem::1d-muon} yields
\begin{subequations}
\label{eq:coord-bound-case1-final}
\begin{equation}
\label{eq:coord-bound}
|\sigma_{i,t+1}^2 - \lambda_i^\star| \le 8\lambda_{\max}^\star\rho^{t}
\end{equation}
for all $t\geq 0$, 
with the proviso that $|\sigma_{i,0}|\leq \eta_0$ for all $1\leq i\leq r$. 
Taking this collectively with property~(\ref{eq:case1-form}) leads to the following convergence bound for all $t\geq 0$: 
\begin{align}
\big\| \mU_{t+1}\mU_{t+1}^\top - \mM^\star \big\|
&= \big\| \mV^{\star}\mSigma_{t+1}^2\mV^{\star\top} - \mV^{\star} \mLambda^{\star}\mV^{\star\top} \big\|= \max_{1 \le i \le r} |\sigma_{i,t+1}^2 - \lambda_i^\star|
\le 8\lambda_{\max}^\star\rho^{t}.
\label{eq:case1-final}
\end{align}
\end{subequations}

\begin{proof}[Proof of \Cref{lem:muon-mf-special-init}]
Let us prove this lemma by induction. 
\begin{itemize}
\item {\em Base case with $t=0$.}
This holds trivially given our assumption (\ref{eq:special-U0-mf}). 

\item {\em Inductive step.}
Assuming the induction hypothesis (\ref{eq:case1-form}) holds at time $t$, we can compute the gradient as
\begin{equation}
\label{eq:grad-case1}
\nabla f(\mU_t)
= (\mU_t\mU_t^\top - \mM^\star)\mU_t
= \mV^\star (\mSigma_t^2 - \mLambda^\star)\mSigma_t \bm{O}_{\mathsf{init}}^{\top}.
\end{equation}
Given that both $\mSigma_t$ and $\mLambda^{\star}$ are diagonal matrices, the matrix sign of $\nabla f(\mU_t)$ is given by
\begin{equation}
\label{eq:sign-case1}
\msign\big(\nabla f(\mU_t)\big)
= \mV^\star \dsign\big( (\mSigma_t^2 - \mLambda^\star)\mSigma_t \big) \bm{O}_{\mathsf{init}}^{\top}.
\end{equation}
Here, we recall that $\dsign(\bm{D})=\diag\{\sign(D_{1,1}),\dots,\sign(D_{r,r})\}$ for any diagonal matrix $\bm{D}=\diag\{D_{1,1},\dots,D_{r,r}\}$.  
As a consequence, 
\begin{equation}
\label{eq:update-case1}
\mU_{t+1}
= \mU_t - \eta_t \msign\big(\nabla f(\mU_t)\big)
= \mV^\star \left( \mSigma_t - \eta_t \dsign\big( (\mSigma_t^2 - \mLambda^\star)\mSigma_t \big) \right)\bm{O}_{\mathsf{init}}^{\top}.
\end{equation}
Thus, this validates the claim (\ref{eq:case1-form}) for $t+1$ and demonstrates that
$$
    \mSigma_{t+1} = \mSigma_t - \eta_t \dsign\big( (\mSigma_t^2 - \mLambda^\star)\mSigma_t \big), 
$$
as claimed in (\ref{eq:case1-Sigma-iterative-updates}). 
\end{itemize}
The proof is thus complete by induction.  
\end{proof}

\subsection{Step 3: analysis for the case with $k\geq d$}

Turning to the general case, we begin by analyzing the scenario with $k\geq d$. 
In this setting, we find it convenient to work with the decomposition $\mM^\star = \mV^\star \mLambda^\star \mV^{\star\top}$ with $\mLambda^\star = \diag\{\lambda_1^\star,\dots,\lambda_d^\star\}$, where we take $r=d$ and allow some of the eigenvalues in $\{\lambda_1^\star,\dots,\lambda_d^\star\}$ to be zero.

Recall the initialization $\mU_0 = \alpha \mO$, where $\mO \in \mathbb{R}^{d\times k}$ is an arbitrary orthonormal matrix obeying $\bm{O}\bm{O}^{\top}=\mI_d$ and $\alpha \le \eta_0 $. One can express $\mU_0$ alternatively as
\begin{equation}
    \mU_0 = \alpha \mO = \mV^{\star} (\alpha \bm{I}_d) \mV^{\star\top} \mO \eqqcolon 
    \mV^{\star} (\alpha \bm{I}_d)  \mO_{\mathsf{init}}^{\top},
\end{equation}
where $
\mO_{\mathsf{init}}^{\top}\mO_{\mathsf{init}}= \mV^{\star\top} \mO \mO^{\top} \mV^{\star} = \bm{I}_d.
$
This indicates that the initialization $\mU_0$ satisfies Condition~(\ref{eq:special-U0-mf}). Therefore, by applying \Cref{lem:muon-mf-special-init} and inequality (\ref{eq:case1-final}), we see that with probability 1, 
 $\norm{\mU_T\mU_T^\top - \mM^\star} \le \varepsilon$ holds as long as
\[
T> \frac{1}{1-\rho} \log \bigg( \frac{8\lambda_{\max}^{\star}}{\varepsilon} \bigg). 
\]

\subsection{Step 4: analysis for the case with $r\leq k< d$}

We now switch attention to the case with $r\leq k< d$, which is substantially more challenging to analyze than the preceding setting. Here, we shall employ a random orthonormal initialization $\mU_0 = \alpha \mO$ obeying $\mO^{\top} \mO = \mI_k$. Our proof arguments unfold in several steps, as described below.

\paragraph{Step 4.1: initial subspace alignment.} A key property that we would like to establish is that: after the first \muon iteration, $\mU_1$ is already well aligned with the eigenspace $\mV^{\star}$.
Note that when initialized at  $\mU_0=\alpha\mO$, the gradient takes the following form
\begin{equation}
    \mG_0 \coloneqq\nabla f(\mU_0)=(\mU_0\mU_0^\top - \mM^\star)\mU_0 = \underset{\eqqcolon\, \mQ}{\underbrace{-\alpha \mM^\star \mO}} + \alpha^3 \mO, 
\end{equation}
where $\mQ$ denotes the leading term for small enough $\alpha$.  
Let us decompose $\mG_0$ into two components as $$\mG_0 = \mG_{0, \leq r} + \mG_{0, > r},$$ 
where $\mG_{0, \leq r}$ is the best rank-$r$ approximation of $\mG_0$ (i.e., it is composed of the $r$ leading singular components of $\mG_0$), and  $\mG_{0, > r}$ consists of the remaining $k-r$ singular components.  Given that $\mG_{0, \leq r}$ and $ \mG_{0, > r}$ are orthogonal to each other, the matrix sign of $\mG_0$ admits the following decomposition: 
\begin{equation}
    \msign(\mG_0)=\msign(\mG_{0, \leq r})+\msign(\mG_{0, > r}).
\end{equation}

As it turns out, \( \msign(\mG_{0, \leq r}) \) and \( \msign(\mQ) \) can be fairly close for small enough $\alpha$, as asserted by the following lemma. The proof is postponed to \Cref{sec:proof:lem:diff-msign-Q-G}. 
\begin{lemma}
    \label{lem:diff-msign-Q-G}
    There exists some universal constant $c_0>0$ such that, with probability at least 0.995, 
    \begin{equation}
    \big\| \msign(\mG_{0, \leq r})-\msign(\mQ) \big\|
    \leq \frac{16\alpha^2\sqrt{dr}}{c_0\lambda_r^\star}
\end{equation}
holds as long as $4\alpha^{2}\leq c_{0}\lambda_{r}^{\star}/\sqrt{{dr}}$. 

\end{lemma}

In addition, given that $\mG_{0, \leq r}$ and $\mG_{0, >r}$ are orthogonal to each other,  \Cref{lem:orth-completion} in \Cref{sec:technical-lemmas} reveals the existence of a matrix $\widetilde{\mG}_0\in \bR^{d\times k}$ such that %
\begin{subequations}
\label{eq:construction-Gtilde-0}
\begin{align}
\widetilde{\mG}_{0, \leq r}&=\mQ
\qquad \text{and} \\
    \big\|\msign(\mG_0)-\msign(\widetilde{\mG}_0)\big\| &\leq \sqrt{2}\,\big\|\msign(\mG_{0, \leq r})-\msign(\mQ) \big\|
    \leq \frac{16\sqrt{2}\alpha^2\sqrt{dr}}{c_0\lambda_r^\star}.
\end{align}
\end{subequations}
Taking this together with the first iteration $\mU_1 = \alpha \mO - \eta_0 \msign(\mG_0)$ leads to
\begin{subequations}
\label{eq:auxiliary-U1-R10}
\begin{equation}
    \mU_1 = \alpha \mO - \eta_0 \big(\msign(\widetilde{\mG}_0) + \bm{R}_0\big) = - \eta_0 \msign(\widetilde{\mG}_0) + \bm{R}_1,
    \label{eq:U1-muon-residual-mf}
\end{equation}
where the residual terms $\mR_0,\mR_1$ satisfy
\begin{align}
    \|\mR_0\| \leq \frac{16\sqrt{2}\alpha^2\sqrt{dr}}{c_0\lambda_r^\star}
    \qquad \text{and} \qquad \|\mR_1\|\leq 
    \alpha + \frac{16\sqrt{2}\eta_0\alpha^2\sqrt{dr}}{c_0\lambda_r^\star}
    \leq 2\alpha, 
    \label{eq:U1-muon-residual-mf-size}
\end{align}
\end{subequations}
provided that $\alpha\leq {c_{0}\lambda_{r}^{\star}}/(32\sqrt{2\lambda_{\max}^{\star}dr})$.

\paragraph{Step 4.2: construction of an auxiliary trajectory.} To facilitate analysis, we find it helpful to construct an auxiliary trajectory $\{\widetilde{\mU}_t\}_{t\geq 1}$ as follows:
\begin{subequations}
\label{eq:auxiliary-muon-mf}
\begin{align}
\widetilde{\mU}_1 &=- \eta_0 \msign(\widetilde{\mG}_0), 
\label{eq:auxiliary-muon-mf-init}\\
    \widetilde{\mU}_{t+1}&=\widetilde{\mU}_t-\eta_t\msign\big(\nabla f(\widetilde{\mU}_t)\big),\qquad t=1,2,\cdots  
    \label{eq:auxiliary-muon-mf-init-update}
\end{align}
\end{subequations}
In words, this auxiliary trajectory is also generated by simplified \texttt{Muon} in (\ref{eq:muon-mf}), but with a slightly modified initialization that discards the residual term $\mR_1$ appearing in the original iteration (\ref{eq:U1-muon-residual-mf}). In particular, it follows from (\ref{eq:U1-muon-residual-mf-size}) that
\begin{align}
\big\|\widetilde{\mU}_1-\mU_1\big\|=\norm{\mR_1}\leq 2\alpha. 
\label{eq:U1-U1tilde-gap}
\end{align}

Next, we demonstrate that the dynamics of this auxiliary trajectory $\{\widetilde{\mU}_t\}_{t\geq 1}$ can be decomposed into a collection of independent scalar dynamics, akin to Step 2.  To see this, we first claim that with high probability, $\msign(\mQ)$ can be decomposed as
\begin{align}
    \msign(\mQ) = \mV^{\star}\mO^{\prime\top}
    \label{eq:msign-Q-mf}
\end{align}
for some matrix $\mO^{\prime}\in \mathbb{R}^{k\times r}$ obeying $\mO^{\prime\top}\mO^{\prime}=\mI_r$. 
\begin{proof}[Proof of property~(\ref{eq:msign-Q-mf})]
Observe that
\begin{equation*}
-\msign(\mQ)= 
\msign(\mM^\star\mO) =\msign(\mV^\star\mLambda^\star \mB)=\mV^\star \mLambda^\star \mB(\mB^\top\mLambda^{\star 2}\mB)^{\dagger/2}=\mV^\star \msign(\mLambda^\star \mB),
\end{equation*}
where we take $\mB=\mV^{\star\top}\mO$. \Cref{lem::init} asserts that with probability at least $0.995$, $\sigma_r(\mLambda^\star\mB)>0$, thus implying that $\big(\msign(\mLambda^\star \mB)\big)^{\top}\in \cO_{k\times r}$.
This completes the proof.
\end{proof}
Armed with this property, we can readily repeat the analysis in Step 2 to establish convergence guarantees for $\{\widetilde{\mU}_t\}$. It can be easily seen from (\ref{eq:msign-Q-mf}) and our construction of $\widetilde{\mG}_0$ that:    
there exist two orthonormal matrices $\mV\in \cO_{d\times k}$ and $\mR\in \cO_{k\times k}$ such that 
\begin{align}
    \mV_{:,1:r}=\mV^\star, 
    \qquad \mR_{:,1:r}=\mO^{\prime}, \qquad
    \text{and} \qquad
    \msign(\widetilde{\mG}_0)=\mV\mR^{\top},
    \label{eq:choice-V-R-mf}
 \end{align}
where $\mM_{:,1:r}$ denotes the first $r$ columns of a matrix $\mM$. 
In the meantime, Condition~(\ref{eq:choice-V-R-mf}) allows us to express the ground truth as
\begin{align}
    \mX^{\star} = \mV^{\star} \mLambda^{\star} \mV^{\star\top}
    = \mV \mLambda_{\mathsf{aug}} \mV^{\top},
    \label{eq:Xstar-aug-mf}
\end{align}
where $\mLambda_{\mathsf{aug}}\in \mathbb{R}^{k\times k}$ is an augmented diagonal matrix $\mLambda_{\mathsf{aug}}=\diag\{\lambda_1^{\star},\dots,\lambda_k^{\star}\}$ with $\lambda_{r+1}^{\star}=\dots=\lambda_k^{\star}=0$.

Recall that $\widetilde{\mU}_1=-\eta_0\msign(\widetilde{\mG}_0)=-\eta_0 \mV\mR^{\top}$ (cf.~(\ref{eq:auxiliary-muon-mf-init})). 
Combining this together with \Cref{eq:Xstar-aug-mf}, we can readily invoke \Cref{lem:muon-mf-special-init} to show that: for each $t\geq 1$,  
$\widetilde{\bm{U}}_t$ can be decomposed as
\begin{subequations}
\label{eq:tilde-Ut-decompose-mf}
\begin{equation}
\widetilde{\mU}_t = \mV \widetilde{\mSigma}_t \bm{R}^{\top}
\qquad
\text{for some }\mSigma_t = \diag\{\widetilde{\sigma}_{1,t},\dots,\widetilde{\sigma}_{k,t}\} \in \bR^{k\times k},
\end{equation}
where $|\widetilde{\sigma}_{i,1}|= \eta_0$ ($1\leq i\leq k$), and for every $t\geq 1$ and $1\leq i\leq k$,  
\begin{align}
\widetilde{\sigma}_{i,t+1} = \widetilde{\sigma}_{i,t} - \eta_{t} \sign\big( (\widetilde{\sigma}_{i,t}^2 - \lambda_i^\star)\widetilde{\sigma}_{i,t} \big).
\end{align}
\end{subequations}
Repeating the same convergence analysis for (\ref{eq:coord-bound-case1-final}) tells us that (see \Cref{lem::1d-muon-varying-C} for a slight extension that accounts for random learning rates): for every $t\geq 0$ and every $1\leq i\leq k$ one has 
\begin{subequations}
\label{eq:auxiliary-convergence-mf}
\begin{align}
|\widetilde{\sigma}_{i,t+1}^2 - \lambda_i^\star| &\le \frac{8}{(1-\rho)^2}\lambda_{\max}^\star\rho^{t}, 
\label{eq:auxiliary-convergence-mf-sigma}\\
\big\|\widetilde{\mU}_{t+1}\widetilde{\mU}_{t+1}^\top-\mM^\star\big\| &\le \frac{8}{(1-\rho)^2}\lambda_{\max}^\star\rho^{t}.
\label{eq:auxiliary-convergence-mf-Sigma}
\end{align}
\end{subequations}
Consequently, one achieves 
\begin{equation}
    \big\|\widetilde{\mU}_{T}\widetilde{\mU}_{T}^\top-\mM^\star\big\| \leq \varepsilon / 2
\end{equation}
as long as 
$T> \frac{1}{1-\rho} \log \big( \frac{16\lambda_{\max}^{\star}}{(1-\rho)^2\varepsilon} \big)$.

\paragraph{Step 4.3: proximity between the original and auxiliary trajectories.} 
With the desirable convergence property of $\{\widetilde{\mU}_t\}$ in place, it remains to show that the original iterates $\{\mU_t\}$ remain close to the auxiliary iterates.  
First, we bound the differences between $\mU_t$ and $\widetilde{\mU}_t$, as well as between their associated gradients, in the following lemma; the proof is deferred to \Cref{sec:proof-lem:diff-Ut-Utilde-t-grad-mf}. 
\begin{lemma}
    \label{lem:diff-Ut-Utilde-t-grad-mf}
    Assume that $\sigma_{\min}(\nabla f(\mU_t)),\sigma_{\min}(\nabla f(\widetilde{\mU}_{t}))>0$. Then it holds that
    \begin{align}
        \big\|\mU_{t+1}-\widetilde{\mU}_{t+1}\big\|&\leq
        \left(1+\eta_t \frac{147\lambda_{\max}^\star}{(1-\rho)^2\sigma_{\min}(\nabla f(\widetilde{\mU}_{t}))}\right)\big\|\mU_{t}-\widetilde{\mU}_{t}\big\|.
        \label{eq::30-U}
    \end{align}
    In addition, we have $\max\big\{\norm{\mU_t}, \| \widetilde{\mU}_t\|\big\}\leq \frac{4\sqrt{\lambda_{\max}^\star}}{1-\rho}$.
\end{lemma}

Repeating the analysis of \Cref{lem::never-reach-zero} (which we omit here for brevity), we can easily see that with probability one, $\sigma_{\min}(\nabla f(\mU_t))>0$ and $\sigma_{\min}(\nabla f(\widetilde{\mU}_{t}))>0$ hold for all $t\geq 1$. \Cref{lem:diff-Ut-Utilde-t-grad-mf} then tells us that
\begin{align}
        \big\|\mU_{T}-\widetilde{\mU}_{T}\big\|&\leq
        \prod_{t=1}^{T-1}\left(1+\eta_t \frac{147\lambda_{\max}^\star}{(1-\rho)^2\sigma_{\min}(\nabla f(\widetilde{\mU}_{t}))}\right)\big\|\mU_{1}-\widetilde{\mU}_{1}\big\| \notag\\
        &\leq \underset{\eqqcolon\, \Pi_T}{\underbrace{\Bigg\{ \prod_{t=1}^{T-1}\left(1+ \frac{294\lambda_{\max}^{\star 3/2}\rho^t}{(1-\rho)^2\sigma_{\min}(\nabla f(\widetilde{\mU}_{t}))}\right) \Bigg\}}} \big\|\mU_{1}-\widetilde{\mU}_{1}\big\| \notag\\
        &\leq 2\alpha \Pi_T ,
    \label{eq:defn-PiT-UT-bound}
\end{align}
where we have used  $\eta_t=C_{\eta,t}\sqrt{\lambda_{\max}^{\star}}\rho^t\leq 2\sqrt{\lambda_{\max}^{\star}}\rho^t$ as well as (\ref{eq:U1-U1tilde-gap}).

In order to invoke (\ref{eq:defn-PiT-UT-bound}) to control  $\big\|\mU_{T}-\widetilde{\mU}_{T}\big\|$ and $\Pi_T$, 
a crucial step is to lower bound $\sigma_{\min}(\nabla f(\widetilde{\mU}_t))=\min_{1\leq i\leq k}|(\widetilde{\sigma}_{i,t}^2 - \lambda_i^\star)\widetilde{\sigma}_{i,t}|$, as accomplished by the following lemma. See \Cref{sec:proof:lem::small-ball-prob} for the proof. 

\begin{lemma}
\label{lem::small-ball-prob}
Consider any $0 < \varepsilon \leq \frac{1}{4} (\lambda_{\min}^{\star})^{3/2}$. 
   Then for every step $t\geq 1$, we have
    \begin{equation}
        \bP\left(\sigma_{\min}\big(\nabla f(\widetilde{\mU}_{t+1})\big)\leq \varepsilon\mid \mathcal F_{t}\right)\leq \frac{2(k-r)\sqrt[3]{\varepsilon}}{\sqrt{\lambda_{\max}^\star}\rho^t}+\frac{12r\varepsilon}{\lambda_{\min}^{\star}\sqrt{\lambda_{\max}^\star}\rho^t},
    \end{equation}
    where $\mathcal F_t$ represents all events that happen up to and including time $t$.
\end{lemma}

One can then exploit \Cref{lem::small-ball-prob} to establish high-probability upper bounds on the quantity $\Pi_T$ defined in (\ref{eq:defn-PiT-UT-bound}). The resulting bounds are stated in the lemma below, whose proof is deferred to \Cref{sec:proof:lem:prod-mgf}. 
\begin{lemma}
\label{lem:prod-mgf}
Consider any $\delta\in(0,1)$.
Then, with probability at least $1-\delta$, the following results hold. 

\begin{itemize}
\item[(i)] If $k=r$, then
\begin{equation}
    \Pi_T \le
\exp\left(
O\left(T\log\Big(\frac{r\kappa}{1-\rho}\Big) + \log\frac{1}{\delta}\right)
\right).
\end{equation}

\item[(ii)] If $k>r$, then
\begin{equation}
    \Pi_T \le
\exp\left(O\left(T^2+T\log\left(\frac{(k-r)r\kappa}{1-\rho}\right)
+
\log\frac{1}{\delta}\right)
\right).
\end{equation}
\end{itemize}
In particular, if $T=\left\lceil\frac{1}{1-\rho} \log \big( \frac{16\lambda_{\max}^{\star}}{(1-\rho)^2\varepsilon} \big) \right\rceil$ and $\delta=\mathsf{poly}(\varepsilon/\lambda_{\max}^{\star})$, then with probability at least $1-\delta$ one has
\begin{equation}
\Pi_{T}\leq\bigg(\frac{\lambda_{\max}^{\star}}{\varepsilon}\bigg)^{\zeta_{\mathrm{exp}}}~\text{with }\zeta_{\mathrm{exp}}=\begin{cases}
O\left(\frac{1}{1-\rho} \log\left(\frac{\lambda_{\max}^{\star}}{(1-\rho)\varepsilon}\right)\log\left(\frac{r\kappa}{1-\rho}\right)\right), & \text{if }k=r.\\
O\left(\frac{1}{(1-\rho)^{2}}\log^{2}\left(\frac{\lambda_{\max}^{\star}}{(1-\rho)\varepsilon}\right)+\frac{1}{1-\rho}\log\left(\frac{\lambda_{\max}^{\star}}{(1-\rho)\varepsilon}\right)\log\left(\frac{(k-r)r\kappa}{1-\rho}\right)\right), & \text{if }k>r.
\end{cases}\label{eq:PiT-exponent-UB-special}
\end{equation}
\end{lemma}
Taking this lemma together with (\ref{eq:defn-PiT-UT-bound}) yields: 
with probability at least $1- 0.001(\varepsilon/\lambda_{\max}^{\star})$, we have
\begin{equation}
    \label{eq:UtUt-Utilde-Utilde-ub-mf}
    \big\|\mU_T-\widetilde{\mU}_T\big\|
    \leq 2\alpha \bigg(\frac{\lambda_{\max}^{\star}}{\varepsilon}\bigg)^{\zeta_{\mathrm{exp}}}
    \leq \frac{(1-\rho)\varepsilon}{16\sqrt{\lambda_{\max}^\star}}.
\end{equation}
provided that 
\begin{align}
\alpha 
\leq \frac{(1-\rho)\varepsilon}{32\sqrt{\lambda_{\max}^\star}} \bigg(\frac{\varepsilon}{\lambda_{\max}^{\star}}\bigg)^{\zeta_{\mathrm{exp}}}.
\end{align}

\paragraph{Step 4.4: putting everything together.} 
Invoking (\ref{eq:auxiliary-convergence-mf-Sigma}),
(\ref{eq:UtUt-Utilde-Utilde-ub-mf}) and \Cref{lem:diff-Ut-Utilde-t-grad-mf}, and
applying the union bound, 
we conclude that with probability at least $0.99$,
\begin{equation}
    \begin{aligned}
        \big\|\mU_T\mU_T^\top-\mM^\star\big\|
    &\leq \big\|\widetilde{\mU}_T\widetilde{\mU}_T^\top-\mM^\star\big\| + \big\|\mU_T\mU_T^\top-\widetilde{\mU}_T\widetilde{\mU}_T^\top\big\|
    \\
    &\leq
    \big\|\widetilde{\mU}_T\widetilde{\mU}_T^\top-\mM^\star\big\|
    +
    \left(\norm{\mU_T}+\big\|\widetilde{\mU}_T\big\|\right)\big\|\mU_{T}-\widetilde{\mU}_{T}\big\|\\
    &\leq \frac{\varepsilon}{2}+2\cdot \frac{4\sqrt{\lambda_{\max}^\star}}{1-\rho}\cdot \frac{(1-\rho)\varepsilon}{16\sqrt{\lambda_{\max}^\star}} =\varepsilon, 
    \end{aligned}
\end{equation}
provided that $\alpha$ is sufficiently small. 
This completes the proof.

\section{Analysis for linear transformers (proof of \Cref{thm::muon-transformer})}
\label{sec:analysis-muon-transformer}

Recall that the gradient of $f(\mQ)$ w.r.t.~$\mQ$ is given by
\begin{align}
\nabla f(\mQ)
 = 
 \mS (\mS \mQ - \mI) \mS
 = 
 \mS^2 \mQ \mS -  \mS^2.
\end{align}
To proceed, let us denote the eigen-decomposition of $\mS$ as $\mS=\mV^{\star}\mLambda^{\star}\mV^{\star\top}$, where $\mLambda^{\star}=\diag\{\lambda_1^{\star},\dots,\lambda_d^{\star}\}$ is a diagonal matrix containing the eigenvalues $\{\lambda_i^{\star}\}$ of $\mS$, and 
$\mV^{\star}$ consists of orthonormal columns corresponding to the eigenvectors of $\mS$. As a key step of this proof, we would like to show that: 
\begin{lemma}
\label{lem:Q-decompose-transformer}
For each $t\geq 0$, 
the simplified \muon iterates (\ref{eq:muon-update-transformer}) can be decomposed as 
\begin{subequations}
\begin{align}
    \label{eq:Qt-decompose-lem-transformer}
    \mQ_t = \mV^{\star} \mTheta_t \mV^{\star\top}
\end{align}
for some diagonal matrix $\mTheta_t=\diag\{\theta_{1,t},\dots,\theta_{d,t}\}$. In particular, $\{\mTheta_t\}$ evolves according to
\begin{align}
\mTheta_{t+1} = \mTheta_t - \eta_t \dsign(\mLambda^{\star} \mTheta_t - \mI),\qquad t=0,1,\cdots 
\end{align}
\end{subequations}
where $\dsign(\mM)\coloneqq \diag\{\sign(M_{1,1}),\dots, \sign(M_{d,d})\}$ for any diagonal matrix $\mM=\diag\{M_{1,1},\dots,M_{d,d}\}$.
\end{lemma}

\begin{proof}[Proof of \Cref{lem:Q-decompose-transformer}]
The base case with $t=0$ holds trivially, since the initialization $\mQ_0=\mathbf{0}$ is equivalent to taking 
$\mLambda_0 = \mathbf{0}$. Assuming the inductive hypothesis (\ref{eq:Qt-decompose-lem-transformer}) holds at step $t$, we have
\[
\nabla f(\bm{Q}_t) = 
\mS^2\mQ_t\mS - \mS^2 
= 
\mV^{\star} \big(\mLambda^{\star 3}\mTheta_t -\mLambda^{\star 2}\big)\mV^{\star\top}, 
\]
and as a result, 
\begin{equation}
    \begin{aligned}
        \mQ_{t+1}
        &=\mQ_t-\eta_t \msign\big( \nabla L(\bm{Q}_t) \big) \\
        &=\mQ_t-\eta_t \msign\big(\mV^{\star}(\mLambda^{\star 3}\mTheta_t-\mLambda^{\star 2})\mV^{\star\top}\big)\\
        &=\mQ_t-\eta_t \mV^{\star}\dsign(\mLambda^{\star 3}\mTheta_t-\mLambda^{\star 2})\mV^{\star\top}\\
        &=\mV^{\star}\big(\mTheta_t-\eta_t \dsign(\mLambda^{\star 3}\mTheta_t-\mLambda^{\star 2})\big)\mV^{\star\top}\\
        &=\mV^{\star}\big(\mTheta_t-\eta_t \dsign(\mLambda^{\star}\mTheta_t-\mI)\big)\mV^{\star\top}.
    \end{aligned}
\end{equation}
This implies the decomposition $\mQ_{t+1}=\mV^{\star}\mTheta_{t+1}\mV^{\star\top}$, where the diagonal matrix $\mTheta_{t+1}$ can be computed as $\mTheta_{t+1}=\mTheta_t-\eta_t \sign(\mLambda^{\star}\mTheta_t-\mI)$. The proof is thus complete by induction.  
\end{proof}

Importantly, \Cref{lem:Q-decompose-transformer} indicates that the \muon dynamics can be decomposed into a collection of scalar sequences obeying 
\begin{align}
\theta_{i,t+1}=\theta_{i,t}-\eta_t \sign(\lambda_i^{\star}\theta_{i,t}-1),
\qquad t=0,1,\cdots
\end{align}
for each $1\leq i\leq d$. 
As it turns out, the convergence rate of each scalar sequence $\{\theta_{i, t}\}_{t\geq 0}$ can be analyzed through the following lemma.
\begin{lemma}
\label{lem::1d-muon-transformer}
Consider a scalar sequence $\{\theta_t\}_{t\geq 0} \subset \mathbb{R}$ obeying 
$$\theta_{t+1}=\theta_t-\eta_t \sign(\lambda^{\star}\theta_t-1),$$ where the scalar $\lambda^{\star}$ satisfies $\lambda^{\star}\geq \lambda_{\min}^\star>0$. Set the learning rate schedule to be $\eta_t = \frac{C_\eta}{\lambda_{\min}^\star}\rho^t$
for some quantities $1/2 \le \rho < 1$ and $C_\eta \ge 1$. With the initialization $\theta_0 = 0$, one has 
\begin{equation}
\bigg| \theta_{t+1} - \frac{1}{\lambda^{\star}} \bigg| \le \eta_t = \frac{C_\eta}{\lambda_{\min}^\star}\rho^{t}
\qquad 
 \text{for all } t \ge 0.
\end{equation}
\end{lemma} 

To finish up, applying \Cref{lem::1d-muon-transformer} to each scalar sequence $\{\theta_{i, t}\}_{t\geq 0}$, we arrive at
\begin{equation}
\big\| \mQ_{t+1} - \mS^{-1} \big\| 
=\big\| \mTheta_{t+1}-(\mLambda^{\star})^{ -1} \big\| 
=\max_{1\leq i\leq d}\bigg| \theta_{i,t+1} - \frac{1}{\lambda_i^{\star}} \bigg| \le \eta_t. 
\end{equation}
Thus, in order to ensure 
$\big\| \mQ_{T} - \mS^{-1} \big\| \leq \varepsilon$, it suffices to take
$T\ge \frac{1}{1-\rho}\log\big(\frac{C_{\eta}}{\sigma_{\min}(\mS)\varepsilon}\big)$.

\begin{proof}[Proof of \Cref{lem::1d-muon-transformer}]
The proof is analogous to the proof of \Cref{lem::1d-muon}. Define the metric
\begin{align}
\Delta_t \coloneqq \bigg| \theta_{t} - \frac{1}{\lambda^{\star}} \bigg|. 
\end{align}
To bound $\Delta_t$, we first observe that
\begin{equation}
    \begin{aligned}
        \theta_{t+1}-\frac{1}{\lambda^{\star}}&=\theta_t-\frac{1}{\lambda^{\star}}-\eta_t \sign(\lambda^{\star}\theta_t-1)\\
        &=\theta_t-\frac{1}{\lambda^{\star}}-\eta_t \sign\left(\theta_t-\frac{1}{\lambda^{\star}}\right)
        =\sign\left(\theta_t-\frac{1}{\lambda^{\star}}\right)\left(\left|\theta_t-\frac{1}{\lambda^{\star}}\right|-\eta_t\right).
    \end{aligned}
    \label{eq:thetat-recurrence-tf}
\end{equation}
If $\theta_t = 1/\lambda^{\star}$, then it is readily seen from (\ref{eq:thetat-recurrence-tf}) and $\sign(0)=0$ that $\Delta_{t+1}=0 \leq \eta_t$. If instead  $\theta_t \neq  1/\lambda^{\star}$, then it follows from  (\ref{eq:thetat-recurrence-tf}) that 
\begin{equation}
\Delta_{t+1} = \big| \Delta_t - \eta_t \big| \le \max\{\Delta_t - \eta_t, \eta_t\}.
\label{eq:Delta-t-UB-transformer}
\end{equation}
In summary, this inequality (\ref{eq:Delta-t-UB-transformer}) holds for both cases, which 
coincides with the bound (\ref{eq:delta-rec}) in the proof of \Cref{lem::1d-muon}. 

When $t=0$, it holds that $$\Delta_1=|\Delta_0-\eta_0|=\left|\frac{1}{\lambda^{\star}}-\frac{C_\eta}{\lambda_{\min}^\star}\right|\leq \frac{C_\eta}{\lambda_{\min}^\star}=\eta_0.$$ 
Then, repeating the same arguments as in the proof of \Cref{lem::1d-muon}, we conclude that
$$\left| \theta_{t+1} - \frac{1}{\lambda^{\star}} \right| = \Delta_{t+1} \le \eta_t = \frac{C_\eta}{\lambda_{\min}^\star}\rho^t$$
as claimed. 
\end{proof}

\section{Discussion}

In this paper, we have rigorously characterized the preconditioning benefits of \muon for two matrix optimization problems: matrix factorization, and in-context learning of linear transformers. Our theory implies that \muon's spectral orthogonalization acts as a form of adaptive preconditioners, effectively transforming its dynamics into independent scalar sequences in the spectral domain, each converging at a comparable rate. Both theoretical analyses and empirical studies suggest that \muon yields better-conditioned optimization trajectories, achieving faster convergence than \gd and \adam. We anticipate that this preconditioning mechanism plays a key role in accelerating various matrix-structured optimization problems, and that it may inform the design of new spectrum-aware optimization algorithms.

As noted previously, our theoretical analysis  is limited to two simple problems. This naturally opens up various avenues for future research. We conclude by highlighting two important directions.
\begin{itemize}
    \item {\em Extension to other matrix-structured problems.} 
    Given the limited scope of our analysis to two problems, a natural next step is to investigate whether the preconditioning effect of \muon generalizes to other matrix-structured tasks. 
    In addition to other nonconvex matrix factorization problems described in \cite{chi2019nonconvex}, one  potential example is the matrix linear regression problem given by $$
        \text{minimize}_{\mW\in \bR^{m\times n}} \norm{\mW\mX - \mY}_{\fro}^2,$$ 
        which generalizes classical linear regression to a matrix setting. This problem not only serves as a useful testbed for theoretical analysis, but also captures the training dynamics of linear layers in neural networks. Recent papers have begun to explore this space: \citet{davis2025spectral} derived a criterion under which \muon outperforms \gd in a single step, while \citet{das2024towards} investigated the preconditioning effect of \adam in the vector case. Extending these insights to broader matrix-valued problems could illuminate how \muon interacts with layer-wise structures and whether spectrum-aware optimizers yield more efficient or stable training.

\item {\em Toward a general theory.} Another important direction is to develop a unified theoretical framework that elucidates the preconditioning and acceleration effects of \muon under broad, practically relevant conditions, such as gradient Lipschitz continuity. While recent research has made progress in this direction \citep{davis2025spectral,su2025isotropic,shen2025convergence}, existing analyses remain limited in several key aspects: some  rely on idealized models, others impose intricate per-iteration conditions whose validity has yet to be rigorously established, and many fall short of explaining the observed empirical advantage of \muon over classical optimizers. Overcoming these limitations will  require deeper insight into both the geometry of the loss landscape---especially in transformer architectures---and the way in which \muon's updates dynamically reshape the optimization trajectories. It would also be of great interest to investigate whether important structural properties arising in neural network training, such as block Hessians \citep{zhang2024transformers}, can be efficiently exploited by \muon.

\end{itemize}

\section*{Acknowledgments}

Y.~Chen is supported in part by the Alfred P.~Sloan Research Fellowship,  the ONR grant N00014-25-1-2344,  the NSF grants 2221009 and 2218773, 
the Wharton AI \& Analytics Initiative's AI Research Fund, 
and the Amazon Research Award. Y. Chi is supported in part by NSF under grant ECCS-2537078 and AFOSR under grant FA9550-25-1-0060. Any opinions, findings, and conclusions or recommendations expressed in this material are those of the authors and do not necessarily reflect the views of the United States Air Force.

 \bibliographystyle{apalike}
\bibliography{ref.bib}

\newpage
\appendix
\resumetocwriting
\tableofcontents

\section{Connection between \muon and \scaledgd for matrix factorization}
\label{sec::connection-muon-scaledgd}

A provably efficient preconditioned optimizer for matrix factorization is \scaledgd \citep{tong2021accelerating}, which also achieves convergence rates independent of the condition number. As it turns out, there are some inherent connections between \muon and \scaledgd. More concretely, the update rule of simplified \muon yields
\begin{equation}
\label{eq:muon-connection-mf}
\mU_{t+1} = \mU_t - \eta_t \nabla f(\mU_t) \big( \nabla f(\mU_t)^\top \nabla f(\mU_t) \big)^{-1/2} = \mU_t - \eta_t \nabla f(\mU_t)\big( \mU_t^\top \mDelta_t^2 \mU_t \big)^{-1/2},
\end{equation}
where $\mDelta_t \coloneqq \mU_t\mU_t^\top - \mM^\star$. 
In comparison, the update rule of  \scaledgd is given by
\begin{equation}
\label{eq:scale-gd-mf}
\mU_{t+1} =  \mU_t - \beta_t \nabla f(\mU_t)\big( \mU_t^\top  \mU_t \big)^{-1}
\end{equation}
for some learning rate $\beta_t>0$. In other words, \muon constructs its preconditioner from the gradient, whereas \scaledgd builds its preconditioner from the iterate itself. In the idealistic case where $\mDelta_t^2 \approx c_t \mU_t\mU_t^{\top}$ for some scalar $c_t>0$,  (\ref{eq:muon-connection-mf}) can be simplified as
\begin{equation}
\label{eq:muon-connection-simplified-mf}
\mU_{t+1} \approx \mU_t - \eta_t c_t \nabla f(\mU_t)\big( \mU_t^\top \mU_t \big)^{-1},
\end{equation}
which coincides with the \scaledgd update (\ref{eq:scale-gd-mf}) up to proper scaling of the learning rate.

In general, the condition $\mDelta_t^2 \approx c_t \mU_t\mU_t^{\top}$ cannot possibly hold, but it offers some useful insight in the local regime $\mU_t\mU_t\approx \mM^{\star}$. Adopting once again the simplifying assumption (\ref{eq:assumption-Ut-mf-intuition}), we derive
\begin{equation}
    \mDelta_t=\mV^{\star}\left(\mSigma_t^2-\mLambda^\star\right)\mV^{\star\top}\quad \text{and}\quad \mU_t\mU_t^{\top}=\mV^{\star}\mSigma_t^2\mV^{\star\top}\approx \mV^{\star}\mLambda^\star\mV^{\star\top}.
\end{equation}
To ensure $\mDelta_t^2\approx \mU_t \mU_t^\top$, one needs to show that $(\mSigma_t^2 - \mLambda^\star)^2 \approx c_t \mLambda^\star$, or equivalently,  $$(\sigma_{i, t}^2-\lambda_i^\star)^2\approx c_t \lambda_i^\star, \qquad 1\leq i\leq r.$$
Given that
$\sigma_{i, t}^2-\lambda_i^\star=\big(\sigma_{i, t}-\sqrt{\lambda_i^\star}\big)\big(\sigma_{i, t}+\sqrt{\lambda_i^\star}\big)\approx 2\lambda_i^{\star}\big(\sigma_{i, t}-\sqrt{\lambda_i^\star}\big)$, this condition is equivalent to
\begin{equation}
    4 \big(\sigma_{i, t}-\sqrt{\lambda_i^\star}\big)^2 \approx c_t, \qquad 1\leq i\leq r.
    \label{eq:condition-sigma-lambda-intuition}
\end{equation}

\begin{figure}[t]
     \centering
     \begin{subfigure}[b]{0.3\textwidth}
         \centering
         \includegraphics[width=\textwidth]{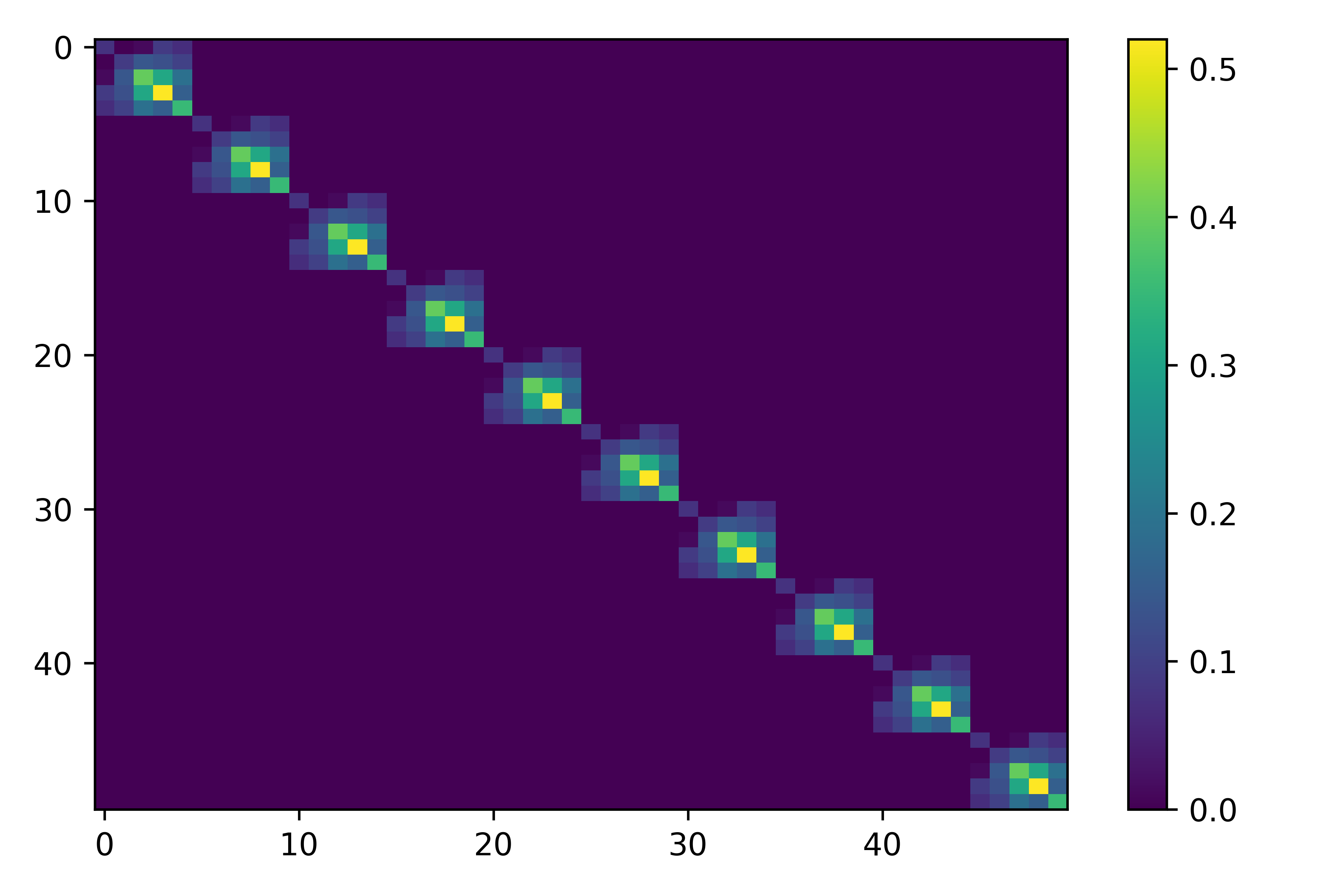}
         \caption{\muon at $t=0$}
         \label{fig::muon-0}
     \end{subfigure}
     \begin{subfigure}[b]{0.3\textwidth}
         \centering
         \includegraphics[width=\textwidth]{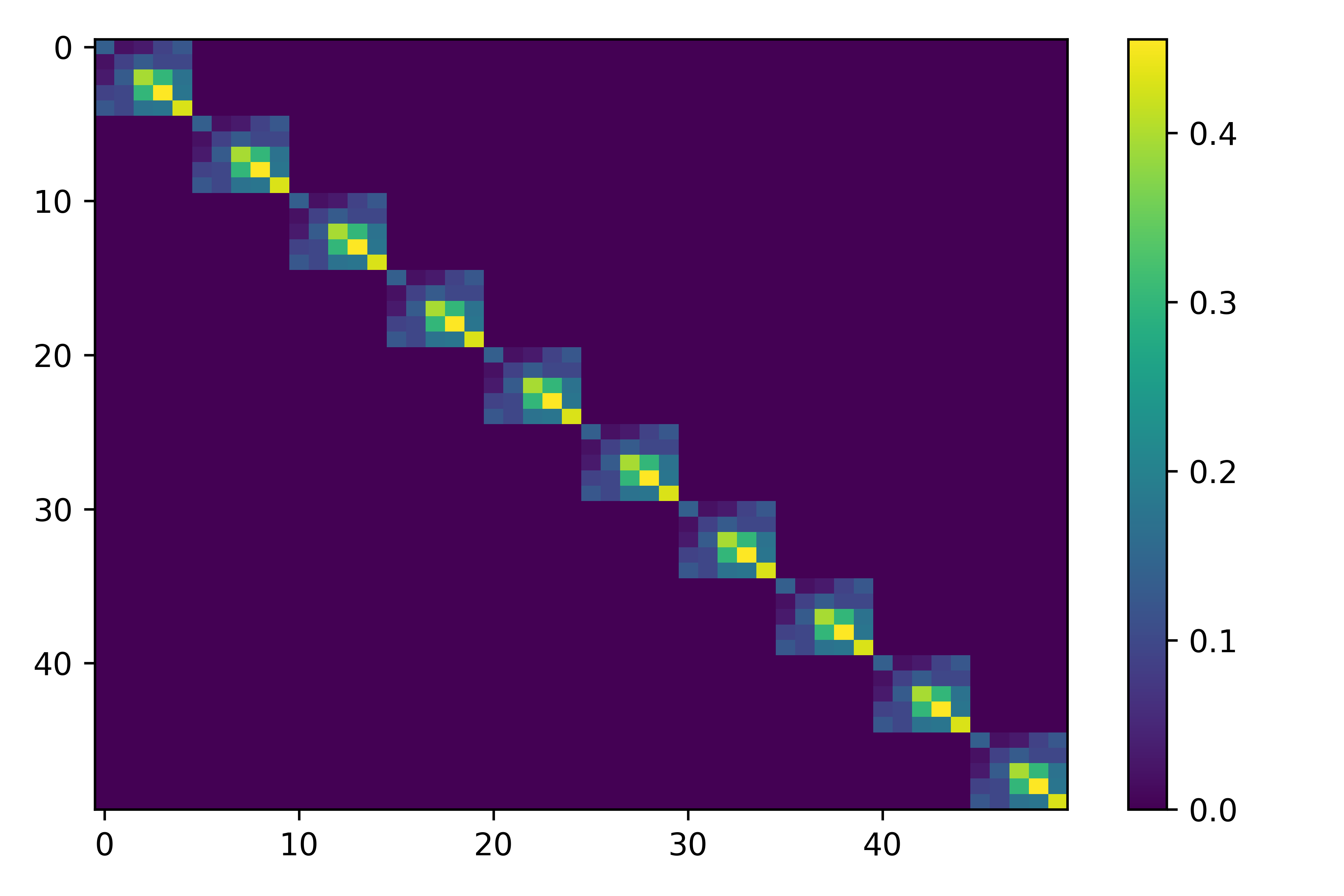}
         \caption{\muon at $t=500$}
         \label{fig::muon-500}
     \end{subfigure}
     \begin{subfigure}[b]{0.3\textwidth}
         \centering
         \includegraphics[width=\textwidth]{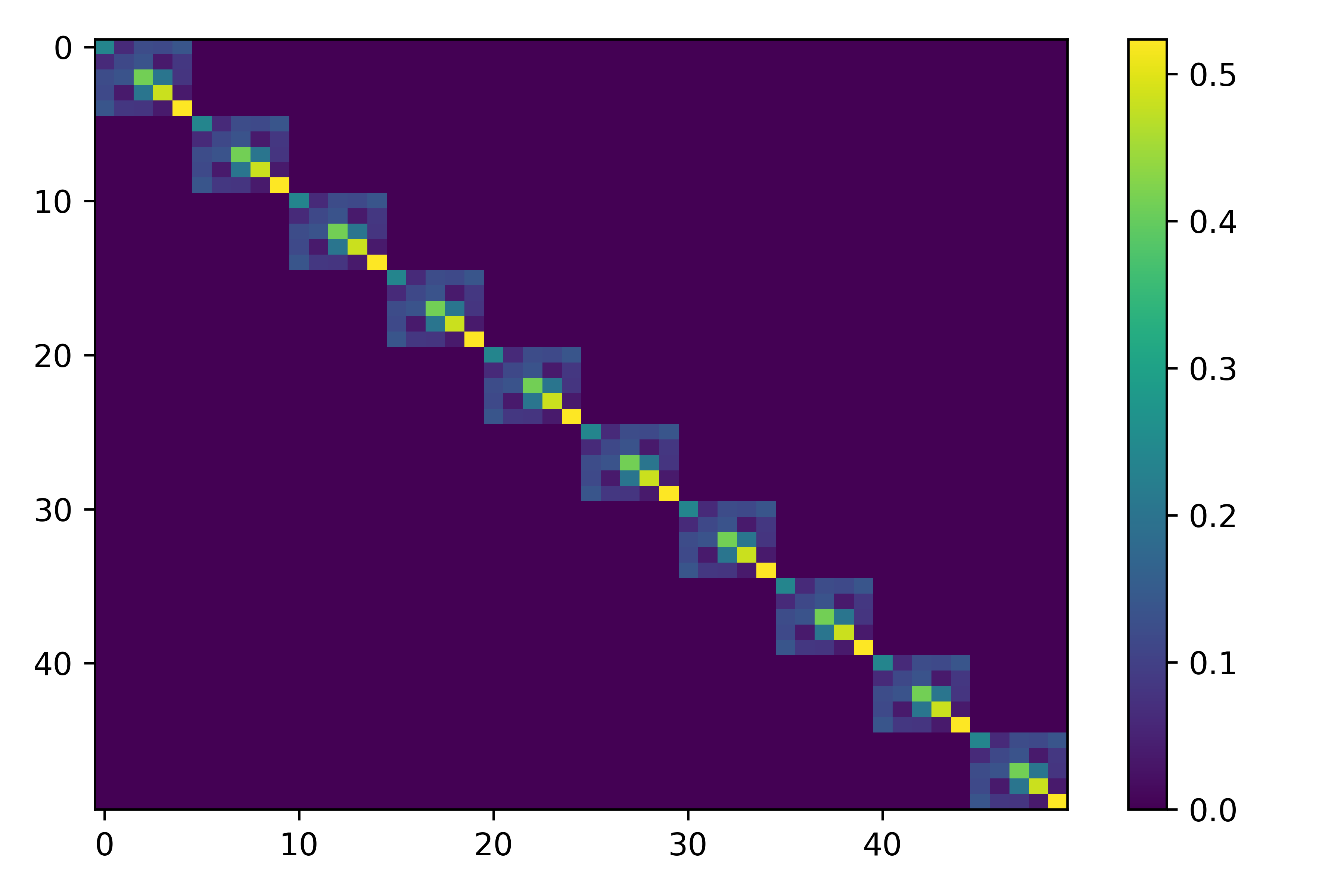}
         \caption{\muon at $t=1000$}
         \label{fig::muon-1000}
     \end{subfigure}\\
     \begin{subfigure}[b]{0.3\textwidth}
         \centering
         \includegraphics[width=\textwidth]{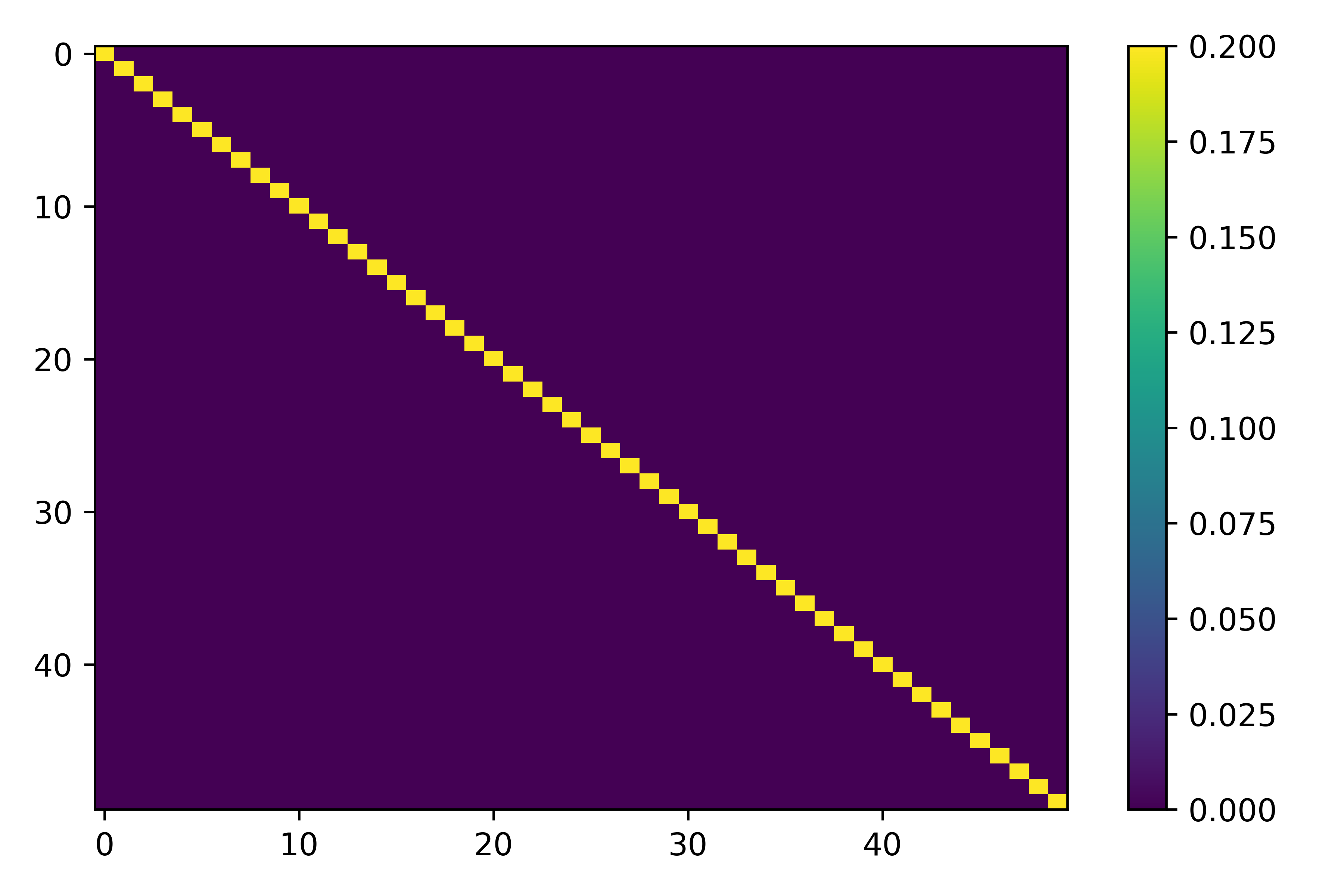}
         \caption{\scaledgd at $t=0$}
         \label{fig::scaledgd-0}
     \end{subfigure}
     \begin{subfigure}[b]{0.3\textwidth}
         \centering
         \includegraphics[width=\textwidth]{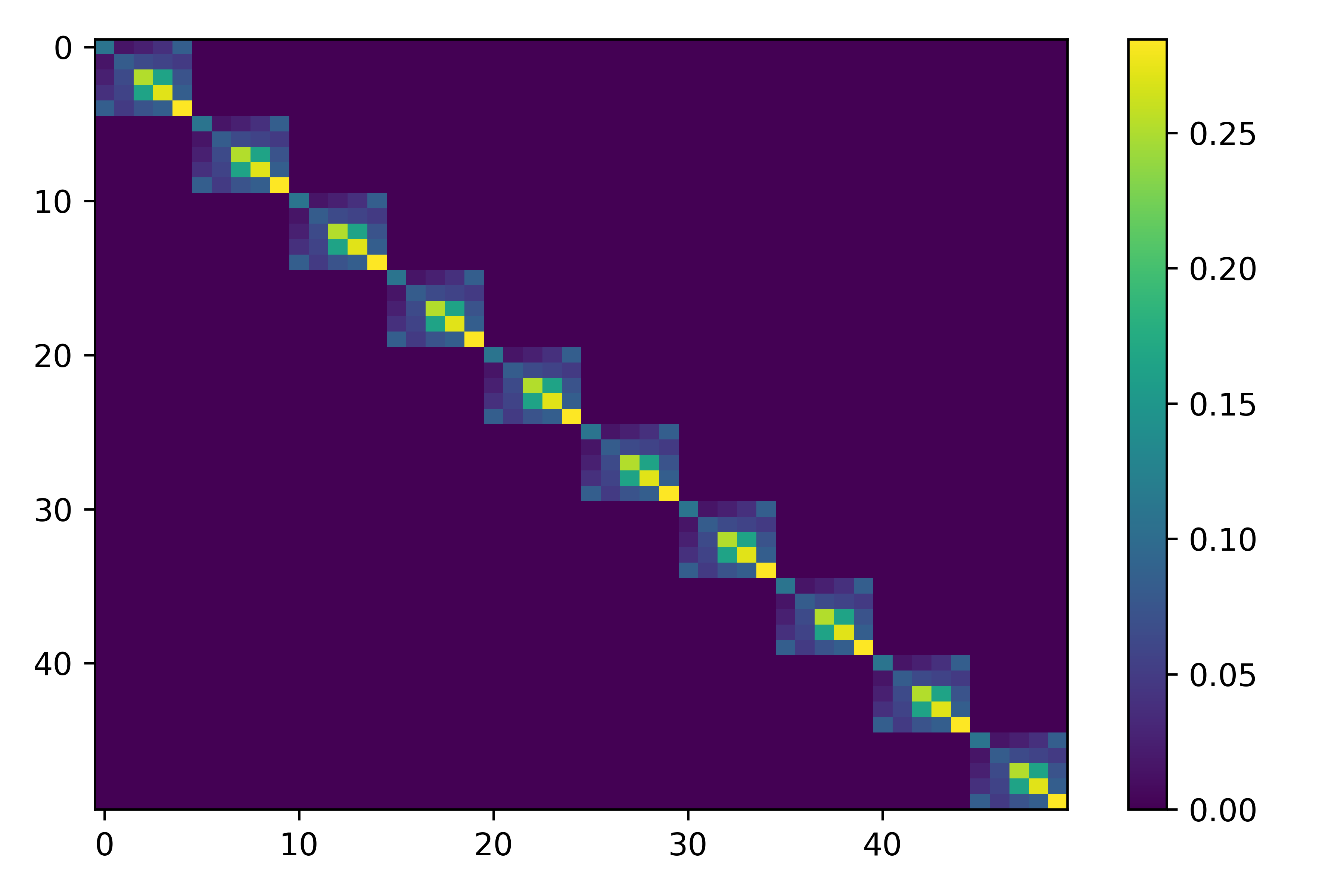}
         \caption{\scaledgd at $t=500$}
         \label{fig::scaledgd-500}
     \end{subfigure}
     \begin{subfigure}[b]{0.3\textwidth}
         \centering
         \includegraphics[width=\textwidth]{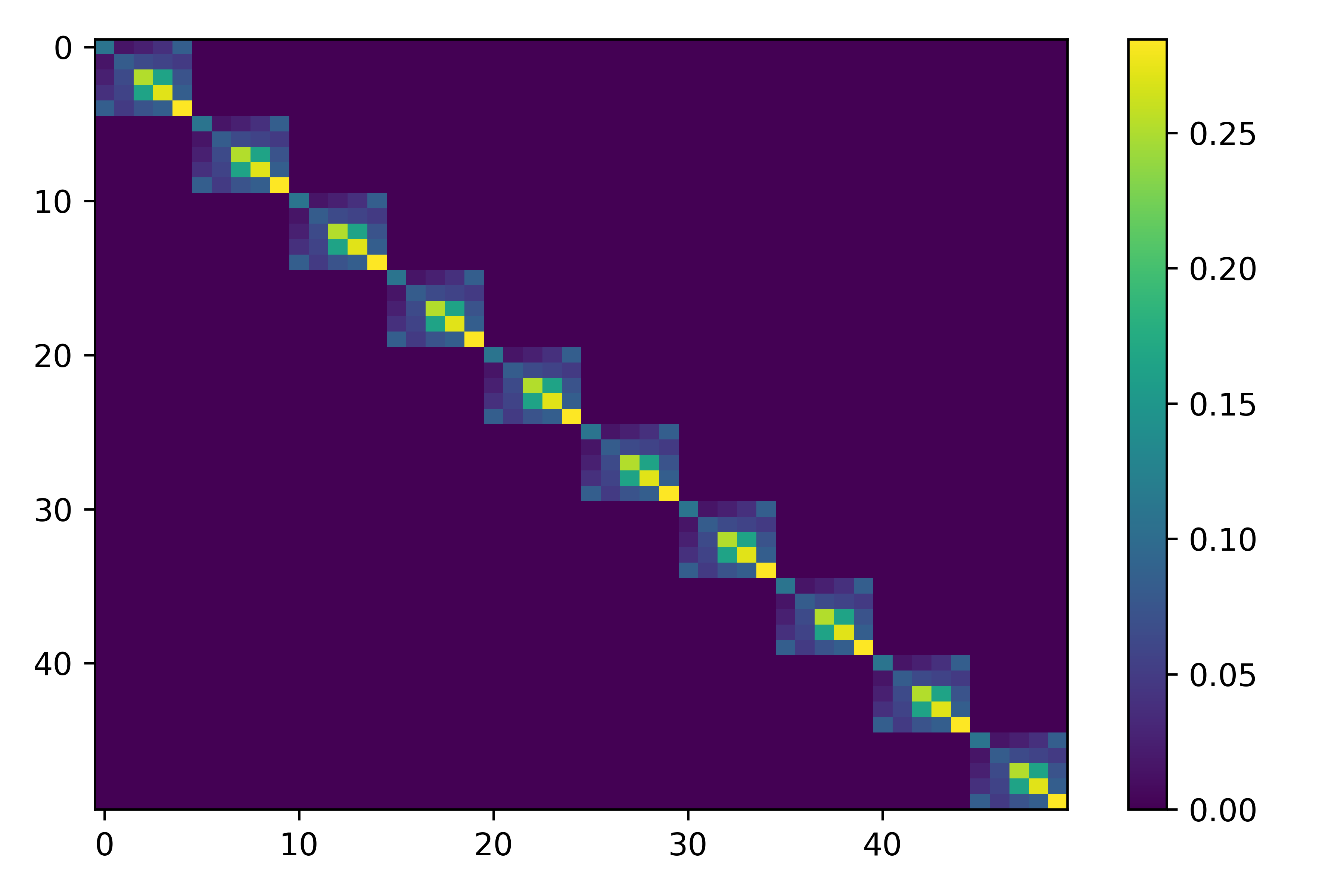}
         \caption{\scaledgd at $t=1000$}
         \label{fig::scaledgd-1000}
     \end{subfigure}
        \caption{Numerical comparison of the preconditioners of \muon and \scaledgd for matrix factorization at various training steps along a \muon trajectory. }
        \label{fig::hessian}
\end{figure}

To justify the approximate feasibility of (\ref{eq:condition-sigma-lambda-intuition}), observe that at each iteration, the scalar sequence $\{\sigma_{i,t}\}$ in (\ref{eq:scalar-sequence-intuition}) moves by a fixed length (i.e., either $\eta_t$ or $-\eta_t$) irrespective of the gradient size. In the local region where $\sigma_{i,t}\approx \lambda_i^{\star}$, the scalar sequence is expected to follow a zigzag trajectory oscillating around $\lambda_i^{\star}$. Under random initialization, one may thus anticipate
$\mathbb{E}[|\sigma_{i,t} - \sqrt{\lambda_i^\star}|] \propto \eta_t$, a scale that is independent of the magnitude of $\lambda_i^{\star}$. This intuition suggests that in the local region, the \muon update may be approximated by \scaledgd. 
Note, however, that these arguments are heuristic in nature; a fully rigorous analysis of their connections is left for future work.

To further understand the connection between \muon and \scaledgd, we conduct experiments to visualize and compare their corresponding preconditioners over the course of a \muon trajectory, as shown in \Cref{fig::hessian}. We consider a matrix factorization task with dimension $d = 10$ and target and search ranks $r = k = 5$, initialized with a small scale $\alpha = 10^{-10}$. We also adopt the same learning rate schedule as in previous experiments. At each step $t$, the \muon preconditioner is defined as $\mH_{\muon, t} = \mI \otimes \left( \nabla f(\mU_t)^\top \nabla f(\mU_t) \right)^{1/2}$, while the \scaledgd preconditioner takes the form $\mH_{\scaledgd, t} = \mI \otimes \left( \mU_t^\top \mU_t \right)$. Throughout training, both preconditioners display a consistent block-diagonal pattern—highlighting their structural similarity and revealing the implicit connection between the two methods. Importantly, the non-diagonal structure of these preconditioners also hints at why methods using diagonal preconditioners, such as \adam, are not well-suited for this setting.

\section{Proof of auxiliary lemmas for matrix factorization}

\subsection{Proof of 
\Cref{lem:diff-msign-Q-G}}
\label{sec:proof:lem:diff-msign-Q-G}

The difference between \( \msign(\mG_{0, \leq r}) \) and \( \msign(\mQ) \) can be bounded by 
\begin{align}
        \big\|\msign(\mG_{0, \leq r})-\msign(\mQ) \big\|&
        \overset{\mathrm{(i)}}{\leq}
        \frac{2}{\min\{\sigma_r(\mQ), \sigma_r(\mG_{0, \leq r})\}}\norm{\mG_{0, \leq r}-\mQ} \notag\\
    &
    \overset{\mathrm{(ii)}}{\leq} \frac{2}{\sigma_r(\mQ)-\alpha^3 }\norm{\mG_{0, \leq r}-\mQ} \notag\\
    &\overset{\mathrm{(iii)}}{\leq} 
    \frac{6+\frac{2(\sigma_{r+1}(\mQ)+\sigma_{r+1}(\mG_0))}{\min\{\sigma_{r}(\mQ),\sigma_{r}(\mG_0)\}-\max\{\sigma_{r+1}(\mQ),\sigma_{r+1}(\mG_0)\}}}{\sigma_r(\mQ)-\alpha^3}\norm{\mG_{0}-\mQ} \notag\\
        &\overset{\mathrm{(iv)}}{\leq} 
    \frac{6+\frac{2\alpha^3}{\sigma_{r}(\mQ)-2\alpha^3}}{\sigma_r(\mQ)-\alpha^3}\norm{\alpha^3 \mO}
    \notag\\
    &\overset{\mathrm{(v)}}{=} \frac{6\alpha^3+\frac{2\alpha^6}{\sigma_{r}(\mQ)-2\alpha^3}}{\sigma_r(\mQ)-\alpha^3}
    \overset{\mathrm{(vi)}}{\leq} \frac{16\alpha^3}{\sigma_r(\mQ)},
    \label{eq:msign-diff-UB-12}
\end{align}
provided that $\sigma_r(\mQ)>4\alpha^3$. 
Here, (i) follows from \Cref{lem::matrix-sign-perturbation-v2}; (ii) is valid since, by Weyl's inequality, 
\begin{align}
\sigma_r(\mG_{0, \leq r})&=\sigma_r(\mG_{0})\geq \sigma_r(\mQ)-\|\alpha^3\mO\|=\sigma_r(\mQ)-\alpha^3 ;
\label{eq:sigma-r-G0-LB12}
\end{align}
(iii) applies \Cref{lem::rank-r-perb};  (iv) results from \Cref{eq:sigma-r-G0-LB12}, the fact $\sigma_{r+1}(\mQ)=0$, as well as the following property (by Weyl's inequality):   
\begin{align*}
\sigma_{r+1}(\mG_{0})& \leq 
\sigma_{r+1}(\mQ)+\|\alpha^3\mO\|=  \|\alpha^3\mO\| = \alpha^3;
\end{align*}
(v) follows since $\bm{O}$ is orthonormal; and (vi) holds as long as $\sigma_r(\mQ)>4\alpha^3$.

To continue upper bounding (\ref{eq:msign-diff-UB-12}),  we develop a lower bound on $\sigma_r(\mQ)$ in the lemma below, whose proof is provided in \Cref{sec::proof-init}.
\begin{lemma}
\label{lem::init}
    There exists some universal constant $c_0>0$ such that, with probability at least $0.995$,  
    \begin{equation}
        \sigma_r(\mM^\star\mO)\geq \frac{c_0\lambda_r^{\star}}{\sqrt{dr}}.
    \end{equation}
\end{lemma} 
\Cref{lem::init} taken together with inequality~(\ref{eq:sigma-r-G0-LB12})   tells us that, with probability exceeding 0.995, 
\begin{equation}
    \sigma_r (\mQ) = \alpha  \sigma_r(\mM^\star\mO) \geq \frac{c_0\alpha\lambda_r^{\star}}{\sqrt{dr}}.
\end{equation}
Therefore, if  $4\alpha^{2}\leq c_{0}\lambda_{r}^{\star}/\sqrt{{dr}}$, then we establish that
\begin{equation}
    \big\| \msign(\mG_{0, \leq r})-\msign(\mQ) \big\|
    \leq \frac{16\alpha^3}{\sigma_{r}(\mQ)} 
    \leq \frac{16\alpha^2\sqrt{dr}}{c_0\lambda_r^\star}. 
\end{equation}

\subsection{Proof of \Cref{lem:diff-Ut-Utilde-t-grad-mf}}
\label{sec:proof-lem:diff-Ut-Utilde-t-grad-mf}

To begin with, the update rule (\ref{eq:muon-mf}) allows us to upper bound the size of $\mU_t$ as
\begin{subequations}
\label{eq:size-Ut-Utilde-t-mf}
\begin{equation}
    \begin{aligned}
        \norm{\mU_t}&=\norm{\mU_0-\sum_{s=0}^{t-1}\eta_s\msign\big(\nabla f(\mU_s)\big)}\leq \norm{\mU_0}+\sum_{s=0}^{t-1}\eta_s\leq \alpha + \frac{2\sqrt{\lambda_{\max}^\star}}{1-\rho}\leq \frac{4\sqrt{\lambda_{\max}^\star}}{1-\rho},
    \end{aligned}
\end{equation}
provided that $\alpha \leq 2\sqrt{\lambda_{\max}^{\star}}/(1-\rho)$. The same argument applies to $\widetilde{\mU}_t$, yielding
\begin{equation}
        \big\| \widetilde{\mU}_t\big\|\leq \frac{4\sqrt{\lambda_{\max}^\star}}{1-\rho}. 
\end{equation}
\end{subequations}
Then, it follows from (\ref{eq:muon-mf}) and our construction (\ref{eq:auxiliary-muon-mf}) that
    \begin{align}
        \big\|\mU_{t+1}-\widetilde{\mU}_{t+1}\big\|&\leq\big\|\mU_{t}-\widetilde{\mU}_{t}\big\|+\eta_t\big\|\msign\big(\nabla f(\mU_t)\big)-\msign\big(\nabla f(\widetilde{\mU}_{t})\big)\big\| \notag\\
        &\hspace{-0.5cm}\stackrel{\text{\Cref{lem::matrix-sign-perturbation}}}{\leq}\big\|\mU_{t}-\widetilde{\mU}_{t}\big\|+\eta_t \frac{3\big\|\nabla f(\mU_t)-\nabla f(\widetilde{\mU}_{t})\big\|}{\sigma_{\min}(\nabla f(\widetilde{\mU}_{t}))} \notag\\
        &\leq \left(1+\eta_t \frac{147\lambda_{\max}^\star}{(1-\rho)^2\sigma_{\min}(\nabla f(\widetilde{\mU}_{t}))}\right)\big\|\mU_{t}-\widetilde{\mU}_{t}\big\|,
        \label{eq::30-proof}
    \end{align}
where the second line relies on our assumption that $\sigma_{\min}(\nabla f(\mU_t)),\sigma_{\min}(\nabla f(\widetilde{\mU}_{t}))>0$, and the last inequality invokes the triangle inequality and (\ref{eq:size-Ut-Utilde-t-mf}) to obtain
    \begin{align*}
        \big\|\nabla f(\mU_t)-\nabla f(\widetilde{\mU}_t)\big\|&=\big\|(\mU_t\mU_t^{\top}-\mM^\star)\mU_t-(\widetilde\mU_t\widetilde\mU_t^{\top}-\mM^\star)\widetilde\mU_t\big\|
    \notag\\
    &\leq \big\|\widetilde\mU_t\widetilde\mU_t^{\top}-\mM^\star\big\|\big\|\mU_t-\widetilde{\mU}_t\big\|+\big\|\mU_t\big\|\big\|\mU_t\mU_t^{\top}-\widetilde{\mU}_t\widetilde{\mU}_t^{\top}\big\| \notag\\
    &\leq \left(\big\|\widetilde\mU_t\widetilde\mU_t^{\top} \big\|+ \|\mM^\star\|+2\max\big\{\norm{\mU_t}^2,\big\|\widetilde{\mU}_t\big\|^2\big\}\right)\big\|\mU_t-\widetilde{\mU}_t\big\| \notag\\
    &\leq \frac{49\lambda_{\max}^\star}{(1-\rho)^2}\big\|\mU_t-\widetilde{\mU}_t\big\|.
    \end{align*}

\subsection{Proof of \Cref{lem::small-ball-prob}}
\label{sec:proof:lem::small-ball-prob}

    Define the following quantity $$g_{i, t} \coloneqq\big|(\widetilde{\sigma}_{i,t}^2 - \lambda_i^\star)\widetilde{\sigma}_{i,t}\big|.$$ We would like to first control $g_{i,t+1}$ for a single $i$, followed by a union bound to cover all indices $\{1, \cdots, k\}$.

Consider any fix $i$, and 
    recall from (\ref{eq:tilde-Ut-decompose-mf}) that the update rule for $\widetilde{\sigma}_{i,t+1}$ is 
    \begin{equation}
        \widetilde{\sigma}_{i,t+1}=\widetilde{\sigma}_{i,t} - \eta_{t} \sign\big( (\widetilde{\sigma}_{i,t}^2 - \lambda_i^\star)\widetilde{\sigma}_{i,t} \big),
    \end{equation}
    where $\eta_t$ is uniform sampled from $[\sqrt{\lambda_{\max}^\star}\rho^t,2\sqrt{\lambda_{\max}^\star}\rho^t]$. 
   Thus, conditional on past randomness, $\widetilde{\sigma}_{i,t+1}$ is uniformly distributed over $\left[\widetilde{\sigma}_{i,t} -2s_{i, t}\sqrt{\lambda_{\max}^\star}\rho^t, \widetilde{\sigma}_{i,t} -s_{i,t}\sqrt{\lambda_{\max}^\star}\rho^t\right]$, where $s_{i,t}=\sign\big( (\widetilde{\sigma}_{i,t}^2 - \lambda_i^\star)\widetilde{\sigma}_{i,t} \big)$; in other words, $\widetilde{\sigma}_{i,t+1}$ is uniformly sampled from an interval of length $\sqrt{\lambda_{\max}^\star}\rho^t$. We now divide into two cases based on whether $\lambda_i^\star=0$ or $\lambda_i^\star>0$.
    
    \paragraph{Case 1: $\lambda_i^\star=0$.} In this case, one has $g_{i, t+1}=|\widetilde{\sigma}_{i,t+1}|^3$, which implies that
    \begin{equation}
        \bP\left(g_{i, t+1}\leq \varepsilon\mid \mathcal F_{t}\right)=\bP\left(|\widetilde{\sigma}_{i,t+1}|\leq \sqrt[3]{\varepsilon}\mid \mathcal F_{t}\right)\leq \frac{2\sqrt[3]{\varepsilon}}{\sqrt{\lambda_{\max}^\star}\rho^t}.
    \end{equation}
    \paragraph{Case 2: $\lambda_i^\star>0$.} For this case, we first claim that for any $0<\varepsilon\leq \frac{1}{4}(\lambda_{i}^{\star})^{ 3/2}$, it holds that
    \begin{equation}
        g_{i, t}\leq \varepsilon\quad \implies\quad  |\widetilde{\sigma}_{i,t}|\in \left[0, \frac{2 \varepsilon}{\lambda_i^\star}\right] \cup\left[\sqrt{\lambda_i^\star}-\frac{2\varepsilon}{\lambda_i^\star}, \sqrt{\lambda_i^\star}+\frac{2\varepsilon}{\lambda_i^\star}\right].
        \label{eq:git-range-interval}
    \end{equation}
    Without loss of generality,  assume that $\widetilde{\sigma}_{i,t}\geq 0$. To justifies this property (\ref{eq:git-range-interval}), consider two sub-cases:
    \begin{itemize}
        \item If $\widetilde{\sigma}_{i,t}\geq  \sqrt{\lambda_i^\star}+\frac{\varepsilon}{\lambda_i^\star}$, then we have  
        \begin{equation}
            \begin{aligned}
                g_{i, t}=(\widetilde{\sigma}_{i,t}^2 - \lambda_i^\star)\widetilde{\sigma}_{i,t}\geq \frac{2\varepsilon}{\sqrt{\lambda_i^\star}}\cdot \sqrt{\lambda_i^\star}\geq 2\varepsilon > \varepsilon.
            \end{aligned}
        \end{equation}
        \item If $\frac{2\varepsilon}{\lambda_i^\star}< \widetilde{\sigma}_{i,t}< \sqrt{\lambda_i^\star}-\frac{2\varepsilon}{\lambda_i^\star}$, then it follows that 
        \begin{equation}
            \begin{aligned}
                g_{i, t}=(\sqrt{\lambda_i^\star}+\widetilde{\sigma}_{i,t})(\sqrt{\lambda_i^\star}-\widetilde{\sigma}_{i,t})\widetilde{\sigma}_{i,t}\geq \sqrt{\lambda_i^\star}(\sqrt{\lambda_i^\star}-\widetilde{\sigma}_{i,t})\widetilde{\sigma}_{i,t}> \sqrt{\lambda_i^\star}\frac{2\varepsilon}{\lambda_i^\star}\left(\sqrt{\lambda_i^\star}-\frac{2\varepsilon}{\lambda_i^\star}\right)\geq \varepsilon,
            \end{aligned}
        \end{equation}
        provided that $0<\varepsilon\leq \frac{1}{4}(\lambda_{i}^{\star})^{ 3/2}$.
    \end{itemize}
    Combining the above two subcases, one can easily see that
    \begin{equation*}
        \begin{aligned}
            \bP\left(g_{i, t+1}\leq \varepsilon\mid \mathcal F_{t}\right)&\leq \bP\left(|\widetilde{\sigma}_{i,t+1}|\in \left[0, \frac{2 \varepsilon}{\lambda_i^\star}\right] \cup\left[\sqrt{\lambda_i^\star}-\frac{2\varepsilon}{\lambda_i^\star}, \sqrt{\lambda_i^\star}+\frac{2\varepsilon}{\lambda_i^\star}\right]\mid \mathcal F_{t}\right)\leq \frac{12\varepsilon/\lambda_{i}^{\star}}{ \sqrt{\lambda_{\max}^\star}\rho^t}\leq \frac{12\varepsilon}{\lambda_{\min}^{\star}\sqrt{\lambda_{\max}^\star}\rho^t}.
        \end{aligned}
    \end{equation*}

    To finish up, 
    apply the union bound over all indices $i \in \{1, \dots, k\}$ to arrive at
\begin{equation*}
    \begin{aligned}
        \bP\left( \sigma_{\min}\big(\nabla f(\widetilde{\mU}_{t+1})\big) \leq \varepsilon \right) &\leq \sum_{i:\, \lambda_i^\star = 0} \bP(g_{i,t+1} \leq \varepsilon) + \sum_{i:\, \lambda_i^\star > 0} \bP(g_{i,t+1} \leq \varepsilon)\\
    &\leq \frac{2(k-r)\sqrt[3]{\varepsilon}}{\sqrt{\lambda_{\max}^\star}\rho^t}+\frac{12r\varepsilon}{\lambda_{\min}^{\star}\sqrt{\lambda_{\max}^\star}\rho^t}.
    \end{aligned}
\end{equation*}

\subsection{Proof of \Cref{lem:prod-mgf}}
\label{sec:proof:lem:prod-mgf}

Recall that $\mathcal{F}_t$ encompasses what happens up to time $t$, and hence $\widetilde{\bm{U}}_t$ is fully determined by $\mathcal{F}_{t-1}$.  
Define
\begin{equation}
    C_t \coloneqq \frac{294\lambda_{\max}^{\star 3/2}\rho^t}{(1-\rho)^2},
\qquad
X_t=\log\left(1+\frac{C_t}{\sigma_{\min}(\nabla f(\widetilde{\mU}_t))}\right),
\qquad
S_T=\sum_{t=1}^{T-1}X_t,
\end{equation}
which allows us to write $\Pi_T=e^{S_T}$. In the sequel, we intend to control $S_T$ by invoking the Chernoff-type arguments and bounding (conditional) moment generating functions (MGFs).

\paragraph{Step 1: a general connection between MGF and tail bounds. }
For any nonnegative random variable $Z$ and any $\theta>0$, the MGF obeys
\begin{equation}
\label{eq:mgf-tail}
\bE\left[e^{\theta Z}\right]
=1+\int_{0}^{\infty}\theta e^{\theta \tau}\bP(Z\ge \tau)\mathrm{d}\tau.
\end{equation}
This follows from integration by parts, namely, $\bE[e^{\theta Z}]=\int_0^\infty e^{\theta z} \mathrm{d}F(z)
=1+\int_0^\infty \theta e^{\theta \tau}\bP(Z\ge \tau) \mathrm{d}\tau$. 

\paragraph{Step 2: a conditional tail bound on $X_t$.}
For any $\tau\ge 0$, the definition of $X_t$ indicates that
\begin{equation}
    \{X_t\ge \tau\}\quad
\iff \quad
\left\{\sigma_{\min}\big(\nabla f(\widetilde{\mU}_t)\big)\le \frac{C_t}{e^\tau-1}\right\}.
\end{equation}
With this equivalence in mind, applying \Cref{lem::small-ball-prob} to $\sigma_{\min}(\nabla f(\widetilde{\mU}_t))$ reveals that, for all $\tau\ge 0$,
\begin{equation}
\label{eq:tailXt-cond}
\bP(X_t\ge \tau \mid \mathcal F_{t-1})
\le
\min\left\{
\frac{2(k-r)\sqrt[3]{C_t}}{\sqrt{\lambda_{\max}^\star}\rho^{t-1}\sqrt[3]{e^\tau-1}}
+
\frac{12r C_t}{\lambda_{\min}^\star\sqrt{\lambda_{\max}^\star}\rho^{t-1}(e^\tau-1)},\, 1
\right\}.
\end{equation}

\paragraph{Step 3: an MGF bound for the case with $k=r$.}
When $k=r$, the first term in (\ref{eq:tailXt-cond}) vanishes. Define 
\begin{equation}
    A \coloneqq \frac{12r C_t}{\lambda_{\min}^\star\sqrt{\lambda_{\max}^\star}\rho^{t-1}}
=
\frac{12r}{\lambda_{\min}^\star\sqrt{\lambda_{\max}^\star}}\cdot
\frac{294\lambda_{\max}^{\star 3/2}\rho^t}{(1-\rho)^2}\cdot \frac{1}{\rho^{t-1}}
=
\frac{3528r\kappa\rho}{(1-\rho)^2}\geq 1,
\end{equation}
which is independent of $t$. Then, it follows from (\ref{eq:tailXt-cond}) that
\begin{equation}
    \bP(X_t\ge \tau \mid \mathcal F_{t-1})\le \min\left\{1,\frac{A}{e^\tau-1}\right\}.
\end{equation}
Let $\tau_0 \coloneqq \log(1+A)$, which satisfies $1=\frac{A}{e^{\tau_0}-1}$. For $\tau\ge \tau_0$, it is seen that $e^\tau-1\ge  e^\tau/2$ due to the fact that $\tau_0= \log(1+A)\geq \log(2)$. Hence, it holds that, for all $\tau\ge \tau_0$,
\begin{equation}
    \bP(X_t\ge \tau \mid \mathcal F_{t-1})
\le \frac{A}{e^\tau-1}\le 2A e^{-\tau}.
\end{equation}

Next,  substitute this tail bound into (\ref{eq:mgf-tail}) to show that: 
for any $\theta\in(0,1)$,
\begin{equation}
    \begin{aligned}
        \bE\left[e^{\theta X_t}\mid \mathcal F_{t-1}\right]
&=1+\int_0^{\tau_0} \theta e^{\theta \tau}\,\bP(X_t\ge \tau\mid \mathcal F_{t-1})\mathrm{d}\tau
  +\int_{\tau_0}^{\infty} \theta e^{\theta \tau}\,\bP(X_t\ge \tau\mid \mathcal F_{t-1})\mathrm{d}\tau \\
&\le 1+\int_0^{\tau_0} \theta e^{\theta \tau}\mathrm{d}\tau
  +\int_{\tau_0}^{\infty} \theta e^{\theta \tau}\cdot 2A e^{-\tau}\mathrm{d}\tau \\
&= 1+(e^{\theta \tau_0}-1) + 2A\theta \int_{\tau_0}^{\infty} e^{-(1-\theta)\tau}\mathrm{d}\tau \\
&= e^{\theta \tau_0} + \frac{2A\theta}{1-\theta}\,e^{-(1-\theta)\tau_0}.
    \end{aligned}
\end{equation}
Recall the identity $e^{\tau_0}=1+A$, we have
\begin{equation}
    e^{\theta \tau_0}=(1+A)^\theta,
\qquad
\frac{2A\theta}{1-\theta}\,e^{-(1-\theta)\tau_0}
=\frac{2\theta}{1-\theta}\,A(1+A)^{-(1-\theta)}
\le \frac{2\theta}{1-\theta}\,(1+A)^\theta,
\end{equation}
since $A(1+A)^{-(1-\theta)}\le (1+A)^\theta$.
As a result, we arrive at
\begin{equation}
\label{eq:mgf-k=r}
\bE\left[e^{\theta X_t}\mid \mathcal F_{t-1}\right]
\le
\left(1+\frac{2\theta}{1-\theta}\right)(1+A)^\theta
\le
3 (1+A)^\theta
\end{equation}
for any $\theta\in (0, 1/2]$.

\paragraph{Step 4: Chernoff bound on $S_T$ when $k=r$.}
For any $\theta\in(0,1/2]$, apply (\ref{eq:mgf-k=r}) recursively to obtain
\begin{equation}
    \begin{aligned}
        \bE\left[e^{\theta S_T}\right]
&=\bE\left[\prod_{t=1}^{T-1} e^{\theta X_t}\right]=\bE\left[\prod_{t=1}^{T-2} e^{\theta X_t}\cdot \bE\left[e^{\theta X_{T-1}}\mid \mathcal F_{T-2}\right]\right] \\
&\le 3(1+A)^\theta \cdot \bE\left[\prod_{t=1}^{T-2} e^{\theta X_t}\right] \le \cdots \\
&\le \left(3(1+A)^\theta\right)^{T-1}.
    \end{aligned}
\end{equation}
Markov’s inequality yields, for any $u>0$,
\begin{equation*}
    \bP(S_T\ge u)\le e^{-\theta u}\,\bE\left[e^{\theta S_T}\right]
\le
\exp\big(-\theta u + (T-1)\log(3) + \theta (T-1)\log(1+A)\big).
\end{equation*}
Choosing
\begin{equation}
    u
=(T-1)\log(1+A) + \frac{T-1}{\theta}\log(3) + \frac{1}{\theta}\log\frac{1}{\delta}
\end{equation}
then gives $\bP(S_T\ge u)\le \delta$. Taking $\theta=1/2$, we obtain that with probability at least $1-\delta$,
\begin{equation}
    S_T \le (T-1) \log(1+A) 2(T-1)\log(3) + 2 \log\frac{1}{\delta}=O\left(T\log\Big(\frac{r\kappa}{1-\rho}\Big) + \log\frac{1}{\delta}\right).
\end{equation}
Exponentiating both sides yields
\begin{equation}
    \Pi_T=\exp(S_T) \le
\exp\left(
O\left(T\log\Big(\frac{r\kappa}{1-\rho}\Big) + \log\frac{1}{\delta}\right)
\right),
\end{equation}
thereby completing the proof of Part (i) of \Cref{lem:prod-mgf}.

\paragraph{Step 5: an MGF bound  for the case with $k>r$.}
When $k>r$, define
\begin{equation}
    A_{1,t} \coloneqq \frac{2(k-r)\sqrt[3]{C_t}}{\sqrt{\lambda_{\max}^\star}\rho^{t-1}}=\frac{2(k-r)}{\rho^{\frac{2t-3}{3}}},
\qquad
A_{2} \coloneqq \frac{12r C_t}{\lambda_{\min}^\star\sqrt{\lambda_{\max}^\star}\rho^{t-1}}
=\frac{3528r\kappa\rho}{(1-\rho)^2}.
\end{equation}
Then it is readily seen from (\ref{eq:tailXt-cond}) that
\begin{equation}
    \bP(X_t\ge \tau \mid \mathcal F_{t-1})
\le \min\left\{1, \frac{A_{1,t}}{(e^\tau-1)^{1/3}}+\frac{A_{2}}{e^\tau-1}\right\}.
\end{equation}
Let $\tau_{1,t}:=\log(1+A_{1,t}^3)$ and $\tau_2:=\log(1+A_2)$, and set $\tau_0:=\max\{\tau_{1,t},\tau_2\}$.
For $\tau\ge \tau_0$, we have
\begin{equation}
    \bP(X_t\ge \tau \mid \mathcal F_{t-1})
\le
2^{1/3}A_{1,t} e^{-\tau/3} + 2A_2 e^{-\tau},
\end{equation}
given that $e^\tau-1\ge  e^\tau /2$. 
Invoke (\ref{eq:mgf-tail}) to show that, for any $\theta\in(0,1/3)$,
\begin{equation}
\label{eq:UB-MGF-conditional-1357}
    \begin{aligned}
        \bE[e^{\theta X_t}\mid \mathcal F_{t-1}]
&\le 1+\int_0^{\tau_0}\theta e^{\theta \tau}\mathrm{d}\tau
+\int_{\tau_0}^{\infty}\theta e^{\theta \tau}\big(2^{1/3}A_{1,t} e^{-\tau/3}+2A_2 e^{-\tau}\big)\mathrm{d}\tau \\
&= e^{\theta \tau_0}
+2^{1/3}A_{1,t}\theta\int_{\tau_0}^\infty e^{-(1/3-\theta)\tau}\mathrm{d}\tau
+2A_2\theta\int_{\tau_0}^\infty e^{-(1-\theta)\tau}\mathrm{d}\tau \\
&= e^{\theta \tau_0}
+\frac{2^{1/3}A_{1,t}\theta}{1/3-\theta}\,e^{-(1/3-\theta)\tau_0}
+\frac{2A_2\theta}{1-\theta}\,e^{-(1-\theta)\tau_0}.
    \end{aligned}
\end{equation}
Now, recognizing that $ \tau_{1,t}=\log(1+A_{1,t}^3)$ and $ \tau_2=\log(1+A_2)$, we can further derive
\begin{equation}
    e^{\theta \tau_0}\le e^{\theta\tau_{1,t}+\theta\tau_2}
    \le (1+A_{1,t}^3)^\theta (1+A_2)^\theta,
\end{equation}
and also (since $\tau_0\geq \tau_{1,t}$ and $\tau_0\geq \tau_2$)
\begin{equation}
    A_{1,t} e^{-(1/3-\theta)\tau_0}
\le A_{1,t}(1+A_{1,t}^3)^{-(1/3-\theta)}
\le (1+A_{1,t}^3)^\theta,
\qquad
A_2 e^{-(1-\theta)\tau_0}\le (1+A_2)^\theta.
\end{equation}
Substitution into (\ref{eq:UB-MGF-conditional-1357}) reveals that, for any given $\theta\in(0,1/6]$, 
\begin{equation}
\label{eq:mgf-k>r}
\bE[e^{\theta X_t}\mid \mathcal F_{t-1}]
\le
C_1(1+A_{1,t}^3)^\theta\,(1+A_2)^\theta,
\end{equation}
where $C_1$ is a constant given by $C_1=1+\frac{2^{1/3}\theta}{1/3-\theta}+\frac{2\theta}{1-\theta}$.

\paragraph{Step 6: Chernoff bound on $S_T$ when $k>r$.}
Iterating conditional expectations as before and invoking (\ref{eq:mgf-k>r}), we arrive at
\begin{equation}
    \begin{aligned}
        \bE[e^{\theta S_T}]
&\le
\prod_{t=1}^{T-1}\Big(C_1(1+A_2)^\theta(1+A_{1,t}^3)^\theta\Big)
\\
&=
\exp\left(
(T-1)\log\left(C_1\right) + \theta (T-1)\log(1+A_2) + \theta\sum_{t=1}^{T-1}\log(1+A_{1,t}^3)
\right).
    \end{aligned}
\end{equation}
Akin to Step 4, Markov's inequality then yields
\begin{equation*}
    \bP(S_T\ge u)
\le
\exp\left(
-\theta u + (T-1)\log\left(C_1\right) + \theta (T-1)\log(1+A_2) + \theta\sum_{t=1}^{T-1}\log(1+A_{1,t}^3)
\right).
\end{equation*}
Clearly, choosing
\begin{equation}
    u
=
(T-1)\log(1+A_2)+\sum_{t=1}^{T-1}\log(1+A_{1,t}^3)
+\frac{T-1}{\theta}\log\left(C_1\right)+\frac{1}{\theta}\log\frac{1}{\delta}.
\end{equation}
yields $\bP(S_T\ge u)\le \delta$. 
Taking $\theta=1/6$ above and recognizing the facts that
\begin{align*}
    \sum_{t=1}^{T-1}\log(1+A_{1,t}^3)&=\sum_{t=1}^{T-1}\log\left(1+\frac{8(k-r)^3}{\rho^{2t-3}}\right)=O\left(T^2\log\frac{1}{\rho}+T\log(k-r)\right), \\
   \log(1+A_2) &=  O\bigg( \log
\Big(\frac{r\kappa}{1-\rho} \Big) \bigg),
\end{align*}
we can use $\rho \geq 2/3$ to demonstrate that
\begin{equation*}
    S_T\leq O\left(T^2+T\log\left(\frac{(k-r)r\kappa}{1-\rho}\right)
+
\log\frac{1}{\delta}\right),
\end{equation*}
with probability at least $1-\delta$, and as a consequence,
\begin{equation}
    \Pi_T=\exp(S_T) \le
\exp\left(
O\left(T^2+T\log\left(\frac{(k-r)r\kappa}{1-\rho}\right)
+
\log\frac{1}{\delta}\right)
\right).
\end{equation}
This establishes Part (ii) of \Cref{lem:prod-mgf}.

\subsection{Proof of \Cref{lem::init}}
\label{sec::proof-init}
    Recalling that $\mM^\star=\mV^\star\mLambda^\star\mV^{\star\top}$, we can derive
    \begin{equation}
        \begin{aligned}
            \sigma_r(\mM^\star\mO)&=\sigma_r(\mLambda^\star\mV^{\star\top}\mO)\geq \lambda_r^\star\cdot \sigma_{r}(\mV^{\star\top}\mO)\geq \lambda_r^\star\cdot \sigma_{r}(\mV^{\star\top}\mO_{:, 1:r}),
        \end{aligned}
    \end{equation}
    where $\mO_{:, 1:r}\in \bR^{d\times r}$ is composed of the first $r$ columns of $\mO$.

    To proceed, observe that $\mO_{:, 1:r}$ has the same distribution as $\mG(\mG^{\top}\mG)^{-1/2}$, where $\mG\in \bR^{d\times r}$ is a random matrix with i.i.d.~standard Gaussian entries. Hence, it suffices to develop a high-probability lower bound for 
    $\sigma_r(\mV^{\star\top}\mG(\mG^{\top}\mG)^{-1/2})$. Towards this end, we first make the observation that
    \begin{equation}
        \sigma_r\big(\mV^{\star\top}\mG(\mG^{\top}\mG)^{-1/2}\big)\geq \sigma_r\big(\mV^{\star\top}\mG\big)
        \sigma_r\big((\mG^{\top}\mG)^{-1/2}\big)=\frac{\sigma_r(\mV^{\star\top}\mG)}{\sigma_1(\mG)}.
    \end{equation}
    It is clearly seen that $\mV^{\star\top}\mG$ is also a random matrix with i.i.d.~standard Gaussian entries. 
    In view of \Cref{lem::concentration-maximal-singular-value,lem::concentration-minimal-singular-value},  there exists some universal constant $c_0>0$ such that 
    \begin{equation}
        \frac{\sigma_r(\mV^{\star\top}\mG)}{\sigma_1(\mG)}\geq c_0\frac{1/\sqrt{r}}{\sqrt{d}}=\frac{c_0}{\sqrt{dr}} 
    \end{equation}
    holds with probability at least $0.995$. 
   Taking the above arguments together, we arrive at
    \begin{equation}
        \sigma_r(\mM^\star\mO)\geq \frac{c_0\lambda_r^\star}{\sqrt{dr}}
    \end{equation}
with probability at least $0.995$.

\subsection{Scalar dynamics with time-varying prefactors in learning rates} 

This subsection presents a slight extension of \Cref{lem::1d-muon} to accommodate slightly broader learning rates. 

\begin{lemma}
\label{lem::1d-muon-varying-C}
Consider the scalar updates in (\ref{eq:1d-muon-rec}), where $0 \le \lambda^\star \le \lambda_{\max}^\star$.  
Set the learning rate schedule to be
\[
\eta_t = C_{\eta,t}\sqrt{\lambda_{\max}^\star}\rho^t
\qquad \text{for some } 1\le C_{\eta,t}\le 2 \text{ and }
2/3\le \rho < 1.
\]
Assume that $0<|u_0| \le \eta_0$.
Then, with probability $1$, for all $t \ge 0$, it holds that
\begin{subequations}
\begin{align}
\big| |u_t| - \sqrt{\lambda^\star} \big|
&
\le \frac{2}{1-\rho}\sqrt{\lambda_{\max}^\star}\rho^t,
\label{eq:1d-bound-varying-C}
\\
|u_t^2 - \lambda^\star|
&
\le
\left(\frac{4}{(1-\rho)^2}+\frac{4}{1-\rho}\right)\lambda_{\max}^\star\rho^{t}.
\label{eq:1d-square-bound-varying-C}
\end{align}
\end{subequations}
\end{lemma}

\begin{proof}[Proof of \Cref{lem::1d-muon-varying-C}]
Similarly to the proof of \Cref{lem::1d-muon}, define $\Delta_t \coloneqq \big| |u_t| - \sqrt{\lambda^\star} \big|$, which satisfies (see (\ref{eq:delta-rec}))
\begin{equation}
\label{eq:delta-rec-varying-C}
\Delta_{t+1} = \big| \Delta_t - \eta_t \big|.
\end{equation}
Next, define the tail sum
$S_t \coloneqq \sum_{s=t}^{\infty}\eta_s.$
We claim for the moment that
\begin{equation}
\label{eq:Delta-le-tail}
\Delta_t \le S_t
\qquad \text{for all }t\ge 0.
\end{equation}
Once \Cref{eq:Delta-le-tail} is established, the first claim (\ref{eq:1d-bound-varying-C}) follows immediately since
\begin{equation}
\label{eq:tail-sum-upper}
S_t = \sum_{s=t}^{\infty} C_{\eta,s}\sqrt{\lambda_{\max}^\star}\rho^s
\le 2\sqrt{\lambda_{\max}^\star}\sum_{s=t}^{\infty}\rho^s
= \frac{2}{1-\rho}\sqrt{\lambda_{\max}^\star}\rho^t,
\end{equation}
where we have used $C_{\eta,s}\le 2$ for all $s\geq 0$.

It remains to prove the claim (\ref{eq:Delta-le-tail}), which we accomplish by induction.
\begin{itemize}
\item \textit{Base case ($t=0$).}
Recalling that $\sqrt{\lambda^\star}\le \sqrt{\lambda_{\max}^\star}$ and $|u_0|\le \eta_0$, we have
\begin{equation}
\label{eq:Delta0-upper}
\Delta_0 = \big| |u_0| - \sqrt{\lambda^\star} \big|
\le |u_0| + \sqrt{\lambda^\star}
\le \eta_0 + \sqrt{\lambda_{\max}^\star}
\le 2\sqrt{\lambda_{\max}^\star} + \sqrt{\lambda_{\max}^\star}
= 3\sqrt{\lambda_{\max}^\star},
\end{equation}
which follows since $\eta_0=C_{\eta,0}\sqrt{\lambda_{\max}^\star}\le 2\sqrt{\lambda_{\max}^\star}$.
Moreover, since $C_{\eta,s}\ge 1$ for all $s$, we obtain 
\begin{equation}
\label{eq:S0-lower}
S_0=\sum_{s=0}^{\infty}\eta_s
=\sum_{s=0}^{\infty} C_{\eta,s}\sqrt{\lambda_{\max}^\star}\rho^s
\ge \sqrt{\lambda_{\max}^\star}\sum_{s=0}^{\infty}\rho^s
=\frac{1}{1-\rho}\sqrt{\lambda_{\max}^\star}
\geq 3\sqrt{\lambda_{\max}^\star},
\end{equation}
with the proviso that $\rho\ge 2/3$. Therefore, $S_0\ge 3\sqrt{\lambda_{\max}^\star}\ge \Delta_0$, thus validating the base case.

\item \textit{Inductive step.}
Now assume that $\Delta_t \le S_t$ for some $t\ge 0$. 
To bound $\Delta_{t+1}$, we divide into two cases.

\begin{itemize}
\item 
\textit{Case 1: $\Delta_t \ge \eta_t$.}
In this case, \Cref{eq:delta-rec-varying-C} yields $\Delta_{t+1}=\Delta_t-\eta_t \le S_t-\eta_t = S_{t+1}$, which holds since $S_{t+1}=S_t-\eta_t$.

\item 
\textit{Case 2: $\Delta_t \le \eta_t$.}
In this case, \Cref{eq:delta-rec-varying-C} yields $\Delta_{t+1}=\eta_t-\Delta_t \le \eta_t$.
Thus it suffices to show that $\eta_t \le S_{t+1}$.
Given that $C_{\eta,t}\le 2$ and $C_{\eta,s}\ge 1$ for all $s$, we derive
\begin{align*}
\eta_t &= C_{\eta,t}\sqrt{\lambda_{\max}^\star}\rho^t
\le 2\sqrt{\lambda_{\max}^\star}\rho^t, \\
S_{t+1} &=\sum_{s=t+1}^{\infty}\eta_s
\ge \sqrt{\lambda_{\max}^\star}\sum_{s=t+1}^{\infty}\rho^s
=\frac{\rho}{1-\rho}\sqrt{\lambda_{\max}^\star}\rho^t
\geq 2 \sqrt{\lambda_{\max}^\star}\rho^t,
\end{align*}
provided that $\rho\ge 2/3$. This establishes that $\Delta_{t+1}\le S_{t+1}$. 

\end{itemize}
Combining these cases justifies \Cref{eq:Delta-le-tail} at time $t+1$, which in turn establishes the claim (\ref{eq:Delta-le-tail}).

\end{itemize}

Equipped with \Cref{eq:1d-bound-varying-C}, we can now readily prove \Cref{eq:1d-square-bound-varying-C}.
For any $t\ge 0$,
\begin{equation}
\label{eq:square-diff-varying-C}
|u_t^2 - \lambda^\star|
= \big| |u_t| - \sqrt{\lambda^\star} \big| \big( |u_t| + \sqrt{\lambda^\star} \big)
= \Delta_t\big(\Delta_t+2\sqrt{\lambda^\star}\big)
\le \Delta_t\big(\Delta_t+2\sqrt{\lambda_{\max}^\star}\big),
\end{equation}
where we used $\lambda^\star\le \lambda_{\max}^\star$.
Applying \Cref{eq:1d-bound-varying-C} leads to
\[
|u_t^2 - \lambda^\star|
\le
\frac{2}{1-\rho}\sqrt{\lambda_{\max}^\star}\rho^t
\left(
\frac{2}{1-\rho}\sqrt{\lambda_{\max}^\star}\rho^t
+2\sqrt{\lambda_{\max}^\star}
\right)
\le
\left(\frac{4}{(1-\rho)^2}+\frac{4}{1-\rho}\right)\lambda_{\max}^\star\rho^t
\]
as claimed. 
\end{proof}

\section{Lower bound for \signgd in matrix factorization (Proof of \Cref{thm::adam-lower-bound})}
\label{sec:lower-bound-signgd-mf}

In this proof, 
we first establish a convergence lower bound for a two-dimensional quadratic optimization problem, and then show that a $2 \times 2$ matrix factorization instance can be reduced to this problem, thereby inheriting the same lower bound.

\paragraph{Step 1: a convergence lower bound for a quadratic optimization problem.} 
Specifically, consider the following 2-dimensional quadratic minimization problem:
\begin{equation}
\mathop{\text{minimize}}\limits_{\vz \in \bR^2}\quad  f(\vz) = \frac{1}{2} \vz^\top \mH \vz,
\label{eq:lower-bound-signGD-quadratic}
\end{equation}
where the matrix $\mH$ is symmetric positive semidefinite given by
\begin{equation}
\mH = \frac{1}{2}
\begin{pmatrix}
\kappa + 1 & \kappa - 1 \\
\kappa - 1 & \kappa + 1
\end{pmatrix}
\end{equation}
with two eigenvalues $\kappa \geq 1$ and $1$. Clearly, the condition number of this matrix (or the Hessian of $f(\cdot)$) is $\kappa$, and the optimal objective value of the problem (\ref{eq:lower-bound-signGD-quadratic}) is 0, attained at $\bm{z}=\bm{0}$.   
When applied to this problem, the \signgd algorithm proceeds as 
\begin{equation}
\vz_{t+1} = \vz_t - \eta_t  \sign\big(\nabla f(\vz_t)\big)
= \vz_t - \eta_t \sign(\mH \vz_t),
\qquad t = 0,1,\cdots
\label{eq:signGD-lower-bound-quadratic}
\end{equation}
where $\eta_t > 0$ is the learning rate at iteration $t$, and the $\sign(\cdot)$ operator is applied entrywise.  

We now present a convergence lower bound for \signgd on this structured quadratic objective. The proof is deferred to Section~\ref{sec:proof-lem:signgd-kappa-lb}. 

\begin{lemma}
\label{lem:signgd-kappa-lb}
Consider solving the problem (\ref{eq:lower-bound-signGD-quadratic}) using \signgd (cf.~(\ref{eq:signGD-lower-bound-quadratic})). 
Let $\{\eta_t\}_{t \geq 0}$ be any non-increasing sequence of learning rates, and consider any accuracy level obeying $0 < \varepsilon \leq \eta_0 / \kappa$. Then, one can find an initialization $\vz_0 \in [-2\eta_0, 2\eta_0]^2$ such that 
$\|\bm{z}_t\|_2\leq \varepsilon$ can only happen after $t\geq \frac{\kappa - 1}{4}$. %
\end{lemma}

\paragraph{Step 2: reduction of matrix factorization to quadratic optimization.}
Next, we demonstrate that a $2 \times 2$ instance of matrix factorization can be reduced to the quadratic optimization problem studied in Step 1.

To be precise, consider the following matrix factorization problem: 
\begin{equation}
\mathop{\text{minimize}}\limits_{\mU\in \bR^{2\times 2}}\quad F(\mU)=\frac14\big\|\mU\mU^\top-\mH \big\|_{\mathrm{F}}^2,\qquad
\text{with }\mH=\frac12\begin{pmatrix}\kappa+1&\kappa-1\\ \kappa-1&\kappa+1\end{pmatrix},~ \kappa\ge 1.
\end{equation}
Set  $\mU^\star=\mH^{1/2}$ to be the symmetric square root of $\bm{H}$.  The \signgd algorithm proceeds as
\begin{equation}
\mU_{t+1}=\mU_t-\eta_t \sign\big(\nabla F(\mU_t)\big)=
\mU_t-\eta_t \sign \big((\mU_t\mU_t^\top-\mH)\mU_t\big),
\quad t=0,1,\cdots
\label{eq:sign-gd-trajectory-lb}
\end{equation}
where 
$\sign(\cdot)$ is applied entrywise. The following lemma---whose proof is provided in Section~\ref{sec:proof-lem:signgd-mf-kappa-lb-final}---develops a lower bound on the iteration complexity of \signgd.

\begin{lemma}
\label{lem:signgd-mf-kappa-lb-final}
Consider any learning rate sequence $\{\eta_t\}$ that is non-increasing in $t$. Then, there exists a universal constant $r_0 \in \left(0, 1/16 \right)$ such that:  
for any target accuracy $\varepsilon > 0$ satisfying $\varepsilon \leq \frac{9r_0^2}{4096\kappa^2}$ and any initial $\eta_0\leq r_0$, 
one can find an initialization $\mU_0$
obeying $\|\mU_0-\mU^\star\|_{\mathrm{F}}\le r_0$ such that the \signgd trajectory (\ref{eq:sign-gd-trajectory-lb}) 
cannot yield $F(\mU_T)\le \varepsilon$ unless
\begin{equation}
T  \ge  \frac{\kappa - 1}{4}. 
\end{equation}
\end{lemma}
This concludes the proof of \Cref{thm::adam-lower-bound}.

\subsection{Proof of Lemma~\ref{lem:signgd-kappa-lb}}
\label{sec:proof-lem:signgd-kappa-lb}
The proof is carried out in the following steps.

\medskip\noindent\textbf{Step 1: a rotated basis aligned with the sign geometry.}
For each $t\geq 0$, define
\begin{equation}
    \widetilde{\vz}_t \coloneqq \mR^\top \vz_t
    ~~\text{with }\mR^\top = \frac{1}{\sqrt{2}} \begin{pmatrix} 1 & 1 \\ -1 & 1 \end{pmatrix};
    \qquad \text{and}\quad
    \widetilde{\mH} \coloneqq
    \begin{pmatrix} \kappa &  \\  & 1 \end{pmatrix}
    .
\end{equation}
In words, $\bm{\widetilde{z}}_t=[\widetilde{z}_{1,t},\widetilde{z}_{2,t}]^{\top}$ is obtained by rotating the original iterate $\bm{z}_t$. 
These allow one to express both the objective value and its gradient at iteration $t$ as
\begin{align}
\label{eq:g-standard}
f(\vz_t) = \frac{1}{2}\widetilde{\vz}_t^\top \widetilde{\mH} \widetilde{\vz}_t = \frac{1}{2}\big(\kappa \widetilde{z}_{1,t}^2 + \widetilde{z}_{2,t}^2\big),
\qquad 
\nabla f(\vz_t) = \mH \vz_t = \mR \widetilde{\mH} \widetilde{\vz}_t
= \frac{1}{\sqrt{2}} \begin{pmatrix} \kappa \widetilde{z}_{1,t} - \widetilde{z}_{2,t} \\ \kappa \widetilde{z}_{1,t} + \widetilde{z}_{2,t} \end{pmatrix}.
\end{align}
It is easily seen that 
$\vs_t \coloneqq \sign(\nabla f(\bm{z}_t)) \in \{\pm 1\}^2$. Thus, the \signgd update in the rotated basis becomes
\begin{equation}
    \widetilde{\vz}_{t+1} 
    = \mR^\top \big( \vz_{t} - \eta_t  \vs_t\big)
    = \widetilde{\vz}_t - \eta_t \mR^\top \vs_t.
    \label{eq:tilde-zt-R-st-LB}
\end{equation}
Given that there are only 4 possibilities in $\{\pm 1\}^2$, there are also only 4 possible update directions:
\begin{align}
    \bm{R}^{\top}\vs_t = \begin{cases}
    (\sqrt{2},0) & \text{if }\vs_t = (1,1),\\
    (-\sqrt{2},0) & \text{if }\vs_t = (-1,-1),\\
    (0,-\sqrt{2}) & \text{if }\vs_t = (1,-1),\\
    (0,\sqrt{2}) & \text{if }\vs_t = (-1,1).    
    \end{cases}
\end{align}
This implies that in the eigenbasis (i.e., the above rotated coordinate system), each step of \signgd updates exactly one coordinate---either $\widetilde{z}_{1,t}$ or $\widetilde{z}_{2,t}$, but never both.

\medskip\noindent\textbf{Step 2: a condition that governs the sign patterns of the updates.}
As it turns out, there exists a condition---based on the ratio of $|\widetilde{z}_{1,t}|$ and $|\widetilde{z}_{2,t}|$---that determines when \signgd updates  each coordinate. 
\begin{lemma}
\label{lem:switching}
Consider any iteration $t$. 
\begin{itemize}

   \item If 
    $|\kappa \widetilde{z}_{1,t}| > |\widetilde{z}_{2,t}|$%
    , then
    \begin{equation}
        \widetilde{z}_{1,t+1} = \widetilde{z}_{1,t} - \sqrt{2} \eta_t \sign(\widetilde{z}_{1,t}), \qquad
        \widetilde{z}_{2,t+1} = \widetilde{z}_{2,t}.
    \end{equation}

    \item If %
    $|\kappa \widetilde{z}_{1,t}| < |\widetilde{z}_{2,t}|$%
    , then
    \begin{equation}
        \widetilde{z}_{1,t+1} = \widetilde{z}_{1,t}, \qquad
        \widetilde{z}_{2,t+1} = \widetilde{z}_{2,t} - \sqrt{2} \eta_t \sign(\widetilde{z}_{2,t}).
    \end{equation}

\end{itemize}
\end{lemma}

\begin{proof}[Proof of Lemma~\ref{lem:switching}]
According to (\ref{eq:g-standard}), the signs of the two coordinates of $\nabla f(\bm{z}_t)$ differ when
\begin{equation}
    (\kappa \widetilde{z}_{1,t} - \widetilde{z}_{2,t})(\kappa \widetilde{z}_{1,t} + \widetilde{z}_{2,t}) < 0
    \quad \Longleftrightarrow \quad
    (\kappa \widetilde{z}_{1,t})^2 < \widetilde{z}_{2,t}^2
    \quad \Longleftrightarrow \quad
    |\kappa \widetilde{z}_{1,t}| < |\widetilde{z}_{2,t}|.
\end{equation}
If $|\kappa \widetilde{z}_{1,t}| > |\widetilde{z}_{2,t}|$, then both components of $\nabla f(\bm{z}_t)$ have signs equal to $\sign(\widetilde{z}_{1,t})$, and hence the update vector is $(\pm\sqrt{2}, 0)$ in the rotated basis. 
If instead $|\kappa \widetilde{z}_{1,t}| < |\widetilde{z}_{2,t}|$, then the signs of the two components of $\nabla f(\bm{z}_t)$ are equal to $-\sign(\widetilde{z}_{2,t})$ and $ \sign(\widetilde{z}_{2,t})$, respectively, and hence the update vector in the rotated basis is $(0, \pm \sqrt{2})$. 
\end{proof}

\medskip\noindent\textbf{Step 3: a learning rate barrier.}
We now develop a general lower bound for the following sequence that updates one coordinate at a time. Specifically, consider a sequence $\vx_t = [x_{1,t}, x_{2,t}]^{\top}$, $t\geq 0$, that follows the update rule below: 
\begin{itemize}
    \item If $|x_{1,t}| < |x_{2,t}| / \kappa$, then $x_{2,t+1} = x_{2,t} - \eta_t$ and $x_{1,t+1} =x_{1,t} $;

    \item If $|x_{1,t}| > |x_{2,t}| / \kappa$, then $x_{1,t+1} = x_{1,t} - \eta_t$ and $x_{2,t+1} = x_{2,t}$.

    \item If $|x_{1,t}| = |x_{2,t}| / \kappa$, then $x_{1,t+1}$ and $x_{2,t+1}$ can be chosen arbitrarily. 
\end{itemize}

\begin{lemma}\label{lem:learning rates-lower-bound}
Consider the above  sequence $\{\bm{x}_t\}_{0\leq t \leq T}$ for any finite $T$. Let $\{\eta_t\}_{t \geq 0}$ be a non-increasing sequence of learning rates. For any target accuracy $0 < \varepsilon \leq \eta_0 / \kappa$, there exists an initialization $\vx_0 \in [0, \eta_0]^2$ such that %
$x_{2,t}< \kappa\varepsilon$ can only happen when $\eta_t<4\varepsilon$.  
\end{lemma}

\begin{proof}[Proof of Lemma~\ref{lem:learning rates-lower-bound}]
Let us initialize at $\vx_0 = [x_{1,0}, \kappa \varepsilon]^{\top}$, where $x_{1,0}$ is defined recursively as follows. 
\begin{itemize}
\item Let $T_0 = \min \{T, \max\{t : \eta_t \geq 4\varepsilon\}$\}; choose $x_{1,T_0} \in [2\varepsilon, \eta_{T_0} - 2\varepsilon]$. %
\item Define the previous iterates backward:
\begin{equation}
    x_{1,t} \coloneqq \eta_t - x_{1,t+1}, \qquad \text{for } t = T_0-1, T_0-2, \ldots, 0.
\end{equation}
\end{itemize}
 Now we show by induction that $x_{1,t} \in [2\varepsilon, \eta_t - 2\varepsilon]$ for all $0 \leq t \leq T_0$. The base case with $t = T_0$ holds trivially by construction. Assume the induction hypothesis holds at $t+1$, i.e., $x_{1,t+1} \in [2\varepsilon, \eta_{t+1} - 2\varepsilon]$. Then it follows from the assumption $\eta_t \geq \eta_{t+1}$ that
\begin{align}
    x_{1,t} &= \eta_t - x_{1,t+1} \geq \eta_t - \eta_{t+1} + 2\varepsilon \geq 2\varepsilon>0,\notag\\
    x_{1,t} &= \eta_t - x_{1,t+1} \leq \eta_t - 2\varepsilon,
\end{align}
thus justifying the induction hypothesis at $t$. 
Hence, we establish by induction that $x_{1,t} \in [2\varepsilon, \eta_t - 2\varepsilon]$ holds for all $0\leq t \leq T_0$. 
As immediate consequences, for all $0\leq t\leq T_0$ one has: (i)  $x_{1,t} > \varepsilon$; (ii) 
the update rule described above for $\{\bm{x}_t\}$ always applies only to the first coordinate $x_{1,t}$, with $x_{2,t}$ frozen at $\kappa\varepsilon$ (given that $|x_{1,t}| / |x_{2,t}| > \varepsilon / (\kappa\varepsilon) = 1 / \kappa$).
This concludes the proof. 
\end{proof}

\paragraph{Step 4: putting all this together.} 
Let us initialize \signgd to $\widetilde{\bm{z}}_0= \sqrt{2}\bm{x}_0$, with $\bm{x}_0$ constructed in the proof of Lemma~\ref{lem:learning rates-lower-bound}. Clearly, one has $\widetilde{\bm{z}}_0\in [0,\sqrt{2}\eta_0]^2$, which together with $\bm{z}_0=\bm{R}\widetilde{\bm{z}}_0$ gives ${\bm{z}}_0\in [0,2\eta_0]^2$.   
Moreover, it is seen from Lemma~\ref{lem:switching} that $\big\{\big(\frac{1}{\sqrt{2}}\widetilde{z}_{1,t},\frac{1}{\sqrt{2}}\widetilde{z}_{2,t}\big)\big\}$ follows the same dynamics as $\{\bm{x}_t\}$  in Lemma~\ref{lem:learning rates-lower-bound}---and hence $\frac{1}{\sqrt{2}}\widetilde{z}_{2,t}=\kappa\varepsilon$---before $\eta_t$ drops below $4\varepsilon$. 
To reduce $\widetilde{z}_{2,t}$ from $\sqrt{2}\kappa\varepsilon$ to below $\varepsilon$ using learning rates at most $4\varepsilon$, with each iteration changing the coordinate by at most $\sqrt{2}\eta_t$, the number of iterations needs to at least exceed
\begin{equation}
    \frac{\sqrt{2}\kappa \varepsilon-\varepsilon}{4\sqrt{2}\varepsilon} \geq \frac{\kappa-1}{4},
\end{equation}
thus completing the proof.

\subsection{Proof of Lemma~\ref{lem:signgd-mf-kappa-lb-final}} 
\label{sec:proof-lem:signgd-mf-kappa-lb-final}
The proof comprises several steps as described below. Throughout this proof, we shall focus on initializations residing within the following subspace: 
\begin{equation}
\mathcal S:=\left\{\begin{pmatrix}a&b\\ b&a\end{pmatrix}:(a,b)\in\mathbb R^2\right\}.
\end{equation}
For any 
$\mU = \begin{pmatrix}a&b\\ b&a\end{pmatrix}\in\mathcal S$, we shall refer to $(a,b)$ as its induced parameters. 

\medskip\noindent\textbf{Step 1: invariance of the set $\mathcal{S}$ under \signgd updates.}
We first show that, when initialized in $\mathcal S$, the entire trajectory of \signgd stays within $\mathcal S$.

\begin{lemma}
[Invariance of $\mathcal S$]
\label{lem:S-invariant-final}
If $\mU\in\mathcal S$,  then $\nabla F(\mU)\in\mathcal S$ and
$\sign(\nabla F(\mU))\in\mathcal S$. Consequently, $\mU_0\in\mathcal S$ implies $\mU_t\in\mathcal S$ for all $t$.
\end{lemma}

\begin{proof}[Proof of Lemma~\ref{lem:S-invariant-final}]
Note that any $\mU\in\mathcal S$ can be written as
\begin{equation}
\mU=\begin{pmatrix}a&b\\ b&a\end{pmatrix}=a\mI+b\mJ,\qquad
\text{with }\mI=\begin{pmatrix}1&\\&1\end{pmatrix}
\text{ and }\mJ=\begin{pmatrix}&1\\1&\end{pmatrix}.
\end{equation}
As can be easily verified, products of such matrices from $\mathcal{S}$ remain in $\mathcal S$. As a result,  $\mU\mU^\top=\mU^2\in\mathcal S$,
so $(\mU\mU^\top-\mH)\in\mathcal S$, and multiplying by $\mU\in\mathcal S$ yields $\nabla F(\mU)=(\mU\mU^{\top}-\mH)\mU\in\mathcal S$.
If a matrix has equal diagonals and equal off-diagonals, then applying $\sign(\cdot)$ entrywise preserves these equalities.
\end{proof}

Consequently, 
it suffices to focus on analyzing the dynamics within $\mathcal S$.

\medskip\noindent\textbf{Step 2: equivalent updates of induced parameters.}
Set
\begin{equation}
\mR=\frac1{\sqrt2}\begin{pmatrix}1&-1\\ 1&1\end{pmatrix},\qquad 
\text{and hence}\quad 
\mR^\top \mH \mR=\begin{pmatrix}\kappa&\\ &1\end{pmatrix}.
\end{equation}
For any $\mU \in\mathcal S$ with induced parameters $(a,b)$, one can easily verify that
\begin{equation}
\mR^\top \mU \mR=\diag\{\lambda_1,\lambda_2\},\qquad
\text{with }\lambda_1=a+b,\ \lambda_2=a-b.
\end{equation}
Define 
\begin{equation}
\lambda_1^\star=\sqrt\kappa,\ \lambda_2^\star=1,\qquad
\delta_1:=\lambda_1-\sqrt\kappa,\ \delta_2:=\lambda_2-1,
\end{equation}
where $\lambda_1^\star$ and $\lambda_2^\star$ correspond to the two eigenvalues of $\mU^\star=\mH^{1/2}$. These allow us to convert the gradient into \emph{exact} diagonal form as
\begin{equation}
\label{eq:grad-diag-exact}
\mR^\top \nabla F(\mU)\mR
=\diag\{g_1(\lambda_1),g_2(\lambda_2)\},
\qquad \text{with }
g_1(\lambda)\coloneqq(\lambda^2-\kappa)\lambda,~ g_2(\lambda)\coloneqq(\lambda^2-1)\lambda.
\end{equation}
Equivalently, the gradient in the original basis can be expressed as
\begin{equation}
\nabla F(\mU)=
\begin{pmatrix}G_{\mathrm{d}}&G_{\mathrm{o}}\\ G_{\mathrm{o}}&G_{\mathrm{d}}\end{pmatrix},
\qquad
\text{with }G_{\mathrm{d}}=\frac{g_1(\lambda_1)+g_2(\lambda_2)}{2},~
G_{\mathrm{o}}=\frac{g_1(\lambda_1)-g_2(\lambda_2)}{2}.
\label{eq:nabla-F-expression}
\end{equation}
Given that the update is entrywise, the induced parameter update on $(a,b)$ can be written as
\begin{equation}
\label{eq:ab-update-final}
a_{t+1}=a_t-\eta_t \sign(G_{{\mathrm{d}},t}),\qquad
b_{t+1}=b_t-\eta_t \sign(G_{{\mathrm{o}},t}).
\end{equation}

\medskip\noindent\textbf{Step 3: local gradient signs.}
Next,  expand $g_1(\cdot)$ (resp.~$g_2(\cdot)$) around $\lambda_1^\star=\sqrt\kappa$ (resp.~$\lambda_2^\star=1$) as
\begin{subequations}
\begin{align}
g_1(\sqrt\kappa+\delta_1)
&=\big((\sqrt\kappa+\delta_1)^2-\kappa\big)(\sqrt\kappa+\delta_1)
=(2\sqrt\kappa \delta_1+\delta_1^2)(\sqrt\kappa+\delta_1)=2\kappa \delta_1 + 3\sqrt\kappa \delta_1^2 + \delta_1^3, \label{eq:g1-expand}\\[2mm]
g_2(1+\delta_2)
&=\big((1+\delta_2)^2-1\big)(1+\delta_2)
=(2\delta_2+\delta_2^2)(1+\delta_2)
=2\delta_2 + 3\delta_2^2+\delta_2^3. \label{eq:g2-expand}
\end{align}
\end{subequations}
Fix a universal radius $r_0\in(0,1/16)$ and consider the local region with
\begin{equation}
\label{eq:local-box}
|\delta_1|\le \sqrt\kappa r_0,\qquad |\delta_2|\le r_0.
\end{equation}
In this region, the higher-order terms are dominated by the linear terms: indeed, using \Cref{eq:g1-expand} and $|\delta_1|\le \sqrt\kappa r_0$, we can derive
\begin{equation}
|3\sqrt\kappa \delta_1^2+\delta_1^3|
\le \left(3\sqrt\kappa \delta_1+|\delta_1|^2\right)|\delta_1|
\le \left(3\kappa r_0 +\kappa r_0^2\right)|\delta_1|
\le \frac12\kappa|\delta_1|
\end{equation}
for $r_0\leq 1/16$, which allows us to express
\begin{subequations}
\label{eq:g12-linear-dominates}
\begin{equation}
\label{eq:g1-linear-dominates}
g_1(\sqrt\kappa+\delta_1)=2\kappa\delta_1 + \Delta_1
\qquad \text{for some }|\Delta_1|\le \frac12 \kappa|\delta_1|.
\end{equation}
Similarly, it follows from \Cref{eq:g2-expand} and $|\delta_2|\le r_0$ that
\begin{equation}
\label{eq:g2-linear-dominates}
g_2(1+\delta_2)=2\delta_2+\Delta_2
\qquad \text{for some } |\Delta_2|\le \frac12 |\delta_2|.
\end{equation}
\end{subequations}

Recall the expressions of $G_{\mathrm{d}}$ and $G_{\mathrm{o}}$ in (\ref{eq:nabla-F-expression}),  which combined with (\ref{eq:g12-linear-dominates}) 
yields
\begin{subequations}
\begin{align}
G_{\mathrm{d}} &= \kappa\delta_1+\delta_2 + \frac{\Delta_1+\Delta_2}{2},\label{eq:Gd}\\
G_{\mathrm{o}} &= \kappa\delta_1-\delta_2 + \frac{\Delta_1-\Delta_2}{2}.\label{eq:Go}
\end{align}
\end{subequations}
Moreover, it follows from \Cref{eq:g1-linear-dominates,eq:g2-linear-dominates} that
\begin{equation}
\left|\frac{\Delta_1\pm \Delta_2}{2}\right|
\le \frac{|\Delta_1|+|\Delta_2|}{2}
\le \frac14(\kappa|\delta_1|+|\delta_2|).
\end{equation}
As a result, one has 
\begin{equation}
\label{eq:sign-equivalence-final}
\sign(G_{\mathrm{d}})=\sign(\kappa\delta_1+\delta_2),
\qquad
\sign(G_{\mathrm{o}})=\sign(\kappa\delta_1-\delta_2),
\end{equation}
provided that 
\begin{align}
\min\{|\kappa \delta_1+\delta_2| , |\kappa \delta_1-\delta_2|\}> \frac14(\kappa|\delta_1|+|\delta_2|).
\label{eq:local-condition-delta-12}
\end{align}

\medskip\noindent\textbf{Step 4: \signgd exhibiting matching dynamics as in Lemma~\ref{lem:signgd-kappa-lb}.}
Define 
\begin{equation}
\vs_t \coloneqq \begin{pmatrix}\sign(G_{{\mathrm{d}},t})\\ \sign(G_{{\mathrm{o}},t})\end{pmatrix}\in\{\pm1\}^2.
\end{equation}
The iterative updates of the 
$(a,b)$ parameters described in (\ref{eq:ab-update-final}) can be written compactly as
\begin{equation}
\label{eq:ut-update-LB}
\vu_{t+1}=\vu_t-\eta_t \vs_t\qquad \text{with }\vu_t=\begin{pmatrix}a_t\\ b_t\end{pmatrix}.
\end{equation}
Such update rules can be translated into updates over the eigenvalues. More specifically, 
set the eigenvalues of $\bm{U}_t$ to be $\sqrt{\kappa}+\delta_{1,t}$ and $1+\delta_{2,t}$, which combined with the fact that $\bm{U}_t\in \mathcal{S}$ gives
\begin{equation}
\label{eq:signGD-diag-exact-delta}
\mR^\top \mU_t \mR
=\diag\{\sqrt{\kappa}+\delta_{1,t},1+\delta_{2,t}\}. 
\end{equation}
A little algebra then allows us to translate \Cref{eq:ut-update-LB} into
\begin{equation}
\label{eq:delta-update}
\bm{\delta}_{t+1}=\bm{\delta}_t-\widetilde\eta_t \mR^\top \vs_t
\qquad \text{with }\widetilde\eta_t \coloneqq \sqrt2 \eta_t,
\end{equation}
where $\bm{\delta}_t=[\delta_{1,t},\delta_{2,t}]^{\top}$, and $\{\widetilde\eta_t\}$ is clearly also a non-increasing learning rate sequence.

The above update rule (\ref{eq:delta-update}) bears similarity with the one (\ref{eq:tilde-zt-R-st-LB}) analyzed in Lemma~\ref{lem:signgd-kappa-lb}. By initializing $\bm{\delta}_0$ to be $\widetilde{\bm{z}}_0$ as in the proof of Lemma~\ref{lem:signgd-kappa-lb}---except that $\eta_t$ is replaced with $\widetilde{\eta}_t$ and $\varepsilon$ replaced with $\varepsilon_q$ (to be specified shortly) in the construction of this initialization---we see from the proof of Lemma~\ref{lem:signgd-kappa-lb} that 
\[
    \kappa |\delta_{1,0}| \geq 2 |\delta_{2,0}|,
\]
which satisfies the condition described in (\ref{eq:local-condition-delta-12}). Thus, combining it with \Cref{eq:sign-equivalence-final} leads to
\begin{equation}
\vs_0=\begin{pmatrix}\sign(\kappa\delta_{1,0}+\delta_{2,0})\\ \sign(\kappa\delta_{1,0}-\delta_{2,0})\end{pmatrix}
\qquad \Longrightarrow \qquad
\bm{\delta}_{1}=\bm{\delta}_0-\widetilde\eta_0 \mR^\top
\begin{pmatrix}\sign(\kappa\delta_{1,0}+\delta_{2,0})\\ \sign(\kappa\delta_{1,0}-\delta_{2,0})\end{pmatrix}, 
\end{equation}
which is precisely the update rule of $\widetilde{\bm{z}}_1$ in the proof of Lemma~\ref{lem:signgd-kappa-lb}.  
Continuing these arguments and taking advantage of the properties derived in the proof of Lemma~\ref{lem:signgd-kappa-lb}, we can readily see that: for any $t\leq T_0$ with $T_0\coloneqq \min\big\{ \max\{t: \widetilde{\eta}_t\geq 4\varepsilon_q\}, \lceil \frac{\kappa-1}{4} \rceil \big\}$, one has
\[
    \kappa |\delta_{1,t}| \geq 2 |\delta_{2,t}|,
\]
which obeys the condition described in \Cref{eq:local-condition-delta-12} and in turns results in 
\begin{equation}
\vs_t=\begin{pmatrix}\sign(\kappa\delta_{1,t}+\delta_{2,t})\\ \sign(\kappa\delta_{1,t}-\delta_{2,t})\end{pmatrix}
\qquad \Longrightarrow \qquad
\bm{\delta}_{t+1}=\bm{\delta}_t-\widetilde\eta_t \mR^\top
\begin{pmatrix}\sign(\kappa\delta_{1,t}+\delta_{2,t})\\ \sign(\kappa\delta_{1,t}-\delta_{2,t})\end{pmatrix}. 
\end{equation}
Consequently, by construction (again see the proof of Lemma~\ref{lem:signgd-kappa-lb}) one has
\begin{align}
    \delta_{2,t} \geq \kappa \varepsilon_q \qquad \text{ for every }t \leq T_0. 
\end{align}

\medskip\noindent\textbf{Step 5: connecting $F(\mU)$ with $\bm{\delta}_t$-updates.}
On $\mathcal S$, the objective admits an exact eigen-form:
\begin{equation}
\label{eq:F-exact-eigs}
F(\mU)=\frac14\big((\lambda_1^2-\kappa)^2+(\lambda_2^2-1)^2\big)
=\frac14\big((2\sqrt\kappa \delta_1+\delta_1^2)^2+(2\delta_2+\delta_2^2)^2\big).
\end{equation}
where as before we take the eigenvalues of $\bm{U}$ to be $\sqrt{\kappa}+\delta_1$ and $1+\delta_2$. 
In the local region described in \Cref{eq:local-box} with $r_0\le 1/16$, we have $|\delta_2|\le r_0\le 1/16$, hence
\begin{equation}
|2\delta_2+\delta_2^2|\ge 2|\delta_2|-\delta_2^2 \ge \frac32|\delta_2|.
\end{equation}
Substitution into \Cref{eq:F-exact-eigs} yields the local lower bound: 
\begin{equation}
\label{eq:F-lb-d2}
F(\mU) \ge \frac14 (2\delta_2+\delta_2^2)^2  \ge  \frac{9}{16} \delta_2^2.
\end{equation}
Therefore, any iterate $\mU_T$ obeying $F(\mU_T)\le \varepsilon$ necessarily satisfies
\begin{equation}
\label{eq:need-d2-small}
|\delta_{2,T}| \le \frac{4}{3}\sqrt\varepsilon.
\end{equation}
As a consequence, setting the target level $\varepsilon_q$ in Step 4 as 
$\varepsilon_q \coloneqq \frac{4}{3}\sqrt\varepsilon$,
we see from Lemma~\ref{lem:signgd-kappa-lb} that 
\[
T\geq \frac{\kappa-1}{4}, 
\]
provided that $(\delta_{1,t},\delta_{2,t})$ satisfies Condition~(\ref{eq:local-box}). To finish up, it suffices to note that Condition~(\ref{eq:local-box})  is guaranteed as long as
$$
\eta_0
    \leq r_0,
    \qquad 
    \kappa \varepsilon_q = \frac{4}{3}\kappa \sqrt{\varepsilon} \leq r_0 \leq \frac{1}{16}.
$$
This follows from the fact that, for all $t\ge T$, we have $\delta_{1, t}\in [2\varepsilon_q, \eta_t-2\varepsilon_q]$ and $\delta_{2,t}\equiv \kappa \varepsilon_q$ according to the proof of \Cref{lem:signgd-kappa-lb}.

\section{Derivation of the training objective in \Cref{sec:main-results-transformer}}
\label{app:icl-linear-transformer}

In this section, we provide a more detailed explanation of how the objective (\ref{eq:transformer-obj}) arises from the framework of in-context learning (ICL).   
A common way to formalize ICL is to place a distribution over tasks~\citep{garg2022can}, viewing each task as a function $h$ drawn from a function class $\mathcal{H}$. A prompt consists of $N$ input--label pairs  followed by a query:
\[
P = (\vx_1, h(\vx_1), \ldots, \vx_N, h(\vx_N), \vx_{\mathrm{q}}),
\]
where inputs $\vx_i$ and query $\vx_{\mathrm{q}}$ are sampled independently from certain data distribution $\mathcal{D}_{\mathcal{X}}$, and the task function $h$ is drawn from $h\sim \mathcal{D}_{\mathcal{H}}$.

A model is said to have \textit{in-context learned} the function class $\mathcal{H}$ if, when presented with a \emph{fresh} task
$h'$ drawn from $\mathcal{H}$ and a corresponding fresh prompt, it can reliably predict the output $h'(\vx_{\mathrm{q}})$ without updating its parameters.
To understand how models acquire this ability through training, \cite{garg2022can} proposed a meta-learning protocol:
at each training step, a task $h$ and a sequence of data points are sampled to form a prompt, 
and the model parameters are updated to minimize the prediction error on the query.
They empirically demonstrated that transformers trained in this manner can
in-context learn, e.g., linear function classes.
Motivated by these findings, a growing body of theoretical work has adopted this framework to study the optimization dynamics~\citep{ahn2023transformers,zhang2024trained,huang2023context}.

\paragraph{Our instantiation: linear tasks with a fixed support set.}
Let us focus on linear regression tasks, where $h(\vx)=\vw^\top \vx$ for a task parameterized by vector $\vw\in\bR^d$.
We adopt the fixed-design setting~\citep{yang2024context}: the first $N$ input tokens $\{\vx_i\}_{i=1}^N\subset\bR^d$ in the prompt are fixed, with empirical covariance
$\mS=\frac{1}{N}\sum_{i=1}^N \vx_i\vx_i^\top .$ 
We draw $\vw\sim\mathcal D$ with $\bE[\vw]=\mathbf 0$ and $\bE[\vw\vw^\top]=\mI$, and generate noiseless labels
$y_{\vw,i}=\vw^\top \vx_i$.
The query is sampled uniformly from the support set, i.e., $\vx_{\mathrm q}\sim\mathrm{Unif}\{\vx_1,\ldots,\vx_N\}$.
Following standard ICL practice~\citep{garg2022can,zhang2024trained,ahn2023transformers}, we embed the prompt as
\begin{equation}
\mE_{\vw}
=\left(\begin{array}{llllc}
\vx_1 & \vx_2 & \cdots & \vx_N & \vx_{\mathrm q}\\
y_{\vw,1} & y_{\vw,2} & \cdots & y_{\vw,N} & 0
\end{array}\right)\in\mathbb R^{(d+1)\times(N+1)} .
\end{equation}
The goal of ICL training is to optimize a model in order to reliably predict $\vw^\top \vx_{\mathrm q}$ from $\mE_{\vw}$.

\paragraph{Single-layer linear transformer.} 
A standard single-layer transformer with input $\mE_{\vw}$ computes its output using softmax attention~\citep{vaswani2017attention}:
\[
F_{\mathsf{softmax}}(\mW_K, \mW_Q, \mW_V; \mE_{\vw}) \coloneqq   \mW_V \mE_{\vw} \cdot \mathsf{softmax}\left(\frac{ {(\mW_K \mE_{\vw})^\top (\mW_Q \mE_{\vw})} }{\gamma}\right),
\]
where $\mW_K, \mW_Q, \mW_V \in \bR^{(d+1) \times (d+1)}$ represent the key, query, and value weight matrices, $\gamma>0$ is a normalization factor, and the softmax operator $\mathsf{softmax}(\cdot)$ is applied column-wise. In this work, we consider a simplified model that is more
amenable to theoretical analysis and commonly adopted in existing theoretical literature for ICL~\citep{zhang2024trained, ahn2023transformers,huang2023context}. Specifically, we  remove the softmax nonlinearity and merge $\mW_Q, \mW_V$ into a single $\mW_{KQ}$, and take $\gamma=N$, resulting in
\begin{equation}
    F_{\mathsf{linear}}(\mW_V, \mW_{KQ}; \mE_{\vw}) = \mW_{V} \mE_{\vw}\Big(\frac{\mE_{\vw}^\top \mW_{KQ} \mE_{\vw} }{N} \Big).
\end{equation}
Furthermore, we
take $\mW_V$ and $\mW_{KQ}$ to be the following specific forms as adopted in \citep{huang2023context,yang2024context,huang2024theoretical}:
\begin{equation*}
    \mW_V=\left(\begin{array}{cc}
\mathbf{0}_{d \times d} & \mathbf{0}_d \\
\mathbf{0}_d^{\top} & 1
\end{array}\right), \quad \mW_{K Q}=\left(\begin{array}{cc}
\mQ & \mathbf{0}_d \\
\mathbf{0}_d^{\top} & 0
\end{array}\right).
\end{equation*}
Therefore, the model can be parameterized by $\mQ$, and the prediction for $\vx_{\mathrm q}$ is read off from the bottom-right entry:
\begin{equation}
\widehat{y}_{\text{q}} :=  \widehat{y}_{\text{q}} (\mQ; \mE_{\vw})= [F_{\mathsf{linear}}(\mQ; \mE_{\vw})]_{(d+1), (N+1)} .
\end{equation}
By direct calculation, this admits a simplified closed-form expression: 
\begin{align*}
\widehat{y}_{\mathrm q}
&=\left(\begin{array}{cc} \mathbf{0}^{\top}_d& 1 \end{array}\right)
\bigg(\frac{\mE_{\vw}\mE_{\vw}^{\top}}{N}\bigg)
\left(\begin{array}{c} \mQ \\ \mathbf{0}_d^{\top} \end{array}\right)\vx_{\mathrm q} =\frac{1}{N} \sum_{i=1}^N (\vw^\top \vx_i)\vx_i^\top \mQ\vx_{\mathrm q}
=\vw^\top \mS \mQ \vx_{\mathrm q}.
\end{align*}

\paragraph{in-context learning objective.} The training goal is to optimize $\mQ$ to minimize the expected squared prediction risk,
where the randomness comes from $\vw$ and $\vx_{\mathrm q}$ across prompts. Therefore,
\begin{equation}
f(\mQ)
\coloneqq
\frac{1}{2} \bE_{\vw, \vx_{\text{q}}} \Big[ \big(  \widehat{y}_{\text{q}} - \vw^\top \vx \big)^2 \Big]
= 
\frac{1}{2N} \sum_{i=1}^N \| \mS \mQ \vx_i - \vx_i \|_2^2 =
\frac{1}{2}\tr\!\big( (\mS\mQ-\mI)\mS(\mS\mQ-\mI)^\top \big).
\end{equation}
Minimizing this objective is exactly equivalent to solving the quadratic optimization problem (\ref{eq:transformer-obj}).

\section{Lower bounds for \signgd in ICL (Proof of \Cref{prop::lower-bound-transformer})}
\label{sec:lower-bound-signgd-transformer}

Consider any $\kappa\ge 2$. In what follows, we will construct an instance (i.e., a covariance matrix $\mS$ obeying
$\kappa(\mS)^3=\kappa$), on which \signgd needs $\Omega(\kappa)$ iterations to achieve the target accuracy.

\paragraph{Step 1: construction of a $2$-dimensional instance.}
Let $d=2$ and define the rotation matrix
\begin{equation}
\label{eq:R-45}
\mR \coloneqq \frac{1}{\sqrt{2}}
\begin{pmatrix}
1 & 1 \\
-1 & 1
\end{pmatrix}.
\end{equation}
Set the covariance matrix to be
\begin{equation}
\label{eq:S-hard}
\mS \coloneqq \mR
\begin{pmatrix}
\kappa^{1/3} & 0 \\
0 & 1
\end{pmatrix}
\mR^\top.
\end{equation}
It then follows that $\kappa(\mS)=\kappa^{1/3}$, hence $\kappa(\mS)^3=\kappa$.

\paragraph{Step 2: invariance of a $2$-dimensional slice.}
Define the set
\begin{equation}
\label{eq:slice-Q}
\mathcal{S}\coloneqq
\left\{
\mQ(a,b)\coloneqq
\begin{pmatrix}
a & b\\
b & a
\end{pmatrix}
:\ (a,b)\in\bR^2
\right\}.
\end{equation}
We now claim that: if $\mQ_t\in\mathcal{S}$, then $\mQ_{t+1}$ remains within $\mathcal{S}$.
\begin{proof}
To justify this claim, we first note that for any $\mQ\in\mathcal{S}$, $\mQ$ commutes with $\mR\diag(\cdot)\mR^\top$, hence $\mQ$ commutes with $\mS$ (cf.~(\ref{eq:S-hard})), and therefore $\mS\mQ\mS\in\mathcal{S}$.
The gradient of the objective $f(\cdot)$ is
\begin{equation}
\label{eq:grad-form}
\nabla f(\mQ)
=\mS^2\mQ\mS-\mS^2,
\end{equation}
which also falls within $\mathcal{S}$ whenever $\mQ\in\mathcal{S}$. Additionally, the entrywise sign map preserves the structure
$\begin{psmallmatrix}a&b\\ b&a\end{psmallmatrix}$, and as a result, $\sign(\nabla f(\mQ))\in\mathcal{S}$. These taken together prove that $\mQ_{t+1}\in \mathcal{S}$. \end{proof}

Thus, it suffices to analyze the induced dynamics within $\mathcal{S}$. In what follows, we shall write  $\mQ_t=\mQ(a_t,b_t)$, with  $(a_t,b_t)$ the induced parameters.

\paragraph{Step 3: an equivalent form of the objective.}
Any $\mQ(a,b)\in\mathcal{S}$ is diagonalizable in the basis $\mR$:
\begin{equation}
\label{eq:Q-diag}
\mQ(a,b)=\mR
\begin{pmatrix}
q_1 & 0\\
0 & q_2
\end{pmatrix}
\mR^\top,
\quad \text{where}~~
q_1=a+b \text{ and } q_2=a-b.
\end{equation}
Recall the diagonal form of $\mS$ in (\ref{eq:S-hard}). Letting $\sigma_1=\kappa^{1/3}$ and $\sigma_2=1$, we can write
\begin{equation}
\label{eq:objective-diag}
f\big(\mQ(a,b)\big)
=\frac{\sigma_1}{2}(\sigma_1 q_1-1)^2+\frac{\sigma_2}{2}(\sigma_2 q_2-1)^2
=\frac{\kappa}{2}\big(q_1-\kappa^{-1/3}\big)^2+\frac{1}{2}\big(q_2-1\big)^2.
\end{equation}
Similarly, if we express the solution $\mQ^\star=\mS^{-1}$ as 
\begin{equation}
\label{eq:Q-opt-diag}
\mQ^{\star}=\mQ^{\star}(a^{\star},b^{\star})=\mR
\begin{pmatrix}
q_1^{\star} & 0\\
0 & q_2^{\star}
\end{pmatrix}
\mR^\top,
\quad \text{where}~~
q_1^{\star}=a^{\star}+b^{\star} \text{ and } q_2^{\star}=a^{\star}-b^{\star},
\end{equation}
then it can be easily verified that
\begin{equation}
\label{eq:opt-ab}
q_1^\star=\kappa^{-1/3},\quad q_2^\star=1\quad
\implies\quad
a^\star=\frac{q_1^\star+q_2^\star}{2}=\frac{\kappa^{-1/3}+1}{2},
\quad
b^\star=\frac{q_1^\star-q_2^\star}{2}=\frac{\kappa^{-1/3}-1}{2}.
\end{equation}

Now, let us define the error coordinates
\begin{equation}
\label{eq:z-def}
\vz\coloneqq
\begin{pmatrix}
z_1\\ z_2
\end{pmatrix}
\coloneqq
\begin{pmatrix}
a-a^\star\\ b-b^\star
\end{pmatrix},
\end{equation}
allowing us to write
$$
q_1-q_1^\star=(a-a^\star)+(b-b^\star)=z_1+z_2
\quad \text{and}\quad 
q_2-q_2^\star=(a-a^\star)-(b-b^\star)=z_1-z_2. 
$$
It then follows from \Cref{eq:objective-diag} that
\begin{equation}
\label{eq:objective-z}
f\big(\mQ(a,b)\big)
=\frac{\kappa}{2}(z_1+z_2)^2+\frac{1}{2}(z_1-z_2)^2
=
\frac{1}{2}\vz^\top
\begin{pmatrix}
\kappa+1 & \kappa-1\\
\kappa-1 & \kappa+1
\end{pmatrix}
\vz
=\vz^\top \mH \vz 
\eqqcolon g(\vz),
\end{equation}
where
\begin{equation}
\label{eq:H-hard}
\mH\coloneqq \frac{1}{2}
\begin{pmatrix}
\kappa+1 & \kappa-1\\
\kappa-1 & \kappa+1
\end{pmatrix}.
\end{equation}
In particular, $g(\cdot)$ is a quadratic function with minimizer $\vz=\bm{0}$ and gradient $\nabla g(\vz)=2\mH\vz$.

\paragraph{Step 4: \signgd exhibiting matching dynamics as in \Cref{lem:signgd-kappa-lb}.}
Given the invariance of $\mathcal{S}$ and the fact that $(a_t,b_t)$ are the diagonal and off-diagonal entries of $\mQ_t$, 
the \muon update induces
\begin{equation}
\label{eq:z-update}
\vz_{t+1}
=
\vz_t-\eta_t\,\sign(\mH\vz_t),
\qquad t=0,1,2,\dots,
\end{equation}
where $\vz_t$ is defined by \Cref{eq:z-def} w.r.t.~the $t$-th iterate, and $\mH$ is given in \Cref{eq:H-hard}. This matches precisely the \signgd recursion studied in \Cref{lem:signgd-kappa-lb}. 

Therefore, for any non-increasing $\{\eta_t\}_{t\ge 0}$ and any $0<\varepsilon\le \sqrt{2}\eta_0/\kappa$, \Cref{lem:signgd-kappa-lb} guarantees that one can choose an initialization $\vz_0\in[0,2\eta_0]^2$ such that $\|\vz_t\|_2\le \varepsilon/\sqrt{2}$ can only occur after
\begin{equation}
\label{eq:lb-iter}
t\ge \frac{\kappa-1}{4}.
\end{equation}

\paragraph{Step 5: translating it back to $\mQ_t$.}
Recalling that $\mQ_t=\mQ(a_t,b_t)$ and $\mQ^\star=\mQ(a^\star,b^\star)$, we have
\begin{equation}
\label{eq:norm-Q-vs-z}
\|\mQ_t-\mQ^\star\|_{\fro}^2
=
2(a_t-a^\star)^2+2(b_t-b^\star)^2
=2\|\vz_t\|_2^2,
\quad\Longrightarrow\quad
\|\mQ_t-\mQ^\star\|_{\fro}=\sqrt{2}\,\|\vz_t\|_2.
\end{equation}
Hence, with the above-mentioned initialization, achieving $\|\mQ_t-\mQ^\star\|_{\fro}\le \varepsilon$
requires at least $(\kappa-1)/4$ iterations. This establishes the \signgd lower bound claimed in \Cref{prop::lower-bound-transformer}.

\section{Technical lemmas}
\label{sec:technical-lemmas}

In this section, we gather a couple of technical lemmas that are useful in our analysis. We begin with three lemmas concerned with perturbation bounds for matrix signs and rank-$r$ approximations.

\begin{lemma}[Adapted from Theorem~2.1 in \cite{li2006some}]
\label{lem::matrix-sign-perturbation-v2}
For arbitrary two matrices $\mX, \mY\in \bR^{m\times n}$ of the same rank $r$, we have
\begin{equation}
    \norm{\msign(\mX)-\msign(\mY)}\leq \frac{2}{\min\{\sigma_r(\mX),\sigma_r(\mY)\}}\norm{\mX-\mY}.
\end{equation}
    
\end{lemma}

\begin{lemma}[Adapted from Theorem~2 in \cite{li1995new}]
\label{lem::matrix-sign-perturbation}
For arbitrary two matrices $\mX, \mY\in \bR^{m\times n}$ ($m>n$) of full column rank, we have
\begin{equation}
    \norm{\msign(\mX)-\msign(\mY)}\leq \frac{3}{\sigma_n(\mX)}\norm{\mX-\mY}.
\end{equation}
    
\end{lemma}

\begin{lemma}[Adapted from Equation (4.4) in \cite{wedin1972perturbation}]
\label{lem::rank-r-perb}
    For any two matrices $\mX, \mY\in \bR^{m\times n}$, denote the best rank-$r$ approximations by $\mX_r, \mY_r$, respectively. We define the eigengap $\delta=\min\{\sigma_r(\mX), \sigma_r(\mY)\}-\max\{\sigma_{r+1}(\mX), \sigma_{r+1}(\mY)\}$. Then, we have
    \begin{equation}
        \norm{\mX_r-\mY_r}\leq \norm{\mX-\mY}\left(3+\frac{\sigma_{r+1}(\mX)+\sigma_{r+1}(\mY)}{\delta}\right).
    \end{equation}
\end{lemma}

Next, we gather two lemmas regarding the singular values of Gaussian random matrices. 
\begin{lemma}[Adapted from Theorem~6.1 in \cite{wainwright2019high}]
    \label{lem::concentration-maximal-singular-value}
    Suppose that $\mG\in \bR^{d_1\times d_2}$ is a standard Gaussian matrix, where $d_1\geq d_2$. Then, 
    it holds that
    \begin{equation}
        \begin{aligned}
            \bP\big(\norm{\mG}\geq 3\sqrt{d_1}\big)&\leq \exp\left(-d_1/{2}\right).
        \end{aligned}
    \end{equation}
\end{lemma}

\begin{lemma}[Adapted from Equation~(3.2) in \cite{rudelson2010non}]
    \label{lem::concentration-minimal-singular-value}
    Suppose that the entries of the $\mG\in \bR^{d\times d}$ are i.i.d.~standard Gaussian random variables. Then, for any $\varepsilon>0$, 
    \begin{equation}
        \begin{aligned}
            \bP\big(\sigma_{\min}(\mG)\leq \varepsilon d^{-1/2}\big)&\leq \varepsilon.
        \end{aligned}
    \end{equation}
\end{lemma}

Finally, we show that for any two orthonormal matrices in $\cO_{d\times r}$ with $r<d$, it is plausible to augment each into a square orthonormal matrix, without increasing their spectral-norm difference by much. See \Cref{sec:proof:lem:orth-completion} for the proof of this result.  
\begin{lemma}
\label{lem:orth-completion}
Let $\mO_1,\mO_2 \in \cO_{d\times r}$, where $r< d$. Then there exist
$\mR_1,\mR_2 \in \cO_{d\times (d-r)}$ such that
\begin{align*}
 \mA_1 \coloneqq &[ \mO_1,   \mR_1 ] \in \cO_{d\times d}, \qquad 
 \mA_2 \coloneqq [ \mO_2,   \mR_2 ] \in \cO_{d\times d},
\\
&\text{and} \qquad 
\| \mA_1 -  \mA_2\|  \le  \sqrt{2} \,\|\mO_1 - \mO_2\|.
\end{align*}
\end{lemma}

\subsection{Proof of \Cref{lem:orth-completion}}
\label{sec:proof:lem:orth-completion}
Denote by $\cS_i \coloneqq \mathsf{span}( \mO_i)$ the $r$-dimensional subspace
spanned by the columns of $ \mO_i$. Let
$\theta_1,\dots,\theta_r \in [0,\pi/2]$ represent the principal angles between
$\cS_1$ and $\cS_2$ (see, e.g., \citet[Chapter~6.4.3]{golub2013matrix} and \citet[Section~2.2]{chen2021spectral}).
Define
$
\theta_{\max} \coloneqq \max_{1\le i\le r} \theta_i.
$

\paragraph{Step 1: computing distance between two subspaces.}
Consider any pair of orthonormal bases $ \mQ_1, \mQ_2\in\bR^{d\times r}$ with
$\mathsf{span}( \mQ_i) = \cS_i$. Classical matrix perturbation theory (e.g., \citet[Section~4.3]{edelman1998geometry}) asserts that
\begin{equation}
\label{eq:min-basis-dist}
\inf_{\substack{ \mQ_1, \mQ_2 \in \mathcal{O}_{d\times r}:\, \mathsf{span}(\mQ_1)=\cS_1, \mathsf{span}(\mQ_2)=\cS_2}}
\| \mQ_1 -  \mQ_2\|
=
2\sin \Bigl(\frac{\theta_{\max}}{2}\Bigr), 
\end{equation}
thus implying that
\begin{equation}
\label{eq:O-lower-bound}
2\sin \Bigl(\frac{\theta_{\max}}{2}\Bigr)
 \le 
\|\mO_1 - \mO_2\|.
\end{equation}

Additionally, let $\cS_i^\perp$ denote the $(d-r)$-dimensional orthogonal complement of $\cS_i$.
The maximum principal angle between $\cS_1^\perp$ and $\cS_2^\perp$ is again $\theta_{\max}$. This implies that
\begin{equation}
\label{eq:min-basis-dist-compl}
\inf_{\substack{ \mB_1, \mB_2\in\mathcal{O}_{d\times (d-r)}:\, \ \mathsf{span}(\mB_1)=\cS_1^\perp, \mathsf{span}(\mB_2)=\cS_2^\perp}}
\| \mB_1 -  \mB_2\|
=
2\sin \Bigl(\frac{\theta_{\max}}{2}\Bigr).
\end{equation}

\paragraph{Step 2: choosing orthogonal complements with controlled distance.}
By \Cref{eq:min-basis-dist-compl}, one can find orthonormal bases
$\mR_1\in \bR^{d\times (d-r)}$ (resp.~$\mR_2\in\bR^{d\times (d-r)}$) of   $\cS_1^\perp$ (resp.~$\cS_2^\perp$) such that
\begin{equation}
\label{eq:R-bound}
\|\mR_1 - \mR_2\|
 = 
2\sin \Bigl(\frac{\theta_{\max}}{2}\Bigr).
\end{equation}
Combining \Cref{eq:O-lower-bound,eq:R-bound} yields
\begin{equation}
\label{eq:R-in-terms-of-O}
\|\mR_1 - \mR_2\|
 \le 
\|\mO_1 - \mO_2\|.
\end{equation}
By construction, $\mA_i \coloneqq [   \mO_i, \mR_i ]$ forms a square orthogonal matrix.

\paragraph{Step 3: bounding the distance between $\mA_1$ and  $\mA_2$.}
Observe that
\[
 \mA_1 -  \mA_2
=
[ \mO_1 - \mO_2,    \mR_1 - \mR_2 ].
\]
For any $\vx = {\footnotesize
\begin{bmatrix}
\vx_1 \\ \vx_2
\end{bmatrix} }
\in\bR^d$ with $\vx_1\in\bR^r$ and  $\vx_2\in\bR^{d-r}$, it holds that
\[
( \mA_1 -  \mA_2)\vx
=
(\mO_1 - \mO_2)\vx_1 + (\mR_1 - \mR_2)\vx_2.
\]
This allows one to establish that
\begin{align*}
\| \mA_1 -  \mA_2\|^2
&=
\sup_{\|\vx\|_2=1}
\|( \mA_1 -  \mA_2)\vx\|^2 \\
&\le
\sup_{\|\vx\|_2=1}
\Bigl(
\|\mO_1 - \mO_2\|\|\vx_1\|_2
+
\|\mR_1 - \mR_2\|\|\vx_2\|_2
\Bigr)^2 \\
&\le
\sup_{\|\vx\|_2=1}
\Bigl(
\|\mO_1 - \mO_2\|\|\vx_1\|_2
+
\|\mO_1 - \mO_2\|\|\vx_2\|_2
\Bigr)^2
\quad\text{(by \Cref{eq:R-in-terms-of-O})} \\
&=
\|\mO_1 - \mO_2\|^2
\sup_{\|\vx\|_2=1}
\bigl(\|\vx_1\|_2 + \|\vx_2\|_2\bigr)^2\\
&\leq 
2\|\mO_1 - \mO_2\|^2,
\end{align*}
where the last line holds since, by Cauchy-Schwarz, 
$|\vx_1\|_2 + \|\vx_2\|_2
\le
\sqrt{2} \sqrt{\|\vx_1\|_2^2 + \|\vx_2\|_2^2}
=
\sqrt{2}\|\vx\|_2^2$. This completes the proof.

\end{document}